\documentclass{article}
\usepackage{arxiv}

\usepackage[utf8]{inputenc} % allow utf-8 input
\usepackage[T1]{fontenc}    % use 8-bit T1 fonts
\usepackage{hyperref}       % hyperlinks
\usepackage{url}            % simple URL typesetting
\usepackage{booktabs}       % professional-quality tables
\usepackage{amsfonts}       % blackboard math symbols
\usepackage{nicefrac}       % compact symbols for 1/2, etc.
\usepackage{microtype}      % microtypography
\usepackage{lipsum}

\usepackage[small]{caption}
\usepackage{algorithm}
\usepackage[noend]{algpseudocode}
\usepackage{amssymb}
\usepackage[round]{natbib}
\usepackage{mathtools}
\usepackage{amsmath}
\usepackage{tikz}
\usetikzlibrary{calc}
\usepackage{todonotes}
\usepackage{dsfont}
\usepackage{enumitem}
\usepackage{tikz}
\usepackage{amsthm}
\usepackage{thm-restate}
\usepackage{booktabs}
\usepackage{multirow}
\usepackage{newfloat}
\usepackage{wrapfig}
\usepackage{dsfont}
\usepackage{comment}
\usepackage{multirow}
\usetikzlibrary{decorations.markings}

\newtheorem{theorem}{Theorem}
\newtheorem{lemma}{Lemma}
\newtheorem{remark}{Remark}
\newtheorem{corollary}{Corollary}

\newtheorem{definition}{Definition}

\newtheorem{condition}[theorem]{Condition}

\newcommand{\nb}[3]{{\colorbox{#2}{\bfseries\sffamily\scriptsize\textcolor{white}{#1}}}{\textcolor{#2}{\sf\small\textit{#3}}}}
\newcommand{\stradi}[1]{\nb{Stradi}{red}{#1}}

\newcommand{\BigOL}[1]{\tilde{\mathcal{O}}\left(#1\right)}

\DeclareMathOperator*{\argmax}{arg\,max}

\DeclareMathOperator*{\argmin}{arg\,min}

\newtheoremstyle{eventtheoremstyle}%
  {5pt}% (space above)
  {3pt}% (space below)
  {\itshape}% (body font)
  {}% (indent amount)
  {\bfseries}% {theorem head font}
  {:}% {punctuation after theorem head}
  {.8em}% {space after theorem head}
  {\thmname{#1} #3}% {theorem head spec}

\theoremstyle{eventtheoremstyle}

\allowdisplaybreaks

\author{
	Francesco Emanuele Stradi\\
	Politecnico di Milano\\
	\texttt{francescoemanuele.stradi@polimi.it}
	\And
	Anna Lunghi\\
	Politecnico di Milano\\
	\texttt{anna.lunghi@mail.polimi.it}
	\And
	Matteo Castiglioni\\
	Politecnico di Milano\\
	\texttt{matteo.castiglioni@polimi.it}
	\And
	Alberto Marchesi\\
	Politecnico di Milano\\
	\texttt{alberto.marchesi@polimi.it}
	\And
	Nicola Gatti\\
	Politecnico di Milano\\
	\texttt{nicola.gatti@polimi.it}
}

\title{Learning Constrained Markov Decision Processes With Non-stationary Rewards and Constraints}
\begin{document}

\maketitle

\begin{abstract}
 	In \emph{constrained Markov decision processes} (CMDPs) with \emph{adversarial} rewards and constraints, a well-known impossibility result prevents any algorithm from attaining both sublinear regret and sublinear constraint violation, when competing against a best-in-hindsight policy that satisfies constraints on average.
 	In this paper, we show that this negative result can be eased in CMDPs with \emph{non-stationary} rewards and constraints, by providing algorithms whose performances smoothly degrade as non-stationarity increases.
 	Specifically, we propose algorithms attaining $\tilde{\mathcal{O}} (\sqrt{T} + C)$ regret and \emph{positive} constraint violation under \emph{bandit} feedback, where $C$ is a corruption value measuring the environment non-stationarity.
 	This can be $\Theta(T)$ in the worst case, coherently with the impossibility result for adversarial~CMDPs.
 	First, we design an algorithm with the desired guarantees when $C$ is known.
 	Then, in the case $C$ is \emph{unknown}, we show how to obtain the same results by embedding such an algorithm in a general \emph{meta-procedure}.
 	This is of independent interest, as it can be applied to \emph{any} non-stationary constrained online learning setting.
\end{abstract}
\section{Introduction}\label{introduction}

Reinforcement learning~\citep{sutton2018reinforcement} is concerned with settings where a learner sequentially interacts with an environment modeled as a \emph{Markov decision process} (MDP)~\citep{puterman2014markov}.
Most of the works in the field focus on learning policies that maximize learner's rewards.
However, in most of the real-world applications of interest, the learner also has to meet some additional requirements.
For instance, autonomous vehicles must avoid crashing~\citep{isele2018safe,wen2020safe}, bidding agents in ad auctions must \emph{not} deplete their budget~\citep{wu2018budget,he2021unified}, and users of recommender systems must \emph{not} be exposed to offending content~\citep{singh2020building}.
These requirements can be captured by \emph{constrained} MDPs (CMDPs)~\citep{Altman1999ConstrainedMD}, which generalize MDPs by specifying constraints that the learner has to satisfy while maximizing their rewards.

We study \emph{online learning} in \emph{episodic} CMDPs~(see, \emph{e.g.},~\citep{Exploration_Exploitation}), where the goal of the learner is twofold.
On the one hand, the learner wants to minimize their \emph{regret}, which measures how much reward they lost over the episodes compared to what they would have obtained by always using a best-in-hindsight constraint-satisfying policy.
On the other hand, the learner wants to ensure that the \emph{(cumulative) constraint violation} is minimized during the learning process.
Ideally, one seeks for algorithms with both regret and constraint violation growing sublinearly in the number of episodes $T$.

A crucial feature distinguishing online learning problems in CMDPs is whether rewards and constraints are selected \emph{stochastically} or \emph{adversarially}.
Most of the works focus on the case in which constraints are stochastic (see, \emph{e.g.},~\citep{Online_Learning_in_Weakly_Coupled,Constrained_Upper_Confidence,Exploration_Exploitation,Upper_Confidence_Primal_Dual,bounded,bai2023}), with only one exception addressing settings with adversarial constraints~\citep{germano2023bestofbothworlds}.
This is primarily motivated by a well-known impossibility result by~\citet{Mannor}, which prevents any learning algorithm from attaining both sublinear regret and sublinear constraint violation, when competing against a best-in-hindsight policy that satisfies the constraints \emph{on average}.
However, dealing with adversarially-selected constraints is of paramount importance to cope with real-world environments, which are typically non-stationary.
%
% Indeed, in most of the real-world applications, constraints change over the episodes.
%
% it is unreasonable to assume that constraints are selected stochastically according to a fixed probability distribution, as they rather change over the episodes.

% \albo{Mi sa che le cit ho messe a cazzo perchè mi confondono i nomi delle references.}

\subsection{Original contributions}

The main contribution of this paper is to show how to ease the negative result of~\citep{Mannor}, by considering CMDPs with \emph{non-stationary rewards and constraints}.
Specifically, we address CMDPs where rewards and constraints are selected from probability distributions that are allowed to change \emph{adversarially} from episode to episode.
One may think of our setting as bridging the gap between fully-stochastic and fully-adversarial ones.
We design algorithms whose performances---in terms of regret and constraint violation---smoothly degrade as a suitable measure of non-stationarity increases.
This is called \emph{(adversarial) corruption}, as it intuitively quantifies how much the distributions of rewards and constraints vary over the episodes with respect to some ``fictitious'' non-corrupted counterparts.

We propose algorithms that attain $\tilde{\mathcal{O}} (\sqrt{T} + C)$ regret and constraint violation, where $C$ denotes the corruption of the setting.
We remark that $C$ can be $\Theta(T)$ in the worst case, and, thus, our bounds are coherent with the impossibility result by~\citet{Mannor}.
Notably, our algorithms work under \emph{bandit} feedback, namely, by only observing rewards and constraint costs of the state-action pairs visiting during episodes.
Moreover, they are able to manage \emph{positive} constraint violation.
This means that they do \emph{not} allow for a negative violation (\emph{i.e.}, a constraint satisfaction) to cancel out a positive one across different episodes.
This is a crucial requirement for most of the practical applications.
For instance, in autonomous driving, avoiding a collision does \emph{not} ``repair'' a previously-occurred crash.

In the first part of the paper, we design an algorithm, called \texttt{NS-SOPS}, which works assuming that the value of the corruption $C$ is known.
This algorithm achieves $\tilde{\mathcal{O}} (\sqrt{T} + C)$ regret and positive constraint violation by employing a \emph{policy search} method that is \emph{optimistic} in both reward maximization and constraint satisfaction.
Specifically, the algorithm incorporates $C$ in the confidence bounds of rewards and constraint costs, so as to ``boost'' its optimism and achieve the desired guarantees.

In the second part of the paper, we show how to embed the \texttt{NS-SOPS} algorithm in a \emph{meta-procedure} that allows to achieve $\tilde{\mathcal{O}} (\sqrt{T} + C)$ regret and positive constraint violation when $C$ is \emph{unknown}.
The meta-procedure works by instantiating multiple instances of an algorithm for the case in which $C$ is known, each one taking care of a different ``guess'' on the value of $C$.
Specifically, the meta-procedure acts as a \emph{master} by choosing which instance to follow in order to select a policy at each episode.
To do so, it employs an adversarial online learning algorithm, which is fed with losses constructed starting from the Lagrangian of the CMDP problem, suitably modified to account for \emph{positive} constraint violation.
%
%it employs an FTRL approach with a suitable log-barrier regularization, fed with losses constructed starting from the Lagrangian of the CMDP problem, suitably modified to account for positive constraint violations.
%
Our meta-procedure is of independent interest, as it can be applied in \emph{any} non-stationary constrained online learning setting, so as to relax the knowledge of $C$.

\subsection{Related works}

Within the literature on CMDPs, settings with \emph{stochastic} rewards and constraints have been widely investigated.
However, their \emph{non-stationary} counterparts, including \emph{adversarial} ones in the worst case, are still largely unexplored.
In the following, we discuss the works that are most related to ours, while we refer the reader to Appendix~\ref{app:related} for a comprehensive survey of related works.
%
% In the following, we highlight the works which are most related to ours. Due to space constraints, we refer to Appendix~\ref{app:related} for a complete discussion on related works. While online learning in CMDPs with stochastic rewards and constraints have been widely studied (see, \emph{e.g.},~\citep{Exploration_Exploitation}), the non-stationary (in worst case adversarial) counterpart is far to be completely explored. Specifically, the following are the main works that deal with online CMDPs in presence of non-stationarity. 

%\stradi{questa frase vogliamo leggermente riformularla per renderla meno aggressiva?}

\citet{Upper_Confidence_Primal_Dual} provide the first primal-dual approach to deal with episodic CMDPs with adversarial losses and stochastic constraints, achieving, under full feedback, both sublinear regret and sublinear (non-positive) constraint violation (\emph{i.e.}, allowing for cancellations). 
\citet{stradi2024learning} are the first to tackle CMDPs with adversarial losses and stochastic constraints under {bandit} feedback, by proposing an algorithm that achieves sublinear regret and sublinear positive constraint violation.
%, assuming that the constraint are stochastic.
%
These works do \emph{not} consider settings where constraints are non-stationary, \emph{i.e.}, they may change over the episodes.

\citet{ding_non_stationary}~and~\citet{wei2023provably} consider the case in which rewards and constraints are non-stationary, assuming that their variation is bounded.
%, as in our work.
%
Our work differs from theirs in multiple aspects.
First, we consider \emph{positive} constraint violation, while they allow for cancellations.
As concerns the definition of regret, ours and that used by \citet{ding_non_stationary}~and~\citet{wei2023provably} are \emph{not} comparable.
Indeed, they employ a dynamic regret baseline, which, in general, is harder than the static regret employed in our work.
However, they compare learner's performances against a dynamic policy that satisfies the constraints at every round.
Instead, we consider a policy that satisfies the constraints \emph{on average}, which can perform arbitrarily better than a policy satisfying the constraints at every round.
Furthermore, the dependence on $T$ in their regret bound is much worse than ours, even when the non-stationarity is small, namely, when it is a constant independent of $T$ (and, thus, dynamic regret collapses to static regret).
Finally, we do \emph{not} make any assumption on $T$, while both regret and constraint violation bounds in~\citep{wei2023provably} only hold for large $T$.

Finally,~\citet{germano2023bestofbothworlds} are the first to study CMDPs with adversarial constraints.
Given the impossibility result by~\citet{Mannor}, they propose an algorithm that, under full feedback, attains sublinear (non-positive) constraint violation (\emph{i.e.}, with cancellations allowed) and a fraction of the optimal reward, thus resulting in a regret growing linearly in $T$.
We show that sublinear regret and sublinear constraint violation can indeed be attained simultaneously if one takes into account the corruption $C$, which can be seen as a measure of how much adversarial the environment is.
%
% Assuming that the adversariality of the problem is bounded by a corruption measure $C$, we manage to show that sublinear regret and sublinear violations can be attained even when the constraints are non-stationary.
%
Moreover, let us remark that our algorithms deal with \emph{positive} constraint violation under \emph{bandit} feedback, and, thus, they are much more general than those in~\citep{germano2023bestofbothworlds}.

\section{Preliminaries}\label{sec:Preliminaries}

\subsection{Constrained Markov decision processes}

We study \emph{episodic constrained} MDPs~\citep{Altman1999ConstrainedMD} (CMDPs), in which a learner interacts with an unknown environment over $T$ episodes, with the goal of maximizing long-term rewards subject to some constraints.
%
% which are defined as follows.
%
% There is a finite number $T$ of episodes, with $t\in[T]$ denoting a specific episode.\footnote{In this paper, we denote by $[a \ldots b]$ the set of all the natural numbers from $a \in \mathbb{N}$ to $b \in \mathbb{N}$ (both included), while $[b] \coloneqq [ 1 \ldots b ]$ is the set of the first $b \in \mathbb{N}$ natural numbers.}
%
$X$ is a finite set of states of the environment, $A$ is a finite set of actions available to the learner in each state, while the environment dynamics is  governed by a transition function $P : X \times A \times X  \to [0, 1]$, with $P (x^{\prime} |x, a)$ denoting the probability of going from state $x \in X$ to $x^{\prime} \in X$ by taking action $a \in A$.\footnote{In this paper, we consider w.l.o.g.~\emph{loop-free} CMDPs. This means that $X$ is partitioned into $L$ layers $X_{0}, \dots, X_{L}$ such that the first and the last layers are singletons, \emph{i.e.}, $X_{0} = \{x_{0}\}$ and $X_{L} = \{x_{L}\}$. Moreover, the loop-free property implies that $P(x^{\prime} |x, a) > 0$ only if $x^\prime \in X_{k+1}$ and $x \in X_k$ for some $k \in [0 \ldots L-1]$. Notice that any episodic CMDP with horizon $L$ that is \emph{not} loop-free can be cast into a loop-free one by suitably duplicating the state space $L$ times, \emph{i.e.}, a state $x$ is mapped to a set of new states $(x, k)$, where $k \in [0 \ldots L]$.}%. We will refer as $k(x)$ to the index of the layer that $x$ belongs to.}
%
% $P : X \times A \times X  \to [0, 1]$ is a transition function, with $P (x^{\prime} |x, a)$ denoting the probability of going from state $x \in X$ to state $x^{\prime} \in X$ by taking action $a \in A$.
%
At each episode $t \in [T]$,\footnote{In this paper, we denote by $[a \ldots b]$ the set of all the natural numbers from $a \in \mathbb{N}$ to $b \in \mathbb{N}$ (both included), while $[b] \coloneqq [ 1 \ldots b ]$ is the set of the first $b \in \mathbb{N}$ natural numbers.} a reward vector $r_{t} \in [0,1]^{|X\times A|}$ is sampled according to a probability distribution $\mathcal{R}_t$, with $r_t(x,a)$ being the reward of taking action $a \in A$ in state $x \in X$ at episode $t$.
Moreover, a constraint cost matrix $G_t \in [0,1]^{|X\times A| \times m }$ is sampled according to a probability distribution $\mathcal{G}_t$, with $g_{t,i}(x,a)$ being the cost of constraint $i \in [m]$ when taking action $a \in A$ in state $x \in X$ at episode $t$.
We also denote by $g_{t,i} \in [0,1]^{|X\times A|}$ the vector of all the costs $g_{t,i}(x,a)$ associated with constraint $i$ at episode $t$.
% and $g_{t,i}(x,a)$ being the cost of constraint $i$ when taking action $a \in A$ in state $x \in X$ at episode $t$.
%
Each constraint requires that its corresponding expected cost is kept below a given threshold.
The thresholds of all the $m$ constraints are encoded in a vector $\alpha\in [0,L]^m$, with $\alpha_i$ denoting the threshold of the $i$-th constraint.
%
%$\{\mathcal{R}_t\}_{t=1}^T$ is a sequence of probability distributions with support in $[0,1]^{|X\times A|}$.
%%
%At each episode $t \in [T]$, a reward vector $r_{t} \in [0,1]^{|X\times A|}$ is sampled according to $\mathcal{R}_t$, with $r_t(x,a)$ being the reward of taking action $a \in A$ in state $x \in X$ at episode $t$.
%%
%$\{\mathcal{G}_t\}_{t=1}^T$ is a sequence of probability distributions with support in $[0,1]^{|X\times A|\times m}$.
%%
%At each episode $t \in [T]$, a constraints cost matrix $G_t \in [0,1]^{|X\times A|\times m}$ is sampled according to $\mathcal{G}_t$, with $g_{t,i} \in [0,1]^{|X\times A|}$ denoting the vector of all the costs associated with constraint $i \in [m]$ at episode $t$ and $g_{t,i}(x,a)$ being the cost of constraint $i$ when taking action $a \in A$ in state $x \in X$ at episode $t$.
%%
%Finally, $\alpha\in [0,L]^m$ is a vector of cost thresholds that characterize the $m$ constraints, where $\alpha_i$ denotes the threshold for the $i$-th constraint.

We consider a setting in which the sequences of probability distributions $\{\mathcal{R}_t\}_{t=1}^T$ and $\{\mathcal{G}_t\}_{t=1}^T$ are selected \emph{adversarially}.
Thus, reward vectors $r_t$ and constraint cost matrices $G_t$ are random variables whose distributions are allowed to change arbitrarily from episode to episode.
To measure how much such probability distributions change over the episodes, we introduce the notion of \emph{(adversarial) corruption}.
In particular, we define the adversarial corruption $C_r$ for the rewards as follows:
\begin{equation}\label{def: C^r}
	C_r \coloneqq \min_{r \in[0,1]^{|X\times A|}}\sum_{t\in [T]}\left\lVert \mathbb{E}[r_{t}] - r \right\rVert_1.
\end{equation}
Intuitively, the corruption $C_r$ encodes the sum over all episodes of the distances between the means $\mathbb{E}[r_{t}]$ of the adversarial distributions $\mathcal{R}_t$ and a ``fictitious'' non-corrupted reward vector $r$.
Notice that a similar notion of corruption has been employed in unconstrained MDPs to measure the non-stationarity of transition probabilities; see~\citep{jin2024no}. 
In the following, we let $r^\circ \in[0,1]^{|X\times A|}$ be a reward vector that attains the minimum in the definition of $C_r$.
Similarly, we introduce the adversarial corruption $C_G$ for constraint costs, which is defined as follows:
\begin{equation}\label{def: C^G}
	C_G \coloneqq {\min_{G \in[0,1]^{|X\times A|\times m}}}\underset{t\in [T]}{\sum}\underset{i \in [m]}{\max}\lVert \mathbb{E}[g_{t,i}] - g_i \rVert_1,
\end{equation}
where $g_i$ is the $i$-th component of $G$. We let $G^\circ \in [0,1]^{|X\times A|\times m}$ be the constraint cost matrix that attains the minimum in the definition of $C_G$.
Finally, we introduce the total adversarial corruption $C$, which is defined as $
C \coloneqq \max\{C_G,C_r\}$.

		\begin{algorithm}[]
			\caption{Learner-Environment Interaction}
			\label{alg: Learner-Environment Interaction}
			\begin{algorithmic}[1]
				% \For{$t\in[T]$}
				\State $\mathcal{R}_t$ and $\mathcal{G}_t$ are chosen \textit{adversarially}
				%\State $\{g_{t,i}\}_{i=1}^m$ sampled from the stochastic distribution $\mathcal{G}_t$
				\State Choose a policy $\pi_{t}: X \times A  \to [0, 1]$
				\State Observe initial state $x_{0}$
				\For{$k = 0, \ldots,  L-1$}
				\State Play $a_{k} \sim \pi_t(\cdot|x_{k})$
				\State Observe  $r_t(x_k,a_k)$ and $g_{t,i}(x_k,a_k)$ for $i\in[m]$
				% \State Environment evolves to $x_{k+1}\sim P(\cdot|x_{k},a_{k})$
				\State Observe new state $x_{k+1}\sim P(\cdot|x_{k},a_{k})$
				\EndFor
				% \EndFor
			\end{algorithmic}
		\end{algorithm}

Algorithm~\ref{alg: Learner-Environment Interaction} summarizes how the learner interacts with the environment at episode $t \in [T]$.
In particular, the learner chooses a \emph{policy} $\pi: X \times A \to [0,1]$ at each episode, defining a probability distribution over actions to be employed in each state.
For ease of notation, we denote by $\pi(\cdot|x)$ the probability distribution for a state $x \in X$, with $\pi(a|x)$ being the probability of selecting action $a \in A$.
%
% Algorithm~\ref{alg: Learner-Environment Interaction} depicts the interaction between the learner and the environment in a CMDP.
%
Let us remark that we assume that the learner knows $X$ and $A$, but they do \emph{not} know anything about $P$.
Moreover, the \emph{feedback} received by the learner after each episode is \emph{bandit}, as they observe the realizations of rewards and constraint costs only for the state-action pairs $(x_k,a_k)$ actually visited during that episode.

\subsection{Occupancy measures}

Next, we introduce \emph{occupancy measures}, following the notation by~\citep{OnlineStochasticShortest}.
Given a transition function $P$ and a policy $\pi$, the occupancy measure $q^{P,\pi} \in [0, 1]^{|X\times A\times X|}$ induced by $P$ and $\pi$ is such that, for every $x \in X_k$, $a \in A$, and $x^{\prime} \in X_{k+1}$ with $k \in [0 \ldots L-1]$:
\begin{equation}
	q^{P,\pi}(x,a,x^{\prime}) \coloneqq \mathbb{P}[x_{k}=x, a_{k}=a,x_{k+1}=x^{\prime}|P,\pi], \label{def:occupancy_measure}
\end{equation}
which represents the probability that, under $P$ and $\pi$, the learner reaches state $x$, plays action $a$, and gets to the next state $x^\prime$.
Moreover, we also define the following quantities:
\begin{align}
	q^{P,\pi}(x,a) \coloneqq \sum_{x^\prime\in X_{k+1}}q^{P,\pi}(x,a,x^{\prime}) \quad \text{and} \quad
	q^{P,\pi}(x) \coloneqq \sum_{a\in A}q^{P,\pi}(x,a). \label{def:q vector}
\end{align}

The following lemma characterizes when a vector $  q \in [0, 1]^{|X\times A\times X|}$ is a \emph{valid} occupancy measure.
\begin{lemma}[\citet{rosenberg19a}]\label{lem:occupancy_rosenberg}
	A vector $  q \in [0, 1]^{|X\times A\times X|}$ is a {valid} occupancy measure of an episodic loop-free CMDP if and only if it satisfies the following conditions:
	\[
	\begin{cases}
		\displaystyle\sum_{x \in X_{k}}\sum_{a\in A}\sum_{x^{\prime} \in X_{k+1}} q(x,a,x^{\prime})=1 & \forall k\in[0\dots L-1]
		\\ 
		\displaystyle\sum_{a\in A}\sum\limits_{x^{\prime} \in X_{k+1}}q(x,a,x^{\prime})=
		\sum_{x^{\prime}\in X_{k-1}} \sum_{a\in A}q(x^{\prime},a,x) & \forall k\in[1\dots L-1], \forall x \in X_{k} 
		\\
		P^{q} = P,
	\end{cases}
	\]
	where $P$ is the transition function of the CMDP and $P^q$ is the one induced by $q$ (see Equation~\eqref{eq:induced_trans}).
\end{lemma}

Notice that any valid occupancy measure $q$ induces a transition function $P^{q}$ and a policy $\pi^{q}$ as:
\begin{equation}P^{q}(x^{\prime}|x,a)= \frac{q(x,a,x^{\prime})}{q(x,a)}   \quad \text{and} \quad \pi^{q}(a|x)=\frac{q(x,a)}{q(x)}.\label{eq:induced_trans}
\end{equation}

\subsection{Performance metrics to evaluate learning algorithms}

In order to define the performance metrics used to evaluate our \emph{online} learning algorithms, we need to introduce an \emph{offline} optimization problem.
Given a CMDP with transition function $P$, we define the following parametric \emph{linear program} (Program~\eqref{lp:offline_opt}), which is parametrized by a reward vector $r \in [0,1]^{|X \times A|}$, a constraint cost matrix $G \in [0,1]^{|X \times A| \times m}$ and a threshold vector ${\alpha} \in [0,L]^m$.
\begin{equation}\label{lp:offline_opt}
		\text{OPT}_{r, G, {\alpha}} \coloneqq \begin{cases}
		\max_{q \in \Delta(P)} & r^{\top}  q \quad \text{s.t.}\\
		 & G^{\top}  q \leq  {\alpha} ,
	\end{cases}
\end{equation}
where $ q\in[0,1]^{|X\times A|}$ is a vector encoding an occupancy measure, whose values are defined for state-action pairs according to Equation~\eqref{def:q vector}, and $\Delta(P)$ is the set of all valid occupancy measures given the transition function $P$ (this set can be encoded by linear constraints thanks to Lemma~\ref{lem:occupancy_rosenberg}).

We say that an instance of Program~\eqref{lp:offline_opt} satisfies \emph{Slater's condition} if the following holds.
\begin{condition}[Slater]\label{cond:slater}
	There exists an occupancy measure $q^\circ\in\Delta(P)$ such that $G^\top q^\circ < \alpha$.
\end{condition}
Moreover, we also introduce a problem-specific \emph{feasibility parameter} related to Program~\eqref{lp:offline_opt}. This is denoted by $\rho\in[0,L]$ and formally defined as
$
	\rho \coloneqq \sup_{q \in \Delta(P)} \min_{i \in [m]} \left[  \alpha-G^\top q \right]_i
$.\footnote{In this paper, given a vector $y$, we denote by $[y]_i$ its $i$-th component.}
Intuitively, $\rho$ represents by how much feasible solutions to Program~\eqref{lp:offline_opt} strictly satisfy the constraints.
%
% represents the minimum satisfaction of the constraints attained by a strictly feasible solution $q^\circ$.
%
Notice that Condition~\ref{cond:slater} is equivalent to say that $\rho > 0$, while, whenever $\rho = 0$, there is no occupancy measure that allows to strictly satisfy the constraints $G^{\top}  q \leq  \alpha$ in Program~\eqref{lp:offline_opt}.
%
%In the second part of the paper, namely, when the corruption is unknown (see Section~\ref{sec: unknown_C}), we will make use of the following assumption.
%\begin{assumption}
%	Condition~\ref{cond:slater} holds, namely, it holds that $\rho>0$.
%\end{assumption}
%\stradi{forse dovremmo dire che useremo la $\rho$ di $\text{OPT}_{\overline{r}, \overline{G}, \underline{\alpha}}$, ma non so dove}

We are now ready to introduce the notion of \emph{(cumulative) regret} and \emph{positive (cumulative) constraint violation}, which are the performance metrics that we use to evaluate our learning algorithm.
In particular, we define the cumulative regret over $T$ episodes as follows:
\begin{equation*}
	R_{T} \coloneqq  T \cdot\text{OPT}_{\overline{r}, \overline{G}, \alpha}-  \sum_{t\in[T]}  \mathbb{E}[r_{t}]^{\top}   q^{P, \pi_{t}}  ,
\end{equation*}
where  $\overline{r} \coloneqq 
\frac{1}{T}\sum_{t=1}^{T}  \mathbb{E}[r_{t}]  $ and $\overline{G} \coloneqq 
\frac{1}{T}\sum_{t=1}^{T}  \mathbb{E}[G_t]$.
%
%$\overline{G}$ is such that
%$\left[\overline{G}\right]_i:= 
%\frac{1}{T}\sum_{t=1}^{T}  \mathbb{E}[g_{t,i}],$ for all $i \in [m]  $.
%
% Notice that the regret is computed with respect to an \emph{optimal safe occupancy in hindsight}.
%
In the following, we denote by $q^*$ an occupancy measure solving Program~\eqref{lp:offline_opt} instantiated with $\overline{r}$, $\overline{G}$, and $\alpha$, while its corresponding policy (computed by Equation~\eqref{eq:induced_trans}) is $\pi^*$.
Thus, $\text{OPT}_{\overline{r}, \overline{G}, \alpha}=\overline{r}^{\top} q^*$ and the regret is $R_T \coloneqq \sum_{t=1}^T \mathbb{E}[r_t]^\top (q^*-q^{P,\pi_t})$.
Furthermore, we define the positive cumulative constraint violation over $T$ episodes as:
\begin{equation*}
	V_{T}:= \max_{i\in[m]}\sum_{t\in[T]}\left[ \mathbb{E}[G_t]^{\top} q^{P,\pi_{t}}-{\alpha}\right]^+_i,
\end{equation*}
where we let $[\cdot]^+:=\max\{0, \cdot\}$.
%
% Notice that, in our definition of $V_T$, constraint violations are \emph{not} allowed to cancel out across episodes. This is a much more demanding performance metric than those employed in~\citep{ding_non_stationary,wei2023provably}, which instead allow for cancellations.
%
In the following, for ease of notation, we compactly refer to $q^{P,\pi_{t}}$ as $q_t$, thus omitting the dependency on $P$ and $\pi$.

\begin{remark}[Relation with adversarial/stochastic CMDPs]
Our setting is more akin to CMDPs with adversarial rewards and constraints, rather than stochastic ones.
This is because our notion of regret is computed with respect to an optimal constraint-satisfying policy in hindsight that takes into account the average over episodes of the mean values $\mathbb{E}[r_{t}]$ and $\mathbb{E}[G_t]$ of the adversarially-selected probability distributions $\mathcal{R}_t$ and $\mathcal{G}_t$.
This makes our setting much harder than one with stochastic rewards and constraints.
Indeed, in the special case in which the supports of $\mathcal{R}_t$ and $\mathcal{G}_t$ are singletons (and, thus, mean values are fully revealed after each episode), our setting reduces to a CMDP with adversarial rewards and constraints, given that such supports are selected adversarially.
%
%The careful reader will notice that our setting employs adversarial baselines (e.g. the optimal occupancy in hindsight), while taking the expectation of the rewards/cost for each episode. This is done since, our setting is much harder than the standard stochastic one: indeed, even ignoring the stochasticity of the samples at each episode -- that is, even if the mean values of reards/costs were revealed at each episode --, our setting would be harder that online stochastic MDPs given the adversarial nature of the mean values.
%
\end{remark}

\begin{remark}[Impossibility results carrying over from adversarial CMDPs]
	\citet{Mannor} show that, in online learning problems with constraints selected adversarially, it is impossible to achieve both regret and constraint violation growing sublinearly in $T$.
	This result holds for a regret definition that corresponds to ours. Thus, it carries over to our setting.
	This is why we look for algorithms whose regret and positive constraint violation scale as $\tilde{\mathcal{O}}(\sqrt{T}+C)$, with a linear dependency on the adversarial corruption $C$.
	Notice that the impossibility result by~\citet{Mannor} does not rule out the possibility of achieving such a guarantee, since regret and positive constraint violation are not sublinear when $C$ grows linearly in $T$, as it could be the case in a classical adversarial setting. 
\end{remark}
%
%Given the impossibility result shown by~\citet{Mannor}, that is, the impossibility to achieve both sublinear\footnote{We say that a performance metric is sublinear in $T$, if it is equal to $o(T)$.} regret and violations when the constraints are chosen adversarially and \emph{employing our baseline}, our goal is to achieve sublinear regret and violations bounds which scale linearly with respect to the adversariality (corruption) of the environment. Precisely, in Section~\ref{sec: known_c}, we propose an algorithm which achieves $\BigOL{\sqrt{T}+C}$ regret and violations when the true corruption value is known.
%Finally, in Section~\ref{sec: unknown_C} we propose an algorithm which builds on top the aforementioned one and attains $\BigOL{\sqrt{T}+C}$ regret and violations when the corruption value is \emph{not} known. Please notice that ou bounds are not sublinear only for $C=\Theta(T)$, thus being coherent with the lower bound of~\citep{Mannor}.
\section{Learning when $C$ is known: More optimism is all you need}
\label{sec: known_c}

We start studying the case in which the learner \emph{knows} the adversarial corruption $C$.
We propose an algorithm (called \texttt{NS-SOPS}, see also Algorithm~\ref{alg:NS_UCSPS}), which adopts a suitably-designed UCB-like approach encompassing the adversarial corruption $C$ in the confidence bounds of rewards and constraint costs.
This effectively results in ``boosting'' the \emph{optimism} of the algorithm, and it allows to achieve regret and positive constraint violation of the order of $\tilde{\mathcal{O}}(\sqrt{T}+C)$.
The \texttt{NS-SOPS} algorithm is also a crucial building block in the design of our algorithm for the case in which the adversarial corruption $C$ is \emph{not} known, as we show in the following section.

\subsection{\texttt{NS-SOPS}: non-stationary safe optimistic policy search}

% \stradi{l'algoritmo è in parte simile a quello del Pirotta, lo diciamo?}

Algorithm~\ref{alg:NS_UCSPS} provides the pseudocode of the \emph{non-stationary safe optimistic policy search} (\texttt{NS-SOPS} for short) algorithm.
The algorithm keeps track of suitably-defined confidence bounds for transition probabilities, rewards, and constraint costs.
At each episode $t \in [T]$, the algorithm builds a confidence set $\mathcal{P}_t$ for the transition function $P$ by following the same approach as~\citet{JinLearningAdversarial2019} (see Appendix~\ref{app:auxiliary} for its definition).
Instead, for rewards and constraint costs, the algorithm adopts novel \emph{enlarged} confidence bounds, which are suitably designed to tackle non-stationarity.
Given $\delta \in (0,1)$, by letting $N_t(x,a)$ be the total number of visits to the state-action pair $(x,a) \in X \times A$ up to episode $t$ (excluded), the confidence bound for the reward $r_t(x,a)$ is:
\begin{equation*}
\phi_t(x,a) \coloneqq \min\left\{1,\sqrt{\frac{\ln\left( \nicefrac{2T|X||A|}{\delta}\right)}{2 \max\{N_t(x,a),1\}} }+\frac{C}{\max\{N_t(x,a),1\}}+\frac{C}{T}\right\},
\end{equation*}
while the confidence bound for the constraint costs $g_{t,i}(x,a)$ is defined as:
\begin{equation*} 
\xi_t(x,a) \coloneqq \min\left\{1,\sqrt{\frac{\ln\left( \nicefrac{2mT|X||A|}{\delta}\right)}{2 \max\{N_t(x,a),1\}} }+\frac{{C}}{\max\{N_t(x,a),1\}}+\frac{{C}}{T}\right\}.
\end{equation*} 
Intuitively, the first term in the expressions above is derived from Azuma-Hoeffding inequality, the second term allows to deal with the non-stationarity of rewards and constraint costs, while the third term is needed to bound how much the average reward vector $\overline{r}$ and the average constraint costs $[\overline{G}]_i$ differ from their ``fictitious'' non-corrupted counterparts $r^\circ$ and $[G^\circ]_i$, respectively.   
%
%bound the deviation from the average reward/cost in hindsight and the non-corrupted one.

Algorithm~\ref{alg:NS_UCSPS} also computes empirical rewards and constraint costs.
At each episode $t \in [T]$, for any state-action pair $(x,a) \in X \times A$ and constraint $i \in [m]$, these are defined as follows:
\begin{equation*}
	\widehat{r}_{t}(x,a) \coloneqq \frac{\sum_{\tau \in [t]}\mathbb{I}_\tau(x,a)r_{\tau}(x,a) }{\max\{N_t(x,a),1\}}\quad \text{and} \quad \widehat{g}_{t,i}(x,a) \coloneqq \frac{\sum_{\tau \in [t]}\mathbb{I}_\tau(x,a)g_{\tau,i}(x,a)}{\max\{N_t(x,a),1\}},
\end{equation*}
where $\mathbb{I}_\tau(x,a)=1$ if and only if $(x,a)$ is visited during episode $\tau$, while $\mathbb{I}_\tau(x,a)=0$ otherwise.
For ease of notation, we let $\widehat{G}_t \in [0,1]^{|X \times A| \times m}$ be the matrix with components $\widehat{g}_{t,i}(x,a) $.
We refer the reader to Appendix~\ref{app:confidence} for all the technical results related to confidence bounds.

\
		\begin{algorithm}[]
			\caption{\texttt{NS-SOPS}}
			\label{alg:NS_UCSPS}
			\begin{algorithmic}[1]
				\Require $C$, $\delta \in (0,1)$
				%				\State Initialize 
				%				$N_0(x,a)=M_0(x'|x,a)=0,~ \xi_0(x,a)=\phi_0(x,a)=\epsilon_0(x'|x,a)=1,$ $ \forall(x,a,x')\in X \times A \times X$ and $\pi_1$ any. \label{alg2: line1}
				\State $\pi_1 \gets$ select any policy
				\For{$t\in[T]$}
				\State Choose policy $\pi_t$ in Algorithm~\ref{alg: Learner-Environment Interaction} and observe feedback from interaction
				% \State Observe trajectory from Algorithm~\ref{alg: Learner-Environment Interaction} \label{alg2: line4}
				% \State Update counters $N_t(x_k,a_k) = N_t(x_k,a_k)+1$ and $M_t(x_{k+1}|x_k,a_k) = M_t(x_{k+1}|x_k,a_k)+1$\label{alg2: line5}
				\State Compute $\mathcal{P}_t$, $\overline{r}_t$, and $\underline{G}_{t} $ \label{alg2: line6}
				%				\State $\underline{G}_{t,i}(x,a) = \widehat{g}_{t,i}(x,a) - \xi_t(x,a)$ , for all $(x,a)\in X \times A,i \in [m]$ \label{alg2: line7}
				%				\State $\overline{r}_{t}(x,a)= \widehat{r}_{t}(x,a) + \phi_t(x,a)$, for all  $(x,a)\in X \times A$ \label{alg2: line8}
				\State $q \gets $ solution to  $\text{\texttt{OPT-CB}}_{\Delta(\mathcal{P}_t),\overline{r}_t,\underline{G}_t, {\alpha}}$
				\If{problem is \emph{feasible}} \label{alg2: line9}
				\State $\widehat{q}_{t+1} \gets q$ \label{alg2: line10}
				\Else 
				\State $\widehat q_{t+1} \gets $ take any $q\in\Delta(\mathcal{P}_t)$ \label{alg2: line12}
				\EndIf
				\State $\pi_{t+1} \gets \pi^{\widehat{q}_{t+1}}$ \label{alg2: line13}
				\EndFor
			\end{algorithmic}
		\end{algorithm}

Algorithm~\ref{alg:NS_UCSPS} selects policies with an UCB-like approach encompassing \emph{optimism} in both rewards and constraints satisfaction, following an approach similar to that employed by~\citet{Exploration_Exploitation}.
Specifically, at each episode $t \in [T]$ and for any  state-action pair $(x,a) \in X \times A$, the algorithm employs an \emph{upper} confidence bound for the reward $r_t(x,a)$, defined as $\overline r_t(x,a) \coloneqq \widehat{r}_{t}(x,a) + \phi_t(x,a)$, while it uses \emph{lower} confidence bounds for the constraint costs $g_{t,i}(x,a)$, defined as $\underline{g}_{t,i}(x,a) \coloneqq \widehat{g}_{t,i}(x,a) - \xi_t(x,a)$ for every constraint $i \in [m]$.
Then, by letting $\overline r_t \in [0,1]^{|X \times A|}$ be the vector with components $\overline r_t(x,a)$ and $\underline{G}_t$ be the matrix with entries $\underline{g}_{t,i}(x,a)$, Algorithm~\ref{alg:NS_UCSPS} chooses the policy to be employed in the next episode $t+1$ by solving the following linear program:
\begin{equation}\label{lp:opt_opt}
	\text{\texttt{OPT-CB}}_{\Delta(\mathcal{P}_t), \overline r_t, \underline{G}_t, \alpha} \coloneqq \begin{cases}
		\argmax_{  q \in \Delta(\mathcal{P}_t)} &   \overline r_t^{\top}  q \quad \text{s.t.}\\
	& \underline G_t^{\top}  q \leq  {\alpha},
	\end{cases}
\end{equation}
where $\Delta(\mathcal{P}_t)$ is the set of all the possible valid occupancy measures given the confidence set $\mathcal{P}_t$ (see Appendix~\ref{app:auxiliary}).
%
%Notice that, due to the optimistic nature of Program~\eqref{lp:opt_opt}, both with respect to the non-stationarity and with respect to the parameter estimation, it is possible to show that $\text{\texttt{OPT}}_{\Delta(\mathcal{P}_t), \overline r_t, \underline{G}_t, \underline{\alpha}}$ admits a solution with high probability, and, precisely, its decision space contains the optimum $q^*$ with the same probability.
%
If $\texttt{OPT-CB}_{\Delta(\mathcal{P}_t), \overline r_t, \underline{G}_t, \alpha}$ is feasible, its solution is used to compute a policy to be employed in the next episode, otherwise the algorithm uses any occupancy measure in the set $\Delta(\mathcal P_t)$.

\subsection{Theoretical guarantees of \texttt{NS-SOPS}}

Next, we prove the theoretical guarantees attained by Algorithm~\ref{alg:NS_UCSPS} (see Appendix~\ref{app:known_c} for complete proofs of the theorems and associated lemmas).
First, we analyze the positive cumulative violation incurred by the algorithm.
Formally, we can state the following result.
%
%; precisely, it is possible to state the following result:
\begin{restatable}{theorem}{violationbound}\label{theo: violationbound}
	Given any $\delta\in(0,1)$, with probability at least $1-8\delta$, Algorithm~\ref{alg:NS_UCSPS} attains:
	\begin{equation*}
		V_T = \mathcal{O}\left(L|X|\sqrt{|A|T\ln\left(\nicefrac{mT|X||A|}{\delta}\right)}+\ln(T)|X||A|C\right).
	\end{equation*}
\end{restatable}
Intuitively, Theorem~\ref{theo: violationbound} is proved by showing that every constraint-satisfying occupancy measure is also feasible for Program~\eqref{lp:opt_opt} with high probability.
This holds since Program~\eqref{lp:opt_opt} employs lower confidence bounds for constraint costs.
Thus, in order to bound $V_T$, it is sufficient to analyze at which rate the feasible region of Program~\eqref{lp:opt_opt} concentrates to the \emph{true} one (\emph{i.e.}, the one defined by $\overline{G}$ in Program~\eqref{lp:offline_opt}).
Since by definition of $\xi_t(x,a)$ the feasibility region of Program~\eqref{lp:opt_opt} concentrates as $1/\sqrt{t}+C/t$, the resulting bound for the positive constraint violation $V_T$ is of the order of $\tilde{\mathcal{O}}(\sqrt{T}+C)$.

The regret guaranteed by Algorithm~\ref{alg:NS_UCSPS} is formalized by the following theorem.
\begin{restatable}{theorem}{regretknownCG}\label{theo:regretknownCG}
	Given any $\delta \in (0,1)$, with probability at least $1-9\delta$, Algorithm~\ref{alg:NS_UCSPS} attains:
	\begin{equation*}
		R_T= \mathcal{O}\left( L|X|\sqrt{|A|T \ln \left(\nicefrac{T|X||A|}{\delta}\right)}+ \ln(T)|X||A|C\right).
	\end{equation*}
\end{restatable}
Theorem~\ref{theo:regretknownCG} is proved similarly to Theorem~\ref{theo: violationbound}.
Indeed, since every constraint-satisfying occupancy measure is feasible for Program~\eqref{lp:opt_opt} with high probability, this also holds for $q^*$, as it satisfies constraints by definition.
Thus, since by definition of $\phi_t(x,a)$ the upper confidence bound for the rewards maximized by Program~\eqref{lp:opt_opt} concentrates as $1/\sqrt{t}+C/t$, the regret bound follows.

\begin{remark}[What if some under/overestimate of $C$ is available]
	We also study what happens if the learner runs Algorithm~\ref{alg:NS_UCSPS} with an under/overestimate on the adversarial corruption as input.
	We defer to Appendix~\ref{app:C_notprecise} all the technical results related to this analysis.
	In particular, it is possible to show that any underestimate on $C$ does not detriment the bound on $V_T$, which remains the one in Theorem~\ref{theo: violationbound}.
	On the other hand, an overestimate on $C$, say $\widehat{C}>C$, results in a bound on $V_T$ of the order of $\mathcal{O}(\sqrt T + \widehat{C})$, which is worse than the one in Theorem~\ref{theo: violationbound}.
	Intuitively, this is because using an overestimate makes Algorithm~\ref{alg:NS_UCSPS} too conservative.
	As a result, one could be tempted to conclude that running Algorithm~\ref{alg:NS_UCSPS} with an underestimate of $C$ as input is satisfactory when the true value of $C$ is unknown.
	However, this would lead to a regret $R_T$ growing linearly in $T$, since, intuitively, a regret-minimizing policy could be cut off from the algorithm decision space.
	This motivates the introduction of additional tools to deal with the case in which $C$ is unknown, as we do in Section~\ref{sec: unknown_C}. 
	%	
	%	It is of interest to study the violation bound of Algorithm~\ref{alg:NS_UCSPS} when it is given as input a guess on the corruption, namely, both an overestimate and an underestimate of the true corruption value. Precisely, it is possible to show that any underestimate of the constraint corruption does not damage the violation bound, that is, the same result as in Theorem~\ref{theo: violationbound} is attained. Differently, given a $\widehat{C}>C$ as input, that is, an overestimate of the corruption, the bound is of order $\mathcal{O}(\sqrt T + \widehat{C})$, thus worsening the bound. This is reasonable, since, with an overestimate of the corruption, the algorithm is too conservative to achieve the optimal bound. For the associated lemmas we refer to Appendix~\ref{app:C_notprecise}.
	%	%
	%	It is important to notice that, given an underestimate of the constraints corruption, we cannot avoid linear regret, since, in principle, the optimal solution could lie in the decision space which is cut out by the underestimate of $C$. This consideration pushed us to resort to different techniques in order to relax the assumption on the corruption knowledge. This is done in Section~\ref{sec: unknown_C}.
\end{remark}

\section{Learning when $C$ is \emph{not} known: A Lagrangified meta-procedure}\label{sec: unknown_C}

In this section, we go beyond Section~\ref{sec: known_c} by studying the more relevant case in which the learner does \emph{not} know the value of the adversarial corruption $C$.
In order to tackle this challenging scenario, we develop a \emph{meta-procedure} (called \texttt{Lag-FTRL}, see Algorithm~\ref{alg:unknown_c}) that instantiates multiple instances of an algorithm working for the case in which $C$ is known, with each instance taking care of a different ``guess'' on the value of $C$.
The \texttt{Lag-FTRL} algorithm is inspired by the work of~\citet{corralling} in the context of classical (unconstrained) multi-armed bandit problems.
Let us remark that \texttt{Lag-FTRL} is a general algorithm that is \emph{not} specifically tailored for our non-stationary CMDP setting.
Indeed, it could be applied to any non-stationary online learning problem with constraints when the adversarial corruption $C$ is unknown, provided that an algorithm working for known $C$ is available.
In this section, to deal with our non-stationary CMDP setting, we let \texttt{Lag-FTRL} instantiate multiple instances of the \texttt{NS-SOPS} algorithm developed in Section~\ref{sec: known_c}.
%
%In practice, having a good estimate of the true corruption value $C$ before interacting with the environment is a very strong limitation. In this section, inspired by the work of~\citet{corralling} on unconstrained MAB, we manage to relax the aforementioned assumption, developing a master algorithm which employs different instances of Algorithm~\ref{alg:UOB-REPS with known $C^G$} in order to attain similar regret and violations guarantees.

\subsection{\texttt{Lag-FTRL}: Lagrangified FTRL}

At a high level, the \emph{Lagrangified follow-the-regularized-leader} (\texttt{Lag-FTRL} for short) algorithm works by instantiating several different instances of Algorithm~\ref{alg:NS_UCSPS}, suitably stabilized (see section \ref{app: stability}), with each instance $\texttt{Alg}^j$ being run for a different ``guess'' of the (unknown) adversarial corruption value $C$.
The algorithm plays the role of a \emph{master} by choosing which instance $\texttt{Alg}^j$ to use at each episode.
The selection is done by employing an FTRL approach with a suitable log-barrier regularization.
In particular, at each episode $t \in [T]$, by letting $\texttt{Alg}^{j_t}$ be the selected instance, the \texttt{Lag-FTRL} algorithm employs the policy $\pi_t^{j_i}$ prescribed by $\texttt{Alg}^{j_t}$ and provides the observed feedback to instance $\texttt{Alg}^{j_t}$ only.

The \texttt{Lag-FTRL} algorithm faces two main challenges.
First, the feedback available to the FTRL procedure implemented at the master level is \emph{partial}.
This is because, at each episode $t \in [T]$, the algorithm only observes the result of using the policy $\pi_t^{j_i}$ prescribed by the chosen instance $\texttt{Alg}^{j_t}$, and \emph{not} those of the policies suggested by other instances.
The algorithm tackles this challenge by employing \emph{optimistic loss estimators} in the FTRL selection procedure, following an approach originally introduced by~\citet{neu}.
The second challenge originates from the fact that the goal of the algorithm is to keep under control both the regret and the positive constraint violation.
This is accomplished by feeding the FTRL procedure with losses constructed starting from the Lagrangian of the offline optimization problem in Program~\eqref{lp:offline_opt}, and suitably modified to manage \emph{positive} violations.
%
%In Algorithm~\ref{alg:unknown_c}, we present the pseudocode of \texttt{Lag-FTRL}, a master algorithm based on an optimistic estimate of the Lagrangian CMDP formulation with positive constraint violation and on a FTRL with log-barrier kind of update. At high-level, Algorithm~\ref{alg:unknown_c} aims at optimizing over different instances of Algorithm~\ref{alg:NS_UCSPS}. Indeed, there are two main challenge to face. First of all, the feedback received by the master algorithm is \emph{partial}, that is, only the loss function associated to the instance chosen at episode $t\in[T]$ is revealed. This challenge is tackled employing an optimistic estimator of the loss function (as in~\citet{neu}), which allows us to guarantees high-probability bounds. Secondly, the objective function of the master algorithm has to take into account both the regret and the violations attains by the subroutines. This is done by properly build a Lagrangian-like loss function for any instance of Algorithm~\ref{alg:NS_UCSPS}, which encompasses both the rewards and positive violations attained by that specific subroutine.

		\begin{algorithm}[]
			%\caption{Lagrangified FTRL with Log-Barrier Regularizer}
			\caption{\texttt{Lag-FTRL}}
			\label{alg:unknown_c}
			\begin{algorithmic}[1]
				% \Require $X,A,T,\Lambda=\frac{Lm+1}{\rho},M= \lceil \log_2 T \rceil, \eta= \sqrt{\frac{\ln(T)}{T}}\frac{1}{2L\Lambda},\gamma=\sqrt{\frac{\ln(M/\delta)}{TM}},\alpha,\delta$
				%
				\Require $\delta \in (0,1)$
				\State $\Lambda \gets \frac{Lm+1}{\rho}$, $M \gets  \lceil \log_2 T \rceil$
				\State $\gamma \gets \sqrt{\nicefrac{\ln(M/\delta)}{TM}}$, \ $\eta \gets $$\frac{1}{2\Lambda m\left(\sqrt{\beta_1T}+\beta_2+\beta_5+\sqrt{\beta_4T}\right)}$%$\min\left\{ \frac{\rho}{4Lm(Lm+1)},\frac{\rho}{4\sqrt{T}\left(2m(Lm+1)\beta_7 +\sqrt{\beta_1}+(Lm+1)\sqrt{\beta_4}\right)}\right\}$ 
				%\stradi{$L<<\sqrt{T}$, altrimenti i nostri risultati non hanno senso, quindi puoi darlo per scontato nelle proof, e lasciare solo il secondo, inoltre va scritto in modo che non appaia $\rho$}
				%$\eta \gets \frac{1}{4Lm\Lambda} \sqrt{\nicefrac{1}{2\ln(T)T}}$
				%
				% \State Initialize $M$ instances of Algorithm \ref{alg:NS_UCSPS},  and initialize each algorithm of index $j \in [M]$ with corruption value $C_j=2^j$ \label{alg3: line1}
				%
				\For{$j \in [M]$}
				\State $\texttt{Alg}^j \gets $ %Stabilized version of 
				stabilized Algorithm~\ref{alg:NS_UCSPS} with $C = 2^j$\label{alg3: line1} %\anna{forse può risultare un po' fuorviante citare qui Algorithm \ref{alg:NS_UCSPS} senza citare il fatto che sia mofificato attraverso Stabilize}
				\EndFor
				\State $w_{1,j} \gets \nicefrac{1}{M}$ for all $j \in [M]$ \label{alg3: line2}
				\For{$t\in[T]$}
				\State Sample index $j_t \sim w_t$ \label{alg3: line4}
				\State $\pi_t^{j_t} \gets $ policy that $\texttt{Alg}^{j_t}$ would choose\label{alg3: line5}
				%
				% \State Observe trajectory as in Algorithm~\ref{alg: Learner-Environment Interaction} and send feedback only to Algorithm $j_t$ \label{alg3: line6}
				%
				\State Choose policy $\pi_t^{j_t}$ in Algorithm~\ref{alg: Learner-Environment Interaction} and observe \textcolor{white}{......}feedback from interaction
				\State Let $\texttt{Alg}^{j_t}$ observe received feedback\label{alg3: line6}
				%
				% \State Build loss estimator as: $$\ell_{t,j} = \frac{\mathbb{I}(j_t=j)}{w_{t,j}+ \gamma}\left(\sum_{(x_k^t,a_k^t)}\left(1-r_t(x_k^t,a_k^t)\right)+ \Lambda\sum_{i \in [m]}\left[\widehat{g}_{t,i}^j{}^\top \widehat{q}_t^j - \alpha_i\right]^+\right) , \quad \forall j \in [M] $$\label{alg3: line7}
				%
				\For{$j \in [M]$}
				\State Build $\ell_{t,j}$ as in Equation~\eqref{eq:loss_estimator}\label{alg3: line7}
				\State Build $b_{t,j}$ as in Equation~\eqref{eq: bonus} \label{alg3: line7_bis}
				\EndFor
				\State $\displaystyle w_{t+1} \gets \argmin_{\underset{w_j\ge\nicefrac{1}{T}}{w \in \Delta_M,}} \, w^\top \sum_{\tau \in [t]}(\ell_t - b_t)+ \frac{1}{\eta}\sum_{j\in[M]} \ln\frac{1}{w_j}$ \label{alg3: line8}
				%
				%\State Update weights as: $$ w_{t+1}= \argmin_{w \in \Delta_M} ~ \left(\sum_{\tau \in [t]}\ell_t\right)^\top w + \frac{1}{\eta}\sum_{j\in[M]} \frac{1}{w_j}$$ \label{alg3: line8}
				\EndFor
			\end{algorithmic}
		\end{algorithm}

The pseudocode of the \texttt{Lag-FTRL} algorithm is provided in Algorithm~\ref{alg:unknown_c}.
At Line~\ref{alg3: line1}, it instantiates $M \coloneqq\lceil \log_2 T \rceil$ instances of  Algorithm~\ref{alg:NS_UCSPS}, with each instance $\texttt{Alg}^j$, for $j \in [M]$, receiving as input a ``guess'' on the adversarial corruption $C=2^j$. Notice that, to every instance of Algorithm~\ref{alg:NS_UCSPS}, a standard doubling trick and a stabilization procedure is applied (see Algorithm~\ref{alg: stabilize} for additional details). This modification to Algorithm~\ref{alg:NS_UCSPS} is necessary to guarantee that each instance $j$ attains a regret and positive cumulative constraints violation which smoothly degrade with $\nu_{T,j}=\nicefrac{1}{\min_{t\in[T]}w_{t,j}}$ and linearly in $C$, when employed by the master algorithm. %....\stradi{mettere una per spiegare a cosa serve stabilize} (see Appendix~\ref{app: stability} for further discussions).
The algorithm assigns weights defining a probability distribution to instances $\texttt{Alg}^j$, with $w_{t,j} \in [0,1]$ denoting the weight of instance $\texttt{Alg}^j$ at episode $t \in [T]$.
We denote by $w_t \in\Delta_M$ the weight vector at episode $t $, with $\Delta_M$ being the $M$-dimensional simplex.
At the first episode, all the weights $w_{1,j}$ are initialized to the value $\nicefrac{1}{M}$ (Line~\ref{alg3: line2}).
Then, at each episode $t\in[T]$, the algorithm samples an instance index $j_t \in [M]$ according to the probability distribution defined by the weight vector $w_t$ (Line~\ref{alg3: line4}), and it employs the policy $\pi_t^{j_t}$ prescribed by $\texttt{Alg}^{j_t}$ (Line~\ref{alg3: line5}).
The algorithm observes the feedback from the interaction described in Algorithm~\ref{alg: Learner-Environment Interaction} and it sends such a feedback to instance $\texttt{Alg}^{j_t}$ (Line~\ref{alg3: line6}).
Then, at Line~\ref{alg3: line7}, the algorithm builds an \emph{optimistic} loss estimator to be fed into each instance $\texttt{Alg}^j$.
In particular, at episode $t \in [T]$ and for every $j\in[M]$, the optimistic loss estimator is defined as:
\begin{equation}\label{eq:loss_estimator}
	\ell_{t,j} \coloneqq \frac{\mathbb{I}(j_t=j)}{w_{t,j}+ \gamma}\Bigg( L - \sum_{k \in [0 \ldots L-1]}r_t(x_k^t,a_k^t)+ \Lambda\sum_{i \in [m]}\left[ \Big( \widehat{G}_{t}^{j} \Big)^\top \, \widehat{q}_t^j - \alpha\right]_i^+\Bigg),
\end{equation}
where $\gamma$ is a suitably-defined implicit exploration factor, $(x_k^t,a_k^t)$ is the state-action pair visited at layer $k$ during episode $t $, $\Lambda$ is a suitably-defined upper bound on the optimal values of Lagrangian multipliers,\footnote{Notice that, in the definition of $\Lambda$, $\rho$ is the feasibility parameter of Program~\eqref{lp:offline_opt} for the reward vector $\overline{r}$, the constraint cost matrix $\overline{G}$, and the threshold vector $\alpha$. In order to compute $\Lambda$, Algorithm~\ref{alg:unknown_c} needs knowledge of $\rho$. Nevertheless, our results continue to hold even if Algorithm~\ref{alg:unknown_c} is only given access to a lower bound on $\rho$.}  $\widehat{G}_{t}^j$ is the matrix of empirical constraint costs built by the instance $\texttt{Alg}^j$ of Algorithm~\ref{alg:NS_UCSPS} at episode $t$, while $\widehat q_t^j$ is the occupancy measure computed by instance $\texttt{Alg}^j$ of Algorithm~\ref{alg:NS_UCSPS} at $t$.
Finally, the algorithm updates the weight vector according to an FTRL update on a cut decision space with a suitable log-barrier regularization and a bonus term $b_t$ defined as:
\begin{equation}
	\label{eq: bonus}
b_{t,j}\coloneq  \left(\left(m\Lambda\beta_5+\beta_2\right) + \left(\sqrt{\beta_1} + m\Lambda\sqrt{\beta_4}\right)\sqrt{T}\right)(\nu_{t,j}-\nu_{t-1,j}),,
\end{equation}
where $\nu_{t,j}=\max_{\tau \le t} \frac{1}{w_{\tau,j}}$ and the parameters $\beta$ are linked to the performance of Algorithm \ref{alg:NS_UCSPS} (see Line~\ref{alg3: line7_bis} and Section~\ref{app: stability param} for additional details).
See Line~\ref{alg3: line8} for the complete definition of the update. The bonus term purpose is to balance out the term related to the difference between the performance of Algorithm \ref{alg:NS_UCSPS} updated at each episode and the performance of its stabilized version, which works under the condition imposed by the master algorithm.

\subsection{Theoretical guarantees of \texttt{Lag-FTRL}}

Next, we prove the theoretical guarantees attained by Algorithm~\ref{alg:unknown_c} (see Appendix~\ref{app:unknown_C} for complete proofs of the theorems and associated lemmas).
As a first preliminary step, we extend the well-known strong duality result for CMDPs~\citep{Altman1999ConstrainedMD} to the case of bounded Lagrangian multipliers.
\begin{restatable}{lemma}{strongDuality}\label{lemma:auxiliary dual}
	Given a CMDP with a transition function $P$, for every reward vector $r \in [0,1]^{|X \times A|}$, constraint cost matrix $G \in [0,1]^{|X \times A| \times m}$, and threshold vector $\alpha \in [0,L]^m$, if Program~\eqref{lp:offline_opt} satisfies Slater's condition (Condition~\ref{cond:slater}), then the following holds:
	%
	% Given a reward vector $r$, a cost matrix $G$ and the thresholds vector $\alpha$, if Slater's condition holds with parameter $\rho$ for Program~\eqref{lp:offline_opt}, for $\lambda\in\mathbb{R}^m_{\geq0}$, it holds:
	%
	\begin{align*}
		\min_{\lVert \lambda \rVert_1\in [0,\nicefrac{L}{\rho}]} \max_{  q \in \Delta(P)} r^\top q - \sum_{i \in [m]}\lambda_i\left[ {G}^\top q-\alpha \right]_i & =   \max_{  q \in \Delta(P)} \min_{\lVert \lambda \rVert_1\in [0,\nicefrac{L}{\rho}]}r^\top q - \sum_{i \in [m]}\lambda_i \left[ {G}^\top q-\alpha \right]_i  \\&= \textnormal{OPT}_{r,G,\alpha},
	\end{align*}
	where $\lambda\in\mathbb{R}^m_{\geq0}$ is a vector of Lagrangian multipliers and $\rho$ is the feasibility parameter of Program~\eqref{lp:offline_opt}.
\end{restatable}
Intuitively, Lemma~\ref{lemma:auxiliary dual} states that, under Slater's condition, strong duality continues to hold even when restricting the set of Lagrangian multipliers to the $\lambda\in\mathbb{R}^m_{\geq0}$ having $\lVert \lambda \rVert_1$ bounded by $\nicefrac{L}{\rho}$. 
Furthermore, we extend the result in Lemma~\ref{lemma:auxiliary dual} to the case of a Lagrangian function suitably-modified to encompass \emph{positive} violations.
We call it \emph{positive Lagrangian} of Program~\eqref{lp:offline_opt}, defined as follows.
\begin{definition}[Positive Lagrangian]
	Given a CMDP with a transition function $P$, for every reward vector $r \in [0,1]^{|X \times A|}$, constraint cost matrix $G \in [0,1]^{|X \times A| \times m}$, and threshold vector $\alpha \in [0,L]^m$, the \emph{positive Lagrangian} of Program~\eqref{lp:offline_opt} is defined as a function $\mathcal{L} : \mathbb{R}_+ \times \Delta(P) \to \mathbb{R}$ such that it holds $\mathcal{L}\left(\beta,q\right) \coloneqq r^\top q - \beta \sum_{i \in [m]}\left[{G}^\top q-\alpha\right]_i^+$ for every $\beta \geq 0$ and $q\in \Delta(P).$
	%
	%	Given a reward vector $r$, a cost matrix $G$ and the thresholds vector $\alpha$, we define the positive Lagrangian function for Program~\eqref{lp:offline_opt} as
	%		$\mathcal{L}\left(\Sigma,q\right)= r^\top q - \Sigma\sum_{i \in [m]}\left[{G}_{i}^\top q-\alpha_i\right]^+.$
\end{definition}
The positive Lagrangian is related to the Lagrangian of a variation of Program~\eqref{lp:offline_opt} in which the $[\cdot]^+$ operator is applied to the constraints.
Notice that such a problem does \emph{not} admit Slater's condition, since, by definition of $[\cdot]^+$, it does \emph{not} exist an occupancy measure $q^\circ$ such that $\left[{G}^\top q^\circ-\alpha\right]_i^+ < 0$ for every $i\in[m]$.
Nevertheless, we show that a kind of strong duality result still holds for $\mathcal{L}(\nicefrac{L}{\rho},q)$, when Slater's condition is met by Program~\eqref{lp:offline_opt}. This is done in the following result.
%
% the offline optimization problem defining the positive Lagrangian does \emph{not} admit Slater's condition, since, by definition of the $[\cdot]^+$ operator, it does not exist an occupancy measure such that $q^\circ$ s.t. $\left[{G}^\top q^\circ-\alpha\right]_i^+\leq 0$ for any $i\in[m]$. Nevertheless, we show that a strong duality-kind of result still holds for $\mathcal{L}(\nicefrac{L}{\rho},q)$, when Slater's condition holds for the optimization problem when the $[\cdot]^+$ operator is not applied to the constraints, namely, Program~\eqref{lp:offline_opt}. This is done in the following result.
%
% Thus, we are ready to extend the previous result to the positive Lagrangian.
%
\begin{restatable}{theorem}{positiveStrongDuality}\label{theo: strong duality}
		Given a CMDP with a transition function $P$, for every reward vector $r \in [0,1]^{|X \times A|}$, constraint cost matrix $G \in [0,1]^{|X \times A| \times m}$, and threshold vector $\alpha \in [0,L]^m$, if Program~\eqref{lp:offline_opt} satisfies Slater's condition (Condition~\ref{cond:slater}), then the following holds:
		%
		% Given a reward vector $r$, a cost matrix $G$ and the thresholds vector $\alpha$, if Slater's condition holds with parameter $\rho$ for Program~\eqref{lp:offline_opt}, it holds:
		%
		\begin{align*}
			\max_{q \in \Delta(P)}\mathcal{L}(\nicefrac L \rho,q) = \max_{q \in \Delta(P)} r^\top q - \frac{L}{\rho}\sum_{i \in [m]}\left[{G}^\top q-\alpha \right]_i^+  = \textnormal{OPT}_{r,G,\alpha},
		\end{align*}
		where $\rho$ is the feasibility parameter of Program~\eqref{lp:offline_opt}.
\end{restatable}
Theorem~\ref{theo: strong duality} intuitively shows that a $L/\rho$ multiplicative factor on the positive constraint violation is enough to compensate the large rewards that non-feasible policies would attain when employed by the learner. 
This result is crucial since, without properly defining the Lagrangian function optimized by Algorithm~\ref{alg:unknown_c}, the FTRL optimization procedure would choose instances with both large rewards and large constraint violation, thus preventing the violation bound from being sublinear.

By means of Theorem~\ref{theo: strong duality}, it is possible to provide the following result.
\begin{restatable}{theorem}{ViolationUnknown}
	\label{theo:ViolationUnknown}
	If Program~\eqref{lp:offline_opt} instantiated with $\overline{r}$, $\overline{G}$ and $\alpha$ satisfies Slater's condition (Condition~\ref{cond:slater}), then, given any $\delta \in (0,1)$, Algorithm \ref{alg:unknown_c} attains the following bound with probability at least $1-34\delta$:
	\begin{align*}
		V_T=\mathcal{O}\bigg(m^2L^2|X|&\sqrt{|A|T\log\left(\nicefrac{mT|X||A|}{\delta}\right)}\log(T)^2 \\&+ m^2L|X|^2|A|^2\log(T)^3 \log\left(\nicefrac{\log(T)}{\delta}\right)+ m^2L\log(T)^2|X||A|C\bigg).
	\end{align*}
\end{restatable}
Intuitively, to prove Theorem~\ref{theo:ViolationUnknown}, it is necessary to bound the negative regret attained by the algorithm, \emph{i.e.}, how better Algorithm~\ref{alg:unknown_c} can perform in terms of rewards with respect to an optimal occupancy in hindsight $q^*$. Notice that this is equivalent to showing that the FTRL procedure cannot gain more than $\textnormal{OPT}_{\overline r,\overline G,\alpha}$ by playing policies that are \emph{not} feasible, or, equivalently, by choosing instances $\texttt{Alg}^j$ with a large corruption guess, which, by definition of the confidence sets employed by Algorithm~\ref{alg:NS_UCSPS}, may play non-feasible policies attaining large rewards.
This is done by employing Theorem~\ref{theo: strong duality}, which shows that the positive Lagrangian does \emph{not} allow the algorithm to achieve too large rewards with respect to $q^*$.
Thus, the violations are still upper bounded by $\tilde{\mathcal{O}} (\sqrt{T}+C)$.

Finally, we prove the regret bound attained by Algorithm~\ref{alg:unknown_c}.
\begin{restatable}{theorem}{RegretUnknown}\label{Theo:RegretUnknown}
	If Program~\eqref{lp:offline_opt} instantiated with $\overline{r}$, $\overline{G}$ and $\alpha$ satisfies Slater's condition (Condition~\ref{cond:slater}), then, given any $\delta \in (0,1)$, Algorithm \ref{alg:unknown_c} attains the following bound with probability at least $1-30\delta$:
	\begin{align*}
		R_T  = \mathcal{O}\bigg(m^2L^2|X|&\sqrt{|A|T\log\left(\nicefrac{mT|X||A|}{\delta}\right)}\log(T)^2 \\&+ m^2L|X|^2|A|^2\log(T)^3 \log\left(\nicefrac{\log(T)}{\delta}\right)+ m^2L\log(T)^2|X||A|C\bigg).
	\end{align*}
\end{restatable}
Bounding the regret attained by Algorithm~\ref{alg:unknown_c} requires different techniques with respect to bounding constraint violation. Indeed, strong duality is \emph{not} needed, since, even if $\Lambda$ is set to a too small value and thus the algorithm plays non-feasible policies, then the regret would still be sublinear. 
The regret bound is strongly related to the optimal value of the problem associated with the positive Lagrangian, which, by definition of $[\cdot]^+$ cannot perform worse than the optimum of Program~\eqref{lp:offline_opt}, in terms of rewards gained.
Thus, by letting $j^*$ be the index of the instance associated with true corruption value $C$, proving Theorem~\ref{Theo:RegretUnknown} reduces to bounding the regret and the constraint violation of instance $\texttt{Alg}^{j^*}$, with the additional challenge of bounding the estimation error of the optimistic loss estimator. Finally, by means of the results for the \emph{known} $C$ case derived in Section~\ref{sec: known_c}, we are able to show that the regret is at most $\tilde{\mathcal{O}} (\sqrt{T}+C)$, which is the desired bound.

\bibliography{example_paper}
\bibliographystyle{plainnat}

%\printbibliography
%%%%%%%%%%%%%%%%%%%%%%%%%%%%%%%%%%%%%%%%%%%%%%%%%%%%%%%%%%%%%%%%%%%%%%%%%%%%%%%
%%%%%%%%%%%%%%%%%%%%%%%%%%%%%%%%%%%%%%%%%%%%%%%%%%%%%%%%%%%%%%%%%%%%%%%%%%%%%%%
% APPENDIX
%%%%%%%%%%%%%%%%%%%%%%%%%%%%%%%%%%%%%%%%%%%%%%%%%%%%%%%%%%%%%%%%%%%%%%%%%%%%%%%
%%%%%%%%%%%%%%%%%%%%%%%%%%%%%%%%%%%%%%%%%%%%%%%%%%%%%%%%%%%%%%%%%%%%%%%%%%%%%%%

\newpage

\appendix
%\onecolumn

\appendix

\section*{Appendix}
The appendix is structured as follows:
\begin{itemize}
	\item In Appendix~\ref{app:related} we provide the complete related works.
	\item In Appendix~\ref{app:event} we provide the events dictionary.
	\item In Appendix~\ref{app:confidence} we provide the preliminary results on the confidence sets employed to estimate the unknown parameters of the environment.
	\item In Appendix~\ref{app:known_c} we provide the omitted proofs related to the theoretical guarantees when the corruption value is known by the learner, namely, the results attained by Algorithm~\ref{alg:NS_UCSPS}.
	\item In Appendix~\ref{app:C_notprecise} we provide the omitted proofs of the theoretical guarantees attained by Algorithm~\ref{alg:NS_UCSPS}, when a guess on the corruption is given as input to the algorithm.
	\item In Appendix~\ref{app:unknown_C} we provide the omitted proofs related to the theoretical guarantees when the corruption value is \emph{not} known by the learner, namely, the results attained by Algorithm~\ref{alg:unknown_c}.
	\item In Appendix~\ref{app:auxiliary} we restate useful results from existing works.
	\item In Appendix~\ref{app: stability} we provide the results related to stability a corruption-robustness.
\end{itemize}

\section{Related works}
\label{app:related}
In the following, we discuss some works that are tightly related to ours.
In particular, we first describe works dealing with the online learning problem in MDPs, and, then, we discuss some works studying the constrained version of the classical online learning problem.

\paragraph{Online learning in MDPs}
The literature on online learning problems~\citep{cesa2006prediction} in MDPs is wide (see~\citep{Near_optimal_Regret_Bounds,even2009online,neu2010online} for some initial results on the topic).
In such settings, two types of feedback are usually studied: in the \textit{full-information feedback} model, the entire loss function is observed after the learner's choice, while in the \textit{bandit feedback} model, the learner only observes the loss due to the chosen action. 
\citet{Minimax_Regret} study the problem of optimal exploration in episodic MDPs with unknown transitions and stochastic losses when the feedback is bandit.
The authors present an algorithm whose regret upper bound is $\Tilde{\mathcal{O}}(\sqrt{T})$, 
%where $T$ is the number of episodes, 
thus matching the lower bound for this class of MDPs and improving the previous result by \citet{Near_optimal_Regret_Bounds}.

\paragraph{Online learning in non-stationary MDPs} 
The literature on non-stationary MDPs encompasses both works on non-stationary rewards and non-stationary transitions. As concerns the first research line,
\citet{rosenberg19a} study the online learning problem in episodic MDPs with adversarial losses and unknown transitions when the feedback is full information. The authors present an online algorithm exploiting entropic regularization and providing a regret upper bound of $\Tilde{\mathcal{O}}(\sqrt{T})$.
The same setting is investigated by \citet{OnlineStochasticShortest} when the feedback is bandit. In such a case, the authors provide a regret upper bound of the order of $\Tilde{\mathcal{O}}(T^{3/4})$, which is improved by
\citet{JinLearningAdversarial2019} by providing an  algorithm that achieves in the same setting a regret upper bound of $\Tilde{\mathcal{O}}(\sqrt{T})$. Related to the non-stationarity of the transitions , \citet{wei2022corruption} study MDPs with adversarial corruption on  transition functions and rewards, reaching a regret upper bound of order $\widetilde{\mathcal{O}}(\sqrt{T}+C)$ (where $C$ is the amount of adversarial corruption) with respect to the optimal policy of the non-corrupted MDP .  
Finally, \citet{jin2024no} is the first to study completely adversarial MDPs with changing transition functions, providing a $\Tilde{\mathcal{O}}(\sqrt{T}+C)$ regret bounds, where $C$ is a corruption measure of the adversarially changing transition functions.

\paragraph{Online learning with constraints} 
A central result is provided by \citet{Mannor}, who show that it is impossible to suffer from sublinear regret and sublinear constraint violation when an adversary chooses losses and constraints. 
\citet{liakopoulos2019cautious} try to overcome such an impossibility result by defining a new notion of regret. They study a class of online learning problems with long-term budget constraints that can be chosen by an adversary.
The learner’s regret metric is modified by introducing the notion of a \textit{K-benchmark}, \emph{i.e.}, a comparator that meets the problem’s allotted budget over any window of length $K$. 
\citet{castiglioni2022online, Unifying_Framework} deal with the problem of online learning with stochastic and adversarial losses, providing the first \textit{best-of-both-worlds} algorithm for online learning problems with long-term constraints.

\paragraph{Online learning in CMDPs}
Online Learning In MDPs with constraints is generally studied when the constraints are selected stochastically. Precisely, \citet{Constrained_Upper_Confidence} deal with episodic CMDPs with stochastic losses and constraints, where the transition probabilities are known and the feedback is bandit. The regret upper bound of their algorithm is of the order of $\Tilde{\mathcal{O}}(T^{3/4})$, while the cumulative constraint violation is guaranteed to be below a threshold with a given probability.
\citet{Online_Learning_in_Weakly_Coupled} deal with adversarial losses and stochastic constraints, assuming the transition probabilities are known and the feedback is full information. The authors present an algorithm that guarantees an upper bound of the order of $\Tilde{\mathcal{O}}(\sqrt{T})$ on 
both regret and constraint violation. 
\citet{bai2020provably} provide the first algorithm that achieves sublinear regret when the transition probabilities are unknown, assuming that the rewards are deterministic and the constraints are stochastic with a particular structure.
\citet{Exploration_Exploitation} propose two approaches
%(one based on linear programming and the other on a primal-dual method) 
to deal with the exploration-exploitation dilemma in episodic CMDPs. These approaches guarantee sublinear regret and constraint violation when transition probabilities, rewards, and constraints are unknown and stochastic, while the feedback is bandit. 
\citet{Upper_Confidence_Primal_Dual} provide a primal-dual approach based on \textit{optimism in the face of uncertainty}. This work shows the effectiveness of such an approach when dealing with episodic CMDPs with adversarial losses and stochastic constraints, achieving both sublinear regret and constraint violation with full-information feedback.
\cite{stradi2024learning} is the first work to tackle CMDPs with adversarial losses and bandit feedback. They propose an algorithm which achieves sublinear regret and sublinear positive constraints violations, assuming that the constraints are stochastic.
\citet{germano2023bestofbothworlds} are the first to study CMDPs with adversarial constraints. Given the well-known impossibility result to learn with adversarial constraints, they propose an algorithm that attains sublinear violation (with cancellations allowed) and a fraction of the optimal reward when the feedback is full.
Finally, \citet{ding_non_stationary}~and~\citet{wei2023provably} consider the case in which rewards and constraints are non-stationary, assuming that their variation is bounded, as in our work. Nevertheless, our settings differ in multiple aspects. First of all, we consider positive constraints violations, while the aforementioned works allow the cancellations in their definition.
We consider a static regret adversarial baseline, while \citet{ding_non_stationary}~and~\citet{wei2023provably} consider the stronger baseline of dynamic regret. Nevertheless, our bounds are not comparable, since we achieve linear regret and violations only in the worst case scenario in which $C=T$, while a sublinear corruption would lead to linear dynamic regret in their work. Finally, we do not make any assumption on the number of episodes, while both the regret and violations bounds presented in~\citet{wei2023provably} hold only for large $T$.

\section{Events dictionary}
\label{app:event}
In the following, we introduce the main events which are related to estimation of the unknown stochastic parameters of the environment.
\begin{itemize}
	\item \textbf{Event} $\mathcal{E}_{P}$: 
		for all $t\in [T], P \in \mathcal{P}_t$.  
		$\mathcal{E}_{P}$ holds with probability at least $1-4\delta$ by Lemma~\ref{lem:prob_int_jin}. The event is related to the estimation of the unknown transition function.
	\item \textbf{Event} $\mathcal{E}_G$: for all $ t \in [T], i \in [m],(x,a)\in X\times A$:
	\begin{equation*}
		\bigg\lvert \widehat{g}_{t,i}(x,a) - \frac{1}{T} \underset{\tau \in [T]}{\sum}\mathbb{E}[g_{\tau,i}(x,a)]\bigg\rvert \le \xi_t(x,a).
	\end{equation*}
	Similarly,
	\begin{equation*}
		\bigg \lvert\widehat{g}_{t,i}(x,a)-g_i^\circ(x,a)\bigg \rvert \le \xi_t(x,a),
	\end{equation*}
	where $g_i^\circ\in[0,1]^{|X\times A|}\coloneq [G^\circ]_{i}$.
	
	$\mathcal{E}_G$ holds with probability at least $1-\delta$
	by Corollary~\ref{cor: CI for G}. The event is related to the estimation of the unknown constraint functions.
	\item \textbf{Event} $\mathcal{E}_r$: for all $t\in[T], (x,a)\in X\times A$:
	\begin{equation*}
		\bigg\lvert \widehat{r}_{t}(x,a) - \frac{1}{T} \underset{\tau \in [T]}{\sum}\mathbb{E}[r_{\tau}(x,a)]\bigg\rvert \le \phi_t(x,a).
	\end{equation*}
	Similarly,
	\begin{equation*}
		\bigg \lvert\widehat{r}_{t}(x,a)-r^\circ(x,a)\bigg \rvert \le \phi_t(x,a).
	\end{equation*}
	$\mathcal{E}_r$ holds with probability at least $1-\delta$
	by Corollary~\ref{cor: CI for r}. The event is related to the estimation of the unknown reward function.
	\item \textbf{Event} $\mathcal{E}_{\widehat q \ }$: 	for any $P_t^x \in \mathcal{P}_t$:
	\begin{equation*}
		\sum_{t\in [T]} \sum_{x \in X, a \in A}\left|q^{P_t^x,\pi_t}(x, a)- q_t(x, a)\right| \leq \mathcal{O}\left(L|X| \sqrt{|A| T \ln \left(\frac{T|X||A|}{\delta}\right)}\right).
	\end{equation*}
 $\mathcal{E}_{\widehat q \ }$ holds with probability at least $ 1-6\delta$
	by Lemma \ref{lem:transition_jin}. The event is related to the convergence to the true unknown occupancy measure. Notice that $\mathbb{P}\left[\mathcal{E}_{\widehat q \ }\cap \mathcal{E}_{P}\right]\ge 1-6\delta$ by construction.
	
\end{itemize}

\section{Confidence intervals}
\label{app:confidence}

In this section we will provide the preliminary results related to the high probability confidence sets for the estimation of the cost constraints matrices and the reward vectors.

We start  bounding the distance between the \emph{non-corrupted} costs and rewards with respect to the mean of the adversarial distributions.
\begin{lemma}
	\label{lemma:aux G Gmean and r}
	For all $i \in [m]$, fixing $(x,a)\in X \times A$, it holds:
	\begin{equation*}
		\bigg\lvert {g}_i^\circ(x,a)-\frac{1}{T}\underset{t \in [T] }{\sum}\mathbb{E}[g_{t,i}(x,a)] \bigg\rvert \le \frac{C_G}{T}.
	\end{equation*}
	Similarly, fixing $(x,a)\in X \times A$, it holds:
	\begin{equation*}
		\bigg\lvert {r}^\circ(x,a)-\frac{1}{T}\underset{t \in [T] }{\sum}\mathbb{E}[r_{t}(x,a)] \bigg\rvert \le \frac{C_r}{T},
	\end{equation*}
\end{lemma}
\begin{proof}
	By triangle inequality and from the definition of $C_G$, it holds:
	\begin{align*}
		\bigg|g^\circ_i(x,a)-\frac{1}{T}\underset{t \in [T]}{\sum}\mathbb{E}[g_{t,i}(x,a)] \bigg| & = \bigg|\frac{1}{T}\underset{t \in [T]}{\sum}(g^\circ_i(x,a)-\mathbb{E}[g_{t,i}(x,a)] )\bigg| \\
		&\le \frac{1}{T}\underset{t \in [T]}{\sum}\bigg|g^\circ_i(x,a)-\mathbb{E}[g_{t,i}(x,a)]\bigg| \\& \le \frac{C_G}{T}.
	\end{align*}
	Notice that the proof holds for all $i \in [m]$ since  $C_G$ is defined employing the maximum over $i \in [m]$.
	Following the same steps, it holds:
	\begin{align*}
		\left|r^\circ(x,a)-\frac{1}{T}\underset{t \in [T]}{\sum}\mathbb{E}[r_{t}(x,a)] \right| & = \left|\frac{1}{T}\underset{t \in [T]}{\sum}(r^\circ(x,a)-\mathbb{E}[r_{t}(x,a)] )\right| \\
		&\le \frac{1}{T}\underset{t \in [T]}{\sum}\bigg|r^\circ(x,a)-\mathbb{E}[r_{t}(x,a)]\bigg| \\& \le \frac{C_r}{T},
	\end{align*}
which concludes the proof.
\end{proof}
In the following lemma, we bound the distance between the empirical mean of the constraints function and the true \emph{non-corrupted} value.
\begin{lemma}
	\label{lemma:auxCIG1}
	Fixing $i \in [m]$, $(x,a) \in X\times A$ , $t \in [T]$, for any $\delta\in (0,1)$, it holds with probability at least $1-\delta$:
	\begin{equation*}
		\bigg \lvert\widehat{g}_{t,i}(x,a)-g^\circ_i(x,a)\bigg \rvert \le \sqrt{\frac{1}{2\max\{N_t(x,a),1\}}\ln\left(\frac{2}{\delta} \right) }+ \frac{C_G}{\max\{N_t(x,a),1\}}.
	\end{equation*}
\end{lemma}

\begin{proof}
	We start bounding the quantity of interest as follows:
	\begin{align}
		\bigg|\widehat{g}_{t,i}(x, a)-g^\circ_i(x,a)\bigg|&=
		\left|\left(\frac{\sum_{\tau \in [t]}\mathbb{I}_\tau(x,a)g_{\tau,i}(x, a)}{\max\{N_t(x,a),1\}} \right)- g^\circ_i(x,a)\right| \nonumber\\
		& \le \left|\frac{1}{\max\{N_t(x,a),1\}}\underset{\tau \in [t]}{\sum}\mathbb{I}_\tau(x,a)\left(g_{\tau,i}(x, a) - \mathbb{E}[g_{\tau,i}(x,a)]\right)\right| \nonumber\\ 
		&\ \ +\left|\frac{1}{\max\{N_t(x,a),1\}}\underset{\tau \in [t]}{\sum}\mathbb{I}_\tau(x,a)[\mathbb{E}[g_{\tau,i}(x,a)]-g^\circ_i(x,a)]\right|,\label{lemauxCIG1:eq1}
	\end{align}
	where we employed the triangle inequality and  the definition of $\widehat{g}_{t,i}(x,a)$.
	
	We bound the two terms in Equation~\eqref{lemauxCIG1:eq1} separately.
	For the first term, by Hoeffding's inequality and noticing that constraints values are bounded in $[0,1]$, it holds that:
	\begin{equation*}
		\mathbb{P}\left[\mathcal{A}\geq \frac{c}{\max\{N_t(x,a),1\}}\right] 
		 \leq 2 \exp\Bigg(-\frac{2c^2}{\max\{N_t(x,a),1\}}\Bigg),
	\end{equation*}
	where,
	\begin{equation*}
		\mathcal{A}=\left|\left(\frac{\sum_{\tau \in [t]}\mathbb{I}_\tau(x,a)g_{\tau,i}(x, a)}{\max\{N_t(x,a),1\}} \right) - \left( \frac{\sum_{\tau \in [t]}\mathbb{I}_\tau(x,a) \mathbb{E}[g_{\tau,i}(x,a)]}{\max\{N_t(x,a),1\}}\right)\right|,
	\end{equation*}
	Setting $\delta=2 \exp\left(-\frac{2c^2}{\max\{N_t(x,a),1\}}\right)$ and solving to find a proper value of $c$ we get that with probability at least $1-\delta$:
	\begin{equation*}
		\left|\frac{1}{\max\{N_t(x,a),1\}}\underset{\tau \in [t]}{\sum}\mathbb{I}_\tau(x,a)\left(g_{\tau,i}(x, a) - \mathbb{E}[g_{\tau,i}(x,a)]\right)\right| \le \sqrt{\frac{1}{2\max\{N_t(x,a),1\}}\ln\left(\frac{2}{\delta} \right) }.
	\end{equation*}
	Finally, we focus on the second term. Thus, employing the triangle inequality and the definition of $C_G$, it holds:
	\begin{align*}
		\Bigg|\frac{1}{\max\{N_t(x,a),1\}}\underset{\tau \in [t]}{\sum}\mathbb{I}_\tau(x,a)&\left[\mathbb{E}[g_{\tau,i}(x,a)]-g^\circ_i(x,a)\right]\Bigg| \\ &\le \frac{1}{\max\{N_t(x,a),1\}}\underset{\tau \in [t]}{\sum}\mathbb{I}_\tau(x,a)\bigg|\mathbb{E}[g_{\tau,i}(x,a)]-g^\circ_i(x,a)\bigg| \\
		&\le \frac{1}{\max\{N_t(x,a),1\}}\underset{\tau \in [T]}{\sum}\bigg|\mathbb{E}[g_{\tau,i}(x,a)]-g^\circ_i(x,a)\bigg| \\&\le \frac{C_G}{\max\{N_t(x,a),1\}},
	\end{align*}
	which concludes the proof.
\end{proof}
We now prove a similar result for the rewards function.
\begin{lemma}
	\label{lemma:auxCIr1}
	Fixing $(x,a) \in X\times A$ , $t \in [T]$, for any $\delta\in (0,1)$, it holds with probability at least $1-\delta$:
	\begin{equation*}
		\bigg \lvert\widehat{r}_{t}(x,a)-r^\circ(x,a)\bigg \rvert \le \sqrt{\frac{1}{2\max\{N_t(x,a),1\}}\ln\left(\frac{2}{\delta} \right) }+ \frac{C_r}{\max\{N_t(x,a),1\}}.
	\end{equation*}
\end{lemma}

\begin{proof}
	The proof is analogous to the one of Lemma~\ref{lemma:auxCIG1}.
\end{proof}
We now generalize the previous results as follows. %\stradi{in questi risultati va detto che gli intervalli di confidenza che abbiamo usato sono upper-boundati da 1, un po'come avevo fatto io nel mio paper.}
\begin{lemma}
	\label{lemma:auxCIG2}
	Given any $\delta \in (0,1)$, for any $(x,a)\in X \times A, t \in [T],$ and $ i \in [m]$, it holds with probability at least $1-\delta$:
	\begin{equation*}
		\bigg \lvert\widehat{g}_{t,i}(x,a)-g^\circ_i(x,a)\bigg \rvert \le \sqrt{\frac{1}{2\max\{N_t(x,a),1\}}\ln\left(\frac{2mT|X||A|}{\delta} \right) }+ \frac{C_G}{\max\{N_t(x,a),1\}}.
	\end{equation*}
\end{lemma}
\begin{proof}
	First let's define $\zeta_t(x,a)$ as:
	\begin{equation*}
		\zeta_t(x,a):= \sqrt{\frac{1}{2\max\{N_t(x,a),1\}}\ln\left(\frac{2}{\delta} \right) }+ \frac{C_G}{\max\{N_t(x,a),1\}}.
	\end{equation*}
	From Lemma~\ref{lemma:auxCIG1}, given $\delta'\in(0,1)$, we have, fixed any $i\in[m]$, $t\in[T]$ and $(x,a)\in X \times A$:
	\begin{equation*}
		\mathbb{P}\Bigg[\bigg \lvert\widehat{g}_{t,i}(x,a)-g^\circ_i(x,a)\bigg \rvert \le \zeta_t(x,a)\Bigg]\geq 1-\delta'.
	\end{equation*}
	Now, we are interested in the intersection of all the events, namely,
	\begin{equation*}
		\mathbb{P}\Bigg[\bigcap_{x,a,i,t}\Big\{\Big|\widehat{g}_{t,i}(x,a)-g^\circ_i(x,a)\Big|\leq \zeta_t(x,a)\Big\}\Bigg].
	\end{equation*}
	Thus, we have:
	\begin{align}
		\mathbb{P}\Bigg[\bigcap_{x,a,i,t}\Big\{\Big|\widehat{g}_{t,i}(x,a)-g^\circ_i(x,a)\nonumber&\Big|\leq \zeta_t(x,a)\Big\}\Bigg]\\&= 1 - \mathbb{P}\Bigg[\bigcup_{x,a,i,t}\Big\{\Big|\widehat{g}_{t,i}(x,a)-g^\circ_i(x,a)\Big|\leq \zeta_t(x,a)\Big\}^c\Bigg]\nonumber\\
		&\geq 1 - \sum_{x,a,i,t}\mathbb{P}\Bigg[\Big\{\Big|\widehat{g}_{t,i}(x,a)-g^\circ_i(x,a)\Big|\leq \zeta_t(x,a)\Big\}^c\Bigg] \label{eq:union}\\
		& \geq 1-|X||A|mT\delta^\prime, \nonumber
	\end{align}
	where Inequality~\eqref{eq:union} holds by Union Bound. Noticing that $g_{t,i}(x,a)\leq 1$, substituting $\delta'$ with $\delta:=\delta'/|X||A|mT$ in $\zeta_t(x,a)$ with an additional Union Bound over the possible values of $N_t(x,a)$, we have, with probability at least $1-\delta$:
	\begin{equation*}
	\bigg \lvert\widehat{g}_{t,i}(x,a)-g^\circ_i(x,a)\bigg \rvert \le \sqrt{\frac{1}{2\max\{N_t(x,a),1\}}\ln\left(\frac{2mT|X||A|}{\delta} \right) }+ \frac{C_G}{\max\{N_t(x,a),1\}},
	\end{equation*}
which concludes the proof.
\end{proof}
We provide a similar result for the rewards function.
\begin{lemma}
	\label{lemma:auxCIr2}
	Given any $\delta \in (0,1)$, for any $(x,a)\in X \times A, t \in [T],$ it holds with probability at least $1-\delta$:
	\begin{equation*}
		\bigg \lvert\widehat{r}_{t}(x,a)-r^\circ(x,a)\bigg \rvert \le \sqrt{\frac{1}{2\max\{N_t(x,a),1\}}\ln\left(\frac{2T|X||A|}{\delta} \right) }+ \frac{C_r}{\max\{N_t(x,a),1\}}.
	\end{equation*}
\end{lemma}
\begin{proof}
	First let's define $\psi_t(x,a)$ as:
	\begin{equation*}
		\psi_t(x,a):= \sqrt{\frac{1}{2\max\{N_t(x,a),1\}}\ln\left(\frac{2}{\delta} \right) }+ \frac{C_r}{\max\{N_t(x,a),1\}}.
	\end{equation*}
	From Lemma~\ref{lemma:auxCIr1}, given $\delta'\in(0,1)$, we have fixed any  $t\in[T]$ and $(x,a)\in X \times A$:
	\begin{equation*}
		\mathbb{P}\Bigg[\bigg \lvert\widehat{r}_{t}(x,a)-r^\circ(x,a)\bigg \rvert \le \psi_t(x,a)\Bigg]\geq 1-\delta'.
	\end{equation*}
	Now, we are interested in the intersection of all the events, namely,
	\begin{equation*}
		\mathbb{P}\Bigg[\bigcap_{x,a,t}\Big\{\Big|\widehat{r}_{t}(x,a)-r^\circ(x,a)\Big|\leq \psi_t(x,a)\Big\}\Bigg].
	\end{equation*}
	Thus, we have:
	\begin{align}
		\mathbb{P}\Bigg[\bigcap_{x,a,t}\Big\{\Big|\widehat{r}_{t}(x,a)-r^\circ(x,a)&\Big|\leq \psi_t(x,a)\Big\}\Bigg]\nonumber\\&= 1 - \mathbb{P}\Bigg[\bigcup_{x,a,t}\Big\{\Big|\widehat{r}_{t}(x,a)-r^\circ(x,a)\Big|\leq \psi_t(x,a)\Big\}^c\Bigg]\nonumber\\
		&\geq 1 - \sum_{x,a,t}\mathbb{P}\Bigg[\Big\{\Big|\widehat{r}_{t}(x,a)-r^\circ(x,a)\Big|\leq \psi_t(x,a)\Big\}^c\Bigg] \label{eq:unionr}\\
		& \geq 1-|X||A|T\delta', \nonumber
	\end{align}
	where Inequality~\eqref{eq:unionr} holds by Union Bound. Noticing that $r_{t}(x,a)\leq 1$, substituting $\delta'$ with $\delta:=\delta'/|X||A|T$ in $\psi_t(x,a)$ with an additional Union Bound over the possible values of $N_t(x,a)$, we have, with probability at least $1-\delta$:
	\begin{equation*}
		\bigg \lvert\widehat{r}_{t}(x,a)-r^\circ(x,a)\bigg \rvert \le \sqrt{\frac{1}{2\max\{N_t(x,a),1\}}\ln\left(\frac{2T|X||A|}{\delta} \right) }+ \frac{C_r}{\max\{N_t(x,a),1\}},
	\end{equation*}
	which concludes the proof.
\end{proof}

In the following, we bound the distance between the empirical estimation of the constraints and the empirical mean of the mean values of the constraints distribution during the learning dynamic.

\begin{lemma}
	\label{lemma: conf int G hat}
	Given $\delta \in (0,1)$, for all episodes $t\in[T]$, state-action pairs $(x,a)\in X\times A$ and constraint $i\in[m]$,%and the estimator $\widehat{G}_t$ defined in Equation~\eqref{eq:G estimator}
	 it holds, with probability at least $1-\delta$:
	\begin{equation*}
		\bigg\lvert \widehat{g}_{t,i}(x,a) - \frac{1}{T} \underset{\tau \in [T]}{\sum}\mathbb{E}[g_{\tau,i}(x,a)]\bigg\rvert \le \xi_t(x,a),
	\end{equation*}  
	where,
	\begin{equation*}
		\xi_t(x,a):= \min\left\{1,\sqrt{\frac{1}{2 \max\{N_t(x,a),1\}}\ln\left( \frac{2mT|X||A|}{\delta}\right) }+\frac{C_G}{\max\{N_t(x,a),1\}}+\frac{C_G}{T}\right\}.
	\end{equation*}
\end{lemma}

\begin{proof}
	We first notice that if $\xi_t(x,a)=1$, the results is derived trivially by definition on the cost function.
	We prove now the non trivial case $\sqrt{\frac{1}{2 \max\{N_t(x,a),1\}}\ln\left( \frac{2mT|X||A|}{\delta}\right) }+\frac{C_G}{\max\{N_t(x,a),1\}}+\frac{C_G}{T} \le 1$.
	Employing Lemma~\ref{lemma:aux G Gmean and r} and Lemma~\ref{lemma:auxCIG2}, with probability $1-\delta$  for all $(x,a) \in X \times A$, for all $t \in [T]$ and for all $i \in [m]$, it holds that:
	\begin{align*}
		\Bigg|\widehat{g}_{t,i}(x,a)-\frac{1}{T}&\underset{\tau \in [T]}{\sum}\mathbb{E}[g_{\tau,i}(x,a)]\Bigg| \\ &\le \Bigg \lvert\widehat{g}_{t,i}(x,a)-g^\circ_i(x,a)\Bigg \rvert + \Bigg\lvert {g}^\circ_i(x,a)-\frac{1}{T}\underset{t \in [T] }{\sum}\mathbb{E}[g_{t,i}(x,a)] \Bigg\rvert \\
		&\le \sqrt{\frac{1}{2\max\{N_t(x,a),1\}}\ln\left(\frac{2mT|X||A|}{\delta} \right) }+ \frac{C_G}{\max\{N_t(x,a),1\}}+\frac{C_G}{T},
	\end{align*}
	where the first inequality follows from the triangle inequality. This concludes the proof.
\end{proof}
For the sake of simplicity, we analyze our algorithm with respect to the total corruption of the environment, defined as the maximum between the reward and the constraints corruption. In the following, we show that this choice does not prevent the confidence set events from holding.
\begin{corollary}
	\label{cor: confintG over}
	Given a corruption guess $\widehat{C}\ge C_G$ and $\delta \in (0,1)$,  for all episodes $t\in[T]$, state-action pairs $(x,a)\in X\times A$ and constraint $i\in[m]$, with probability at least $1-\delta$, it holds:
	\begin{equation*}
		\bigg\lvert \widehat{g}_{t,i}(x,a) - \frac{1}{T} \underset{\tau \in [T]}{\sum}\mathbb{E}[g_{\tau,i}(x,a)]\bigg\rvert \le \xi_t(x,a),
	\end{equation*}  
	where,
	\begin{equation*}
		\xi_t(x,a)= \min\left\{1,\sqrt{\frac{1}{2 \max\{N_t(x,a),1\}}\ln\left( \frac{2mT|X||A|}{\delta}\right) }+\frac{\widehat{C}}{\max\{N_t(x,a),1\}}+\frac{\widehat{C}}{T}\right\}.
	\end{equation*}
\end{corollary}
\begin{proof}
	Following the same analysis of Lemma~\ref{lemma: conf int G hat} for $\widehat{C} \ge C_G$, it holds
	\begin{align*}
		\bigg\lvert \widehat{g}_{t,i}(x,a) - &\frac{1}{T} \underset{\tau \in [T]}{\sum}\mathbb{E}[g_{\tau,i}(x,a)]\bigg\rvert \\& \le \sqrt{\frac{1}{2 \max\{N_t(x,a),1\}}\ln\left( \frac{2mT|X||A|}{\delta}\right) }+\frac{C_G}{\max\{N_t(x,a),1\}}+\frac{C_G}{T} \\
		& \le \sqrt{\frac{1}{2 \max\{N_t(x,a),1\}}\ln\left( \frac{2mT|X||A|}{\delta}\right) }+\frac{\widehat{C}}{\max\{N_t(x,a),1\}}+\frac{\widehat{C}}{T},
	\end{align*}
	which concludes the proof.
\end{proof}

\begin{corollary}
	\label{cor: CI for G}
	Taking the definition of $\xi_t$ employed in Lemma \ref{lemma: conf int G hat} 
	and defining $\mathcal{E}_G$ as the intersection event:
	\begin{align*}
		\mathcal{E}_G:=\bigg\{&\big  \lvert\widehat{g}_{t,i}(x,a)-g^\circ_i(x,a)\big \rvert \le \xi_t(x,a),~\forall (x,a) \in X\times A,\forall t \in [T],\forall i \in [m]\bigg\} \quad   \bigcap \\ & \left\{\bigg|\widehat{g}_{t,i}(x,a)-\frac{1}{T}\underset{\tau \in [T]}{\sum}\mathbb{E}[g_{\tau,i}(x,a)]\bigg| \le \xi_t(x,a),~\forall (x,a) \in X\times A,\forall t \in [T],\forall i \in [m]\right\},
	\end{align*}
	it holds that $\mathbb{P}[\mathcal{E}_G]\ge 1-\delta.$
\end{corollary}
Notice that by Corollary~\ref{cor: confintG over}, $\mathcal{E}_G$ includes all the analogous events where $\xi_t$ is built employing an arbitrary adversarial corruption $\widehat{C}$ such that $\widehat{C}\ge C_G$.

In the following, we provide similar results for the reward function.
\begin{lemma}
	\label{lemma: conf int r hat}
	Given $\delta \in (0,1)$, for all episodes $t\in[T]$ and for all state-action pairs $(x,a)\in X\times A$, with probability at least $1-\delta$, it holds:
	\begin{equation*}
		\left\lvert \widehat{r}_{t}(x,a) - \frac{1}{T} \underset{\tau \in [T]}{\sum}\mathbb{E}[r_{\tau}(x,a)]\right\rvert \le \phi_t(x,a),
	\end{equation*}  
	where,
	\begin{equation*}
		\phi_t(x,a):= \min\left\{1,\sqrt{\frac{1}{2 \max\{N_t(x,a),1\}}\ln\left( \frac{2T|X||A|}{\delta}\right) }+\frac{C_r}{\max\{N_t(x,a),1\}}+\frac{C_r}{T}\right\}.
	\end{equation*}
\end{lemma}

\begin{proof}
	
	Employing Lemma~\ref{lemma:aux G Gmean and r} and Lemma~\ref{lemma:auxCIr2}, with probability at least $1-\delta$, for all $(x,a) \in X \times A$ and for all $t \in [T]$, it holds:
	\begin{align*}
		\bigg|\widehat{r}_{t}(x,a)-\frac{1}{T}\underset{\tau \in [T]}{\sum}\mathbb{E}&[r_{\tau}(x,a)]\bigg| \\&\le \bigg \lvert\widehat{r}_{t}(x,a)-r^\circ(x,a)\bigg \rvert + \bigg\lvert {r}^\circ(x,a)-\frac{1}{T}\underset{t \in [T] }{\sum}\mathbb{E}[r_{t}(x,a)] \bigg\rvert \\
		&\le \sqrt{\frac{1}{2\max\{N_t(x,a),1\}}\ln\left(\frac{2T|X||A|}{\delta} \right) }+ \frac{C_r}{\max\{N_t(x,a),1\}}+\frac{C_r}{T},
	\end{align*}
	where the first inequality follows from the triangle inequality. Noticing that, by construction, $$\bigg|\widehat{r}_{t}(x,a)-\frac{1}{T}\underset{\tau \in [T]}{\sum}\mathbb{E}[r_{\tau}(x,a)]\bigg| \le 1,$$ for all episodes $t\in[T]$ and $(x,a)\in X\times A$ concludes the proof.
\end{proof}
We conclude the section, showing the overestimating the reward corruption does not invalidate the confidence set estimation.
\begin{corollary}
	\label{cor: conf int r over}
	Given a corruption guess $\widehat{C}\ge C_r$ and $\delta \in (0,1)$,  for all episodes $t\in[T]$ and for all state-action pairs $(x,a)\in X\times A$, with probability at least $1-\delta$, it holds:
	\begin{equation*}
		\bigg\lvert \widehat{r}_{t}(x,a) - \frac{1}{T} \underset{\tau \in [T]}{\sum}\mathbb{E}[r_{\tau}(x,a)]\bigg\rvert \le \phi_t(x,a),
	\end{equation*}  
	where,
	\begin{equation*}
		\phi_t(x,a):= \min\left\{1,\sqrt{\frac{1}{2 \max\{N_t(x,a),1\}}\ln\left( \frac{2T|X||A|}{\delta}\right) }+\frac{\widehat{C}}{\max\{N_t(x,a),1\}}+\frac{\widehat{C}}{T}\right\}.
	\end{equation*}
\end{corollary}
\begin{proof}
	The proof is analogous to the one of Corollary \ref{cor: confintG over}.
\end{proof}

\begin{corollary}
	\label{cor: CI for r}
	Taking the definition of $\phi_t$ employed in Lemma \ref{lemma: conf int r hat} 
	and defining $\mathcal{E}_r$ as the intersection event:
	%Defining $\mathcal{E}_r$ as the intersection event
	\begin{align*}
	\mathcal{E}_r:=\bigg\{\big \lvert\widehat{r}_{t}(x,a)-&r^\circ(x,a)\big \rvert \le \phi_t(x,a),~\forall(x,a)\in X\times A,\forall t \in [T]\bigg\} \quad   \bigcap \\ & \left\{\bigg|\widehat{r}_{t}(x,a)-\frac{1}{T}\underset{\tau \in [T]}{\sum}\mathbb{E}[r_{\tau}(x,a)]\bigg| \le \phi_t(x,a),~\forall(x,a)\in X\times A,\forall t \in [T]\right\},
	\end{align*}
	it holds that $\mathbb{P}[\mathcal{E}_r]\ge 1-\delta.$
\end{corollary}

Notice that by Corollary~\ref{cor: conf int r over}, $\mathcal{E}_r$ includes all the analogous events where $\phi_t$ is built employing an arbitrary adversarial corruption $\widehat{C}$ such that $\widehat{C}\ge C_r$.

\section{Omitted proofs when the corruption is known}
\label{app:known_c}
In the following, we provide the main results attained by Algorithm~\ref{alg:NS_UCSPS} in term of regret and constraints violations. The following results hold when the corruption of the environment is known to the learner.

We start providing a preliminary result, which shows that the linear program solved by Algorithm~\ref{alg:NS_UCSPS} at each $t\in[T]$ admits a feasible solution, with high probability.
\begin{lemma} 
	\label{lemma: aux regret}
	For any $\delta\in (0,1)$, for all episodes $t \in [T]$, with probability at least $1-5\delta$, the space defined by linear constraints $\left\{q \in \Delta(\mathcal{P}_t):\underline{G}_t^\top q \le \alpha\right\}$ admits a feasible solution and  it holds:
	\begin{equation*}
		\left\{q\in\Delta(P):\overline G^\top q \leq \alpha\right\} \subseteq \left\{q \in \Delta(\mathcal{P}_t):\underline{G}_t^\top q \le \alpha\right\} .
	\end{equation*}
\end{lemma}
\begin{proof}
	%First it is immediate to see that $\left\{q \in \Delta(\mathcal{P}_t):\underline{G}_t^\top q \le \alpha\right\}=\{q \in [0,1]^{|X|\times|A|}:\underline{G}_t^\top q \le \alpha\}\cap\{q \in \Delta(\mathcal{P}_t)\}$ is convex and linear since it is the intersection of convex linear spaces.
	Under the event $\mathcal{E}_{P}$, we have that  $\Delta(P)\subseteq\Delta(\mathcal{P}_t)$, for all episodes $t \in [T]$. Similarly, under the event $\mathcal{E}_G$, it holds that $\left\{q:\frac{1}{T}\underset{t \in [T]}{\sum}\mathbb{E}[G_t]^\top q \le \alpha\right\}\subseteq\left\{q: \underline{G}_t^\top q \le \alpha\right\}$. This implies that any feasible solution of the offline problem, is included in the optimistic safe set $\left\{q \in \Delta(\mathcal{P}_t):\underline{G}_t^\top q \le \alpha\right\}$.
	Taking the intersection event $\mathcal{E}_{P}\cap\mathcal{E}_G$ concludes the proof.
\end{proof}

We are now ready to provide the violation bound attained by Algorithm~\ref{alg:NS_UCSPS}.
\violationbound*
\begin{proof}
	In the following, we will refer as $\mathcal{E}_{\widehat q \ }$ to the event described in Lemma \ref{lem:transition_jin}, which holds with probability at least $1-6\delta$ .
	Thus, under $\mathcal{E}_G\cap \mathcal{E}_{\widehat q \ }$, the linear program solved by Algorithm~\ref{alg:NS_UCSPS} has a feasible solution (see Lemma~\ref{lemma: aux regret}) and it holds:
	\begin{subequations}
		\begin{align}
		V_T& = \underset{i \in [m]}{\max}\underset{t \in [T]}{\sum} \left[\mathbb{E}[G_t] ^\top q_t-\alpha\right]_i^+\nonumber \\ & = \underset{i \in [m]}{\max}\underset{t \in [T]}{\sum} \left[\left(\mathbb{E}[g_{t,i}]-g^\circ_i\right)^\top q_t +g^\circ_i{}^\top q_t -\alpha_i\right]^+ \nonumber
		\\ & \le
		\underset{i \in [m]}{\max}\underset{t \in [T]}{\sum}\left[ \left(\mathbb{E}[g_{t,i}]-g^\circ_i\right)^\top q_t +\left(\underline{g}_{t-1,i}+2\xi_{t-1}\right)^\top q_t -\alpha_i\right]^+ \label{eq:2 xi}\\
		& =
		\underset{i \in [m]}{\max}\underset{t \in [T]}{\sum}\left[ \left(\mathbb{E}[g_{t,i}]-g^\circ_i\right)^\top q_t +\underline{g}_{t-1,i}^\top \left(q_t-\widehat{q}_t\right)+\underline{g}_{t-1,i}^\top\widehat{q}_t +2\xi_{t-1}^\top q_t -\alpha_i\right]^+  \nonumber\\
		& \le \underset{i \in [m]}{\max}\underset{t \in [T]}{\sum}\left[ \left(\mathbb{E}[g_{t,i}]-g^\circ_i\right)^\top q_t +\underline{g}_{t-1,i}^\top \left(q_t-\widehat{q}_t\right) +2\xi_{t-1}^\top q_t  \right]^+ \label{eq: q space}\\
		& \le \underset{i \in [m]}{\max}\underset{t \in [T]}{\sum}\left| \left(\mathbb{E}[g_{t,i}]-g^\circ_i\right)^\top q_t\right| + 2\underset{i \in [m]}{\max}\underset{t \in [T]}{\sum}\left|\xi_{t-1}^\top q_t\right| + \underset{i \in [m]}{\max}\underset{t \in [T]}{\sum}\left|\underline{g}_{t-1,i}^\top \left(q_t-\widehat{q}_t\right)\right| \label{eq: viol +}\\
		& \le \underset{i \in [m]}{\max}\underset{t \in [T]}{\sum}\left\lVert \mathbb{E}[g_{t,i}]-g^\circ_i\right\rVert_1 + 2\underset{i \in [m]}{\max}\underset{t \in [T]}{\sum}\xi_{t-1}^\top q_t + \underset{i \in [m]}{\max}\underset{t \in [T]}{\sum} \left\lVert q_t-\widehat{q}_t\right\rVert_1 \label{eq: viol Holder} \\
		& \le C_G + 2\underset{i \in [m]}{\max}\underset{t \in [T]}{\sum}\xi_{t-1}^\top q_t + \underset{t \in [T]}{\sum} \lVert q_t-\widehat{q}_t\rVert_1 ,\label{eq: viol from def C^G}
	\end{align}
	\end{subequations}
	where Inequality~\eqref{eq:2 xi} follows from Corollary~\ref{cor: CI for G}, Inequality~\eqref{eq: q space} holds since Algorithm~\ref{alg:NS_UCSPS} ensures, for all $ t \in [T]$ and for all $i\in[m]$, that $\underline{g}_{t,i}^\top \widehat{q}_t \le \alpha_i$,  Inequality~\eqref{eq: viol +} holds since $[a+b]^+ \le |a|+|b|$, for all $a,b \in \mathbb{R}$, Inequality~\eqref{eq: viol Holder} follows from Hölder inequality since $||\underline{g}_{t,i}(x,a)||_\infty \le 1$ and $||q_t(x,a)||_\infty \le 1$, and finally Equation~\eqref{eq: viol from def C^G} holds  for the definition of $C_G$.
	
	To bound the last term of Equation~\eqref{eq: viol from def C^G}, we notice that, under $\mathcal{E}_{\widehat{q}}$, by Lemma~\ref{lem:transition_jin}, it holds:
	\begin{equation*}
		\underset{t \in [T]}{\sum} \lVert q_t-\widehat{q}_t\rVert_1 = \mathcal{O}\left(L|X| \sqrt{|A| T \ln \left(\frac{T|X||A|}{\delta}\right)}\right).
	\end{equation*}
	To bound the second term of Equation~\eqref{eq: viol from def C^G} we proceed as follows. Under $\mathcal{E}_{\widehat{q}}$ ,with probability at least $1-\delta$, it holds:
	\begin{subequations}
	\begin{align}
		&\sum_{t\in[T]}\xi_{t-1}^\top q_t  
		\leq  \sum_{t\in[T]}\sum_{x,a}\xi_{t-1}(x,a) \mathbb{I}_t(x,a) + L\sqrt{2T\ln\frac{1}{\delta}} \label{cum_vio:eq2_unknown}\\
		& \leq \underset{x,a}{\sum}\underset{t\in [T]}{\sum}\mathbb{I}_t(x,a)\Bigg(\sqrt{\frac{1}{2\max\{N_{t-1}(x,a),1\}}\ln\left(\frac{2mT|X||A|}{\delta}\right)} + \nonumber\\ & \mkern180mu +\frac{C_G}{\max\{N_{t-1}(x,a),1\}}+\frac{C_G}{T}\Bigg) +  L\sqrt{2T\ln\frac{1}{\delta}}\label{eq: from def xi}\\
		& \le \sqrt{\frac{1}{2}\ln\left(\frac{2mT|X||A|}{\delta}\right)} \underset{x,a}{\sum}\underset{t\in [T]}{\sum}\mathbb{I}_t(x,a)\sqrt{\frac{1}{\max\{N_{t-1}(x,a),1\}}} + \nonumber\\
		& \mkern180mu +C_G\underset{x,a}{\sum}{\sum_{t\in[T]}}\left( \frac{\mathbb{I}_t(x,a) }{\max\{N_{t-1}(x,a),1\}}+\frac{1}{T}\right) +  L\sqrt{2T\ln\frac{1}{\delta}} \nonumber\\
		& \le 3  \sqrt{\frac{1}{2}|X| |A|LT\ln\left(\frac{2mT|X||A|}{\delta}\right)}+ |X||A|(2+\ln(T))C_G + |X||A|C_G +  L\sqrt{2T\ln\frac{1}{\delta}}\label{eq:sqrtT and 1/t}\\
		& \le 3  \sqrt{\frac{1}{2}|X| |A|LT\ln\left(\frac{2mT|X||A|}{\delta}\right)}+ (3+\ln(T))|X||A|C_G +   L\sqrt{2T\ln\frac{1}{\delta}} \nonumber \\
		& = \mathcal{O}\left(\sqrt{|X||A|LT\ln\left(\frac{mT|X||A|}{\delta}\right)}+\ln(T)|X||A|C_G\right) \nonumber,
	\end{align}
\end{subequations}
	where Inequality~\eqref{cum_vio:eq2_unknown} follows from the Azuma-Hoeffding inequality  and noticing that $\sum_{x,a}\xi_{t-1}(x,a) q_t(x,a)\leq L$, Equality~\eqref{eq: from def xi} follows from the definition of $\xi_t$ and finally,
	Inequality~\eqref{eq:sqrtT and 1/t} holds since $1+\sum_{t=1}^{N_T(x,a)}\sqrt{\frac{1}{t}}\le 1+2\sqrt{N_T(x,a)} \le 3 \sqrt{N_T(x,a)}$ , since $1+\sum_{t=1}^{N_{T}(x,a)}\frac{1}{t}\le 2+\ln(T)$ and by  Cauchy-Schwarz inequality.
	Finally, we notice that the intersection event $\mathcal{E}_G\cap \mathcal{E}_{\widehat q \ }\cap\mathcal{E}_{\text{Azuma}}$ holds with the following probability,
	\begin{align*}
		\mathbb{P}\left[ \mathcal{E}_G\cap \mathcal{E}_{\widehat q \ }\cap\mathcal{E}_{\text{Azuma}}\right]  & = 1-\mathbb{P}\left[ \mathcal{E}_G^C\cup \mathcal{E}_{\widehat q \ }^C\cup \mathcal{E}_{\text{Azuma}}^C\right] \\ &\ge 1-\left(\mathbb{P}\left[ \mathcal{E}_G^C\right]+\mathbb{P}\left[  \mathcal{E}_{\widehat q \ }^C\right]+\mathbb{P}\left[  \mathcal{E}_{\text{Azuma}}^C\right]\right) \\& \ge 1-8\delta. 
	\end{align*}
	Noticing that, by Corollary \ref{cor: confintG over}, what holds for a $\xi_t$ built with corruption value $C_G$, still holds for a higher corruption (by definition, $C\ge C_G$) concludes the proof.
\end{proof}

We conclude the section providing the regret bound attained by Algorithm~\ref{alg:NS_UCSPS}.
\regretknownCG*
\begin{proof}
	First, we notice that under the event $\mathcal{E}_r$ it holds that, for all $ (x,a) \in X \times A$ and for all $t \in [T]$:
	\begin{equation*}
		\overline{r}_t(x,a)- 2 \phi_t(x,a) \le \frac{1}{T}\sum_{t \in [T]}\mathbb{E}[r_t(x,a)]. 
	\end{equation*}
	Let's observe that, by Lemma~\ref{lemma: aux regret}, under the event $\mathcal{E}_G \cap \mathcal{E}_P$, $\widehat{q}_t$ is optimal solution for $\overline{r}_{t-1}$ in $\left\{q \in \Delta(\mathcal{P}_t):\underline{G}_t^\top q \le \alpha\right\}$. Thus, under $\mathcal{E}_G \cap \mathcal{E}_P$ the optimal feasible solution $q^*$ is such that:
	\begin{equation*}
		\overline{r}_{t-1}^\top \widehat{q}_t \ge \overline{r}_{t-1}^\top q^*.
	\end{equation*}
	Thus under the event $\mathcal{E}_r$, it holds:
	\begin{align*}
		\frac{1}{T}\sum_{t \in [T]}\mathbb{E}[r_t]^\top q^* &\le \overline{r}_{t-1}^\top q^* \\ &\le \overline{r}_{t-1}^\top \widehat{q}_t \\ & \le \left(\frac{1}{T}\sum_{t \in [T]}\mathbb{E}[r_t] + 2 \phi_{t-1}\right)^\top \widehat{q}_t.
	\end{align*}
	Thus, we can rewrite the regret (under the event $\mathcal{E}_G \cap \mathcal{E}_r \cap \mathcal{E}_P$) as,
	\begin{align*}
		R_T & = \sum_{t\in[T]} \mathbb{E}[r_{t}]^{\top}   (q^*-    q_t)\\
		&=\sum_{t\in[T]}  \frac{1}{T} \sum_{\tau \in [T]}\mathbb{E}[r_{\tau}]^{\top}   (q^*-    q_t) + \sum_{t\in[T]}\left(\mathbb{E}[r_{t}]-\overline{r}\right)^\top (q^*-    q_t)  \\
		& = \sum_{t\in[T]}  \frac{1}{T} \sum_{\tau \in [T]}\mathbb{E}[r_{\tau}]^{\top}   (q^*- \widehat q_t + \widehat q_t-   q_t)  + \sum_{t\in[T]}\left(\mathbb{E}[r_{t}]-r^\circ +r^\circ -\overline{r}\right)^\top (q^*-    q_t) \\
		& \le \sum_{t\in[T]}  \left[\frac{1}{T} \sum_{\tau \in [T]}\mathbb{E}\left[r_{\tau}\right]^{\top}   (q^*- \widehat{q}_t)\right] + \sum_{t\in[T]} \lVert \widehat{q}_t - q_t \rVert_1 + \sum_{t \in [T]}\lVert\mathbb{E}[r_{t}]-r^\circ\rVert_1 + \sum_{t \in [T]}\lVert r^\circ -\overline{r}\rVert_1\\
		& \le \sum_{t\in[T]} 2\phi_{t-1}^\top q_t + \sum_{t\in[T]} \lVert \widehat{q}_t - q_t \rVert_1 + 2 C_r.
	\end{align*}
	
	By Lemma \ref{lem:prob_int_jin} with probability at least $1-6\delta$ under event $\mathcal{E}_{\widehat{q}}$ we can bound $\sum_{t \in [T]} \lVert \widehat{q}_t - q_t \rVert_1$ as:
	\begin{equation*}
		\sum_{t\in[T]} \lVert \widehat{q}_t - q_t \rVert_1 = \mathcal{O}\left(L|X|\sqrt{|A|T \ln \left(\frac{T |X||A|}{\delta}\right)}\right).
	\end{equation*}
	
	Finally with probability at least $1-\delta$ it holds:
	\begin{subequations}
	\begin{align}
		&\sum_{t\in[T]}\phi_{t-1}^\top q_t 
		 \leq  \sum_{t\in[T]}\sum_{x,a}\phi_{t-1}(x,a) \mathbb{I}_t(x,a) + L\sqrt{2T\ln\frac{1}{\delta}} \label{reg:eq2_unknown}\\
		& \leq \underset{x,a}{\sum}\underset{t\in [T]}{\sum}\mathbb{I}_t(x,a)\Bigg(\sqrt{\frac{1}{2\max\{N_{t-1}(x,a),1\}}\ln\left(\frac{2T|X||A|}{\delta}\right)} + \nonumber\\ & \mkern170mu +\frac{C_r}{\max\{N_{t-1}(x,a),1\}}+\frac{C_r}{T}\Bigg) +  L\sqrt{2T\ln\frac{1}{\delta}}\label{eq: from def phi}\\
		& \le \sqrt{\frac{1}{2}\ln\left(\frac{2T|X||A|}{\delta}\right)} \underset{x,a}{\sum}\underset{t\in [T]}{\sum}\mathbb{I}_t(x,a)\sqrt{\frac{1}{\max\{N_{t-1}(x,a),1\}}} + \nonumber\\
		& \mkern170mu +C_r\underset{x,a}{\sum}\underset{t\in[T]}{{\sum}}\left( \frac{\mathbb{I}_t(x,a) }{\max\{N_{t-1}(x,a),1\}}+\frac{1}{T}\right) +  L\sqrt{2T\ln\frac{1}{\delta}} \nonumber\\
		& \le 3  \sqrt{\frac{1}{2}|X| |A|LT\ln\left(\frac{2T|X||A|}{\delta}\right)}+ |X||A|(2+\ln(T))C_r + |X||A|C_r +  L\sqrt{2T\ln\frac{1}{\delta}}\label{eq:sqrtT and 1/t 2}\\
		& \le 3  \sqrt{\frac{1}{2}|X| |A|LT\ln\left(\frac{2T|X||A|}{\delta}\right)}+ (3+\ln(T))|X||A|C_r +   L\sqrt{2T\ln\frac{1}{\delta}} \nonumber \\
		& = \mathcal{O}\left(\sqrt{|X||A|LT\ln\left(\frac{T|X||A|}{\delta}\right)}+\ln(T)|X||A|C_r\right) \nonumber,
	\end{align}
	\end{subequations}
	where Inequality~\eqref{reg:eq2_unknown} follows from Azuma-Hoeffding inequality, Equality~\eqref{eq: from def phi} holds for the definition of $\phi_t$, and Inequality~\eqref{eq:sqrtT and 1/t 2} holds since $1+\sum_{t=1}^{N_T(x,a)}\sqrt{\frac{1}{t}}\le 1+2\sqrt{N_T(x,a)} \le 3 \sqrt{N_T(x,a)}$, $1+\sum_{t=1}^{N_{T}(x,a)}\frac{1}{t}\le 2+\ln(T)$ and by Cauchy-Schwarz inequality.
	Thus, we observe that with probability at least $1-9\delta$ it holds:
	\begin{equation*}
		R_T= \mathcal{O}\left(L|X| \sqrt{|A|T \ln \left(\frac{T|X||A|}{\delta}\right)}+ \ln(T)|X||A|C_r\right).
	\end{equation*}
	Employing Corollary~\ref{cor: conf int r over} and the definition of $C$, which is at least equal to $C_r$, concludes the proof.
\end{proof}

\section{Omitted proofs when the knowledge of $C$ is not precise}
\label{app:C_notprecise}

In this section, we focus on the performances of Algorithm~\ref{alg:NS_UCSPS} when a guess on the corruption value is given as input. These preliminary results are "the first step" to relax the assumption on the knowledge about the corruption.

First, we present some preliminary results on the confidence set.

\begin{lemma}
	\label{lemma: under CG}
	%Under an environment with true adversarial corruption value $C^G$ and under the false assumption of a 
	Given the corruption guess $\widehat{C}_G$, where $C_G=\widehat{C}_G + \epsilon,$ with $ \epsilon>0$, and confidence $\xi_t$ as defined in Algorithm~\ref{alg:NS_UCSPS} using $\widehat{C}_G$ as corruption value, for any $\delta\in(0,1)$, with probability at least $1-\delta$, for all episodes $t\in[T]$, state-action pair $(x,a)\in X\times A$ and constraint $i\in[m]$, the following result holds:
	\begin{equation*}
		g^\circ_i(x,a) \le \widehat{g}_{t,i}(x,a)+\xi_t(x,a)+\left(\frac{\epsilon}{\max\{N_t(x,a),1\}}+ \frac{\epsilon}{T}\right).
	\end{equation*}
	Similarly, recalling the definition of $\underline{G}_t$, for all episodes $t\in[T]$, state-action pairs $(x,a)\in X\times A$ and constraints $i\in[m]$, it holds:
	\begin{equation*}
		g^\circ_i(x,a) \le \underline{g}_{t,i}(x,a)+ 2 \xi_t(x,a)+\left(\frac{\epsilon}{\max\{N_t(x,a),1\}}+ \frac{\epsilon}{T}\right).
	\end{equation*}
\end{lemma}
\begin{proof}
	To prove the result, we recall that, by Corollary~\ref{cor: CI for G}, with probability at least $1-\delta$, the following holds, for all episodes $t\in[T]$, state-action pairs $(x,a)\in X\times A$ and constraints $i\in[m]$:
	\begin{align*}
		\bigg\lvert \widehat{g}_{t,i}(x,a) - &g^\circ_i(x,a)]\bigg\rvert \le \\ &\sqrt{\frac{1}{2 \max\{N_t(x,a),1\}}\ln\left( \frac{2mT|X||A|}{\delta}\right) }+\frac{C_G}{\max\{N_t(x,a),1\}}+\frac{C_G}{T},
	\end{align*}
	which can be rewritten as:
	\begin{equation*}
		\bigg\lvert \widehat{g}_{t,i}(x,a) - g^\circ_i(x,a)]\bigg\rvert \le \xi_t(x,a) + \frac{\epsilon}{\max\{N_t(x,a),1\}}+\frac{\epsilon}{T},
	\end{equation*}  
	where,
	\begin{equation*}
		\xi_t(x,a):= \min\left\{1,\sqrt{\frac{1}{2 \max\{N_t(x,a),1\}}\ln\left( \frac{2mT|X||A|}{\delta}\right) }+\frac{\widehat{C}_G}{\max\{N_t(x,a),1\}}+\frac{\widehat{C}_G}{T}\right\},
	\end{equation*}
	and $C_G=\widehat{C}_G+\epsilon$, which concludes the proof.
\end{proof}
We are now ready study the regret bound attained by the algorithm when the guess on the corruption is an overestimate.

\begin{theorem}
	\label{theo: over Cr}
	%Given an environment denoted by the true value of adversarial corruption on the costs $C^G$, and true value of adversarial corruption on the rewards $C^r$, every instance of 
	For any $\delta\in(0,1)$, Algorithm~\ref{alg:NS_UCSPS}, when instantiated with corruption value $\widehat{C}$ which is an overestimate of the true value of $C$, i.e. $\widehat{C} > C_G$ and $\widehat{C} > C_r$, attains with probability at least $1-8\delta$:
	\begin{equation*}
		R_T= \mathcal{O}\left(L|X|\sqrt{|A|T\ln\left(\frac{T|X||A|}{\delta}\right)} + \ln(T)|X||A|\widehat{C}\right).
	\end{equation*}
\end{theorem}
\begin{proof}
	 By Corollary \ref{cor: confintG over}, it holds that the decision space of the linear program performed by Algorithm~\ref{alg:NS_UCSPS} contains with high probability the optimal solution that respects to the constraints. Furthermore, employing Corollary \ref{cor: conf int r over} and following the proof of Theorem~\ref{theo:regretknownCG} concludes the proof.
\end{proof}

We proceed bounding the violation attained by our algorithm when an underestimate of the corruption is given as input.
\begin{theorem}
	%Given an environment denoted by the true value of adversarial corruption on the costs $C^G$, every instance of 
	For any $\delta\in(0,1)$, Algorithm~\ref{alg:NS_UCSPS}, when instantiated with corruption value $\widehat{C}$ which is an underestimate of the true value of $C_G$, i.e. $ \widehat{C} < C_G$, attains with probability at least $1-9\delta$:
	\begin{equation*}
		V_T = \mathcal{O}\left( L|X|\sqrt{|A|T\ln\left(\frac{mT|X||A|}{\delta}\right)}+ \ln(T)|X||A|C_G\right).
	\end{equation*}
\end{theorem}
\begin{proof}
	First, let's define $\epsilon\in \mathbb{R}^+$ such that $\epsilon:= C_G-\widehat{C}$. Then, with probability at least $1-\delta$:
	\begin{subequations}
	\begin{align}
		V_T& = \underset{i \in [m]}{\max}\underset{t \in [T]}{\sum} \left[\mathbb{E}[G_t] ^\top q_t-\alpha\right]_i^+ \\ &= \underset{i \in [m]}{\max}\underset{t \in [T]}{\sum} \left[\left(\mathbb{E}[g_{t,i}]-g^\circ_i	\right)^\top q_t +g^\circ_i{}^\top q_t -\alpha_i\right]^+ \nonumber
		\nonumber\\ 
		& \le
		\underset{i \in [m]}{\max}\underset{t \in [T]}{\sum}\Bigg[ (\mathbb{E}[g_{t,i}]-g^\circ_i)^\top q_t +\underline{g}_{t-1,i}^\top (q_t-\widehat{q}_t)+\underline{g}_{t-1 ,i}^\top\widehat{q}_t +2\xi_{t-1}^\top q_t + \nonumber \\ &  \mkern200mu +\sum_{x,a}\left(\frac{\epsilon}{\max\{N_{t-1}(x,a),1\}}+\frac{\epsilon}{T}\right) q_t(x,a) -\alpha_i\Bigg]^+ \label{eq: from under CG}\\
		& \le C_G + 2\underset{i \in [m]}{\max}\underset{t \in [T]}{\sum}\xi_{t-1}^\top q_t + \underset{t \in [T]}{\sum} \lVert q_t-\widehat{q}_t\rVert_1 +\nonumber\\&\mkern200mu+ \underset{t \in [T]}{\sum}\sum_{x,a}\frac{\epsilon}{\max\{N_{t-1}(x,a),1\}}q_t(x,a) + \epsilon L,\label{eq: from under CG2}	
	\end{align}
	\end{subequations}
	where Inequality~\eqref{eq: from under CG} follows from Lemma~\ref{lemma: under CG} and Inequality~\eqref{eq: from under CG2} is derived as in the proof of Theorem~\ref{theo: violationbound}, and considering that $\lVert q_t \rVert_1=L, \ \forall t \in [T]$.
	Now, employing the Azuma-Hoeffding inequality, we can bound, with probability at least $1-\delta$  the term $\sum_{t=1}^T\sum_{x,a}\frac{\epsilon}{\max\{N_{t-1}(x,a),1\}}q_t(x,a)$  as follows:
	\begin{align*}
		\underset{t \in [T]}{\sum}\sum_{x,a}\frac{\epsilon}{\max\{N_{t-1}(x,a),1\}}q_t(x,a) & \le L\sqrt{2T \ln \frac{1}{\delta}} + \underset{t \in [T]}{\sum}\sum_{x,a}\frac{\epsilon}{\max\{N_{t-1}(x,a),1\}}\mathbb{I}_t(x,a) \\
		& \le L\sqrt{2T \ln \frac{1}{\delta}}+\epsilon |X||A| (1+\ln (T)),
	\end{align*}
	where we applied Azume Hoeffding inequality and the fact that $\sum_{t \in [N_T(x,a)]}\frac{1}{t}\le 1+\ln(T)$.
	Finally, following the steps of the proof of Theorem \ref{theo: violationbound} to bound the first 3 elements of Inequality \eqref{eq: from under CG2} under $\mathcal{E}_{\widehat{q}}$ with probability at least $1-\delta$, and considering that $\epsilon\le C_G$ and $\widehat{C} \le C_G$, it holds, with probability at least $1-9\delta$, 
	\begin{equation*}
		V_T=\mathcal{O}\left(L|X|\sqrt{|A|T\ln\left(\frac{T|X||A|}{\delta}\right)}+ \ln(T)|X||A|C_G\right),
	\end{equation*}
	which concludes the proof.
\end{proof}

Finally, we provide the violation bound attained by Algorithm~\ref{alg:NS_UCSPS} when an overestimate of the corruption value is given as input.
\begin{theorem}
	\label{theo: over CG}
	%Given an environment denoted by the true value of adversarial corruption on the costs $C^G$, every instance of 
	For any $\delta\in(0,1)$, Algorithm \ref{alg:NS_UCSPS}, when instantiated with corruption value $\widehat{C}$ which is an overestimate of the true value of $C_G$, i.e. $\widehat{C} > C_G$, attains with probability at least $1-8\delta$:
	\begin{equation*}
		V_T= \mathcal{O}\left(L|X|\sqrt{|A|T\ln\left(\frac{T|X||A|}{\delta}\right)} + \ln(T)|X||A|\widehat{C}\right).
	\end{equation*}
\end{theorem}
\begin{proof}
	 The proof follows by employing Corollary \ref{cor: confintG over} to the proof of Theorem \ref{theo: violationbound}.
\end{proof}

\section{Omitted proofs when the corruption is \emph{not} known}
\label{app:unknown_C}
%\stradi{nota generale, secondo me spesso nelle proof usi $\sum_{x_k,a_k}1=L$, questo è vero se considere le coppie stato azione, quindi $(x_k,a_k)$. Quindi credo vada scritto così. Dimmi poi se ho frainteso}\anna{sì, la sommatoria è sulle coppie}\stradi{ok controlla bene di ver messo la parentesi ovunque serva}
In the following section we provide the omitted proofs of the theoretical guarantees attained by Algorithm~\ref{alg:unknown_c}. The algorithm is designed to work when the corruption value is \emph{not} known.
\subsection{Lagrangian formulation of the constrained optimization problem}
%Algorithm \ref{alg:unknown $C^G$} performs a FTRL optimization on a Lagrangian function of  the original problem. In this section we will prove that the optimum of the loss function we choose, is the optimum of the original problem thanks to the strong duality property. 
%The loss function we choose to minimize through FTRL optimization is
%\begin{equation*}
	%\ell_{t,j}:= \frac{\mathbb{I}(i_t=j)}{w_{t,j}+\gamma}\left(\sum_{x_k^t,a_k^t}(1-r_t(x_k^t,a_k^t))+\frac{1}{\rho}\sum_{i \in [m]}\left[\widehat{g}_{t,i}^\top \widehat{q}_{t}^j-\alpha_i\right]^+\right), \quad t \in [T], j\in [M],
%\end{equation*}
%and we recall here the definition of $\rho$ as a problem dependent constant, assuming Slater condition:
%\begin{equation*}
	%\rho := \sup_{q \in \Delta(M)} \min_{t \in [T]} \min_{i \in [m]} \left[\alpha_i-G_{t,i}^\top q\right].
%\end{equation*}
%Note that $\ell_{t,j}$ is an estimator of  $(L-\mathbb{E}[r_t]^\top q)+\frac{1}{\rho}\sum_{i \in [m]}\left[{G}_{t,i}^\top q_{t}^j-\alpha_i\right]^+$ . Thus, we can state the following result: 
Since Algorithm~\ref{alg:unknown_c} is based on a Lagrangian formulation of the constrained problem, it is necessary to show that this approach is well characterized. Precisely, we show that a \emph{strong duality-like} result holds even when the Lagrangian function is defined taking the positive violations.
%\begin{theorem}
	%The optimization problem 
	%\begin{equation*}
		%\text{OPT}^*:= \begin{cases}\max_{q \in \Delta(M)} \mathbb{E}[r_t]^\top q\\
			%s.t.~ G_t^\top q \le \alpha, \end{cases}
	%\end{equation*}
	%that respect Slater's condition
	%presents strong duality property with respect to the dual Lagrangian problem 
	%\begin{equation*}
		%\argmax_{\lambda \ge 0} \min_{q \in \Delta(M)} \left((L-\mathbb{E}[r_t]^\top q)+\lambda \sum_{i \in [m]}[G_{t,i}^\top q_t^j-\alpha_i]^+\right),
	%\end{equation*}
	%$L-\text{OPT}^*= \min_{q \in \Delta(M)} \left((L-\mathbb{E}[r_t]^\top q)+1/\rho \sum_{i \in [m]}[G_{t,i}^\top q_t^j-\alpha_i]^+\right)$ and $q$ optimal solution for $\text{OPT}^*$ is also optimal solution for $\min_{q \in \Delta(M)} \left((L-\mathbb{E}[r_t]^\top q)+1/\rho \sum_{i \in [m]}[G_{t,i}^\top q_t^j-\alpha_i]^+\right)$ .
%\end{theorem}

First, we show that strong duality holds with respect to the standard Lagrangian function, even considering a subset of the Lagrangian multiplier space.

\strongDuality*
\begin{proof}
	The proof follows the one of Theorem 3.3 in \citep{Unifying_Framework}.
	First we prove that, given the Lagrangian function $\mathcal{Q}(\lambda,q):= r^\top q - \sum_{i \in [m]}\lambda_i\left({G}_{i}^\top q-\alpha_i\right)$, it holds:
	\begin{equation*}
		\min_{\lVert \lambda \rVert_1\in [0,L/\rho]} \max_{  q \in \Delta(P)} \mathcal{Q}(\lambda,q)= \min_{ \lambda \in \mathbb{R}^m_{\geq0}} \max_{  q \in \Delta(P)} \mathcal{Q}(\lambda,q),
	\end{equation*}
	with $\lambda\in \mathbb{R}^m_{\geq0}$.
	In fact notice that for all $\lambda\in \mathbb{R}^m_{\geq0}$ such that $\lVert\lambda \rVert_1 > L/\rho$ :
	\begin{equation*}
		\max_{  q \in \Delta(P)}\mathcal{Q}(\lambda,q) \ge \mathcal{Q}(\lambda,{q}^\circ) \ge -\sum_{i \in [m]}\lambda_i \left({G}_{i}^\top q^\circ-\alpha_i\right) \ge \lVert\lambda\rVert_1\rho > L,
	\end{equation*}
	where $q^\circ$ is defined as $q^\circ:=\argmax_{q \in \Delta(P)} \min_{i \in [m]} \left[\alpha_i-G_{i}^\top q\right]$.
	Moreover since 
	\begin{equation*}
		\min_{\lVert \lambda \rVert_1\in [0,L/\rho]} \max_{  q \in \Delta(P)} \mathcal{Q}(\lambda,q) \le \max_{  q \in \Delta(P)} \mathcal{Q}(\underline{0},q) = \max_{  q \in \Delta(P)} r^\top q \le L,
	\end{equation*}
	it holds: 
	\begin{align*}
		\min_{ \lambda \in \mathbb{R}^m_{\geq0}} \max_{  q \in \Delta(P)} \mathcal{Q}(\lambda,q) & = \min\left\{\min_{\lVert \lambda \rVert_1\in [0,L/\rho]} \max_{  q \in \Delta(P)} \mathcal{Q}(\lambda,q), \min_{\lVert \lambda \rVert_1\ge L/\rho} \max_{  q \in \Delta(P)} \mathcal{Q}(\lambda,q)\right\} \\ &= \min_{\lVert \lambda \rVert_1\in [0,L/\rho]} \max_{  q \in \Delta(P)} \mathcal{Q}(\lambda,q).
	\end{align*}
	Thus,
	\begin{align*}
		\text{OPT}_{r,G,\alpha} & =
		\max_{  q \in \Delta(P)} \min_{ \lambda \in \mathbb{R}^m_{\geq0}}  \mathcal{Q}(\lambda,q) \\
		& \le \max_{  q \in \Delta(P)} \min_{\lVert \lambda \rVert_1\ge L/\rho}  \mathcal{Q}(\lambda,q) \\
		& \le \min_{\lVert \lambda \rVert_1\ge L/\rho} \max_{  q \in \Delta(P)}  \mathcal{Q}(\lambda,q)\\
		& = \min_{ \lambda \in \mathbb{R}_{\geq0}^m} \max_{  q \in \Delta(P)} \mathcal{Q}(\lambda,q)\\
		& = \text{OPT}_{r,G,\alpha},
	\end{align*}
	where the second inequality holds by the \emph{max-min} inequality and the last step holds by the well-known strong duality result in CMDPs~\citep{Altman1999ConstrainedMD}. This concludes the proof.
\end{proof}

In the following, we extend the previous result for the Lagrangian function which encompasses the positive violations.
\positiveStrongDuality*
\begin{proof}
	%First, it is easy to check that minimizing $(L-r^\top q)+\frac{1}{\rho}\sum_{i \in [m]}\left[{G}_{i}^\top q-\alpha_i\right]^+$ is equivalent to maximizing $\mathcal{L}(1/\rho,q)= r^\top q - \frac{1}{\rho}\sum_{i \in [m]}\left[{G}_{i}^\top q-\alpha_i\right]^+$. 
	Following the definition of Lagrangian function, we have:
	%\begin{equation*}
		%\text{OPT}^*:= \begin{cases}\max_{q \in \Delta(P)} \mathbb{E}[r_t]^\top q\\
			%s.t. G^\top q \le \alpha \end{cases}
 %\end{equation*}
	\begin{align*}
		\max_{q \in \Delta(P)}\mathcal{L}(L/\rho,q) & %\le \max_{q \in \Delta(P)}\min_{\lambda \in [0,1/\rho]} \mathcal{L}(\lambda,q) \\
		 = \max_{  q \in \Delta(P)} r^\top q - \frac{L}{\rho}\sum_{i \in [m]}\left[{G}_{i}^\top q-\alpha_i\right]^+\\
		%\min_{\lambda \in [0,1/\rho]} \max_{  q \in \Delta(P)} \mathbb{E}[r_t]^\top q - \lambda\sum_{i \in [m]}\left[{G}_{t,i}^\top q_{t}^j-\alpha_i\right]^+ \\
		%& \le \min_{\lVert \lambda \rVert_1\in [0,1/\rho]} \max_{  q \in \Delta(M)} \mathbb{E}[r_t]^\top q - \sum_{i \in [m]}\lambda_i[{G}_{t,i}^\top q_{t}^j-\alpha_i]^+\\ %\anna{andrebbe trovato lo spazio di scrivere $\lambda\in \mathbb{R}^m_{\ge0}$ ma non so bene come impaginarlo} \\
		& \le  \max_{  q \in \Delta(P)} \min_{\lVert \lambda \rVert_1\in [0,L/\rho]} r^\top q - \sum_{i \in [m]}\lambda_i[{G}_{i}^\top q-\alpha_i]^+\\ 
		 & \le \min_{\lVert \lambda \rVert_1\in [0,L/\rho]} \max_{  q \in \Delta(P)} r^\top q - \sum_{i \in [m]}\lambda_i[{G}_{i}^\top q-\alpha_i]^+\\ 
		%& \le \min_{\lVert \lambda \rVert_1\in [0,L/\rho]} \max_{  q \in \Delta(M)} \mathbb{E}[r_t]^\top q - \sum_{i \in [m]}\lambda_i({G}_{t,i}^\top q_{t}^j-\alpha_i) \\
		& \le \min_{\lVert \lambda \rVert_1\in [0,L/\rho]} \max_{  q \in \Delta(P)} r^\top q - \sum_{i \in [m]}\lambda_i\left({G}_{i}^\top q-\alpha_i\right) \\
		& = \text{OPT}_{r,G,\alpha}
	\end{align*}
	where $\lambda\in\mathbb{R}^m_{\geq0}$ is the Lagrangian vector, the second inequality holds by the \emph{max-min inequality} and the last step follows from Lemma~\ref{lemma:auxiliary dual}.
	Noticing that for all $q$ belonging to $\left\{q \in \Delta(P):G^\top q \le \alpha\right\}$, we have $\mathcal{L}(1/\rho,q)=r^\top q$, which implies that $\max_{q \in \Delta(P)}\mathcal{L}(1/\rho,q) \ge \text{OPT}_{r,G,\alpha}$, concludes the proof.
\end{proof}

\subsection{Preliminary results}
%\begin{assumption}
	%The Slater's parameter $\rho$ is such that $\rho\geq 2L\eta=2\sqrt{\frac{L\ln(T)}{T}}$
%\end{assumption}
%\stradi{questo va controllato bene, perchè $\rho\in[0,L]$, probabilmente ha senso prendere $\eta$ che sia lineare in $1/2L$}
In the following sections we will refer as:
\begin{equation}
	\label{def: VThat}
	\widehat{V}_T:= \sum_{t \in [T]} \sum_{j \in [M]} \frac{w_{t,j}\mathbb{I}(j_t=j)}{w_{t,j}+\gamma} \sum_{i \in [m]}\left[\widehat{g}_{t,i}^{\ j\top} \widehat{q}_t^{\ j}-\alpha_i\right]^+,
\end{equation}
to the estimated violation attained by the instances of Algorithm~\ref{alg:unknown_c}. Furthermore, we will refer as:
\begin{equation}
	\label{def: VThat i*}
	\widehat{V}_{T,j^*}:=\sum_{t \in [T]}  \frac{\mathbb{I}(j_t=j^*)}{w_{t,j^*}+\gamma}\sum_{i \in [m]}\left[\widehat{g}_{t,i}^{\ j^*\top} \widehat{q}_t^{\ j^*}-\alpha_i\right]^+,
\end{equation}
to the estimated violation attained by the optimal instance  $j^*$, namely, the integer in $[M]$ such that the true corruption $C \in [2^{j^*-1},2^{j^*}]$.

Furthermore, we will refer as $q_t^j$ to the occupancy measure induced by the policy proposed by $\texttt{Alg}^{j}$  at episode $t$, with $j \in [M],t \in [T]$, and we will refer as:
\begin{equation*}
	\widehat{g}_{t,i}^j(x,a):= \frac{\sum_{\tau \in [t]}\mathbb{I}_\tau(x,a)\mathbb{I}(j_\tau=j)g_{\tau,i}(x,a)}{\max\{N_t^j(x,a),1\}},
\end{equation*}
to the estimate of the cost computed for $j$-th algorithm, where  $N_t^j(x,a)$ is a counter initialize to $0$ in $t=0$, and which increases by one from episode $t$ to episode $t+1$ whenever $\mathbb{I}_t(x,a)\mathbb{I}(j_t=j)=1$.

\subsubsection{Stability parameters}
\label{app: stability param}
In the following sections, we will employ the stability parameters  $\beta$ defined as follows:
%\anna{mancano $\beta_2$ e $\beta_5$ perchè ho mantenuto la numerazione simile al teorema di \cite{jin2024no}, bisogna decidere se teere così o cambiare} \stradi{in che senso?, se li usiamo vanno messi}\anna{definizione di corruption robust è \begin{equation*}
%		\mathbb{E}\left[\sum_{t \in [t']}\overline{r}^\top (q^*-q_t)\right]\le \mathbb{E}\left[\sqrt{\beta_1t'}+\left(\beta_2+\beta_3\theta\right)\mathbb{I}(t'\ge 1)\right] + \mathbb{P}\left(\sum_{t\in[t']}C_t>\theta\right)LT,
%	\end{equation*}
%	and
%	\begin{equation*}
%		\mathbb{E}\left[\sum_{t \in [t']}\sum_{i\in[m]}\left[g_{t,i}^\top q_t -\alpha_i\right]^+\right]\le \mathbb{E}\left[\sqrt{\beta_4t'}+\left(\beta_5+\beta_6\theta\right)\mathbb{I}(t'\ge 1)\right] + \mathbb{P}\left(\sum_{t\in[t']}C_t>\theta\right)LT.
%	\end{equation*} ma per come abbiamo calcolato noi $R_T$ e $V_T$, $\beta_2$ e $\beta_5$ sono 0}\stradi{no scriviamo $\beta_2=0$ ecc}
\begin{itemize}
	\item $\beta_1=\mathcal{O}\left(L^2|X|^2|A|\ln\left(\frac{T|X||A|}{\delta}\right)\right)$
	\item $\beta_2=\mathcal{O}\left(|X|^2|A|^2\log(T)\log\left(\nicefrac{\log(T)}{\delta}\right)\right)$
	\item $\beta_3=\mathcal{O}\left(\ln(T)^2|X||A|\right)$
	\item $\beta_4=\mathcal{O}\left(L^2|X|^2|A|\ln\left(\frac{mT|X||A|}{\delta}\right)\right)$
	\item $\beta_5=\mathcal{O}\left(|X|^2|A|^2\log(T)\log\left(\nicefrac{\log(T)}{\delta}\right)\right)$
	\item $\beta_6=\mathcal{O}\left(\ln(T)^2|X||A|\right)$
	
\end{itemize}

\subsubsection{Omitted proofs and lemmas}

We start providing some preliminary results on the optimistic estimator employed by Algorithm~\ref{alg:unknown_c}.

\begin{lemma}
	\label{lemma: rhat w}
	For any $\delta \in (0,1)$, given $\gamma \in \mathbb{R}_{\ge0}$, with probability at least $1-\delta$, it holds:
	\begin{equation*}
		\widehat{R}_T \leq \mathcal{O}\left(\gamma T L M+ L\sqrt{2T\ln\left(\frac{1}{\delta}\right)}\right),
	\end{equation*}
	where $\widehat{R}_T=\sum_{t \in [T]}\sum_{j \in [M]}\left(w_{t,j}\left(L-\mathbb{E}[r_t]^\top q_t^j\right) - \frac{w_{t,j}\mathbb{I}(j_t=j)}{w_{t,j}+\gamma}\sum_{(x_k^t,a_k^t)}\left(1-r_t\left(x_k^t,a_k^t\right)\right)\right)$.
	
\end{lemma}
\begin{proof}
	We first observe that by construction: 
	\begin{equation*}
		\mathbb{E}\left[\sum_{t \in [T]}\sum_{j \in [M]}\frac{w_{t,j}\mathbb{I}(j_t=j)}{w_{t,j}+\gamma}\sum_{(x_k^t,a_k^t)}\left(1-r_t\left(x_k^t,a_k^t\right)\right)\right]=\sum_{t \in [T]}\sum_{j \in [M]}\frac{w_{t,j}^2}{w_{t,j}+\gamma}\left(L-\mathbb{E}[r_t]^\top q_t^j\right).
	\end{equation*}
	Moreover, still by construction, for all episodes $t\in[T]$, it holds:
	\begin{equation*}
		\sum_{j \in [M]}\frac{w_{t,j}\mathbb{I}(j_t=j)}{w_{t,j}+\gamma}\sum_{(x_k^t,a_k^t)}\left(1-r_t\left(x_k^t,a_k^t\right)\right) \le \sum_{j \in [M]}\mathbb{I}(j_t=j)\sum_{(x_k^t,a_k^t)}\left(1-r_t\left(x_k^t,a_k^t\right)\right) \le L.
	\end{equation*}
	Thus, employing Azuma-Hoeffding inequality, with probability at least $1-\delta$, it holds:
	\begin{align*}
		&\sum_{t \in [T]}\sum_{j \in [M]}\left(\frac{w_{t,j}^2}{w_{t,j}+\gamma}(L-\mathbb{E}[r_t]^\top q_t^j)-\frac{w_{t,j}\mathbb{I}(j_t=j)}{w_{t,j}+\gamma}\sum_{(x_k^t,a_k^t)}(1-r_t(x_k^t,a_k^t))\right) \le L\sqrt{2T\ln\left(\frac{1}{\delta}\right)}.
	\end{align*}
	Finally we notice that:
	\begin{align*}
		\sum_{t \in [T]}\sum_{j \in [M]}w_{t,j}&\left(L-\mathbb{E}[r_t]^\top q_t^{j}\right)-\sum_{t \in [T]}\sum_{j \in [M]}\frac{w_{t,j}^2}{w_{t,j}+\gamma}\left(L-\mathbb{E}[r_t]^\top q_t^j\right) \\ &=
		 \sum_{t \in [T]} \sum_{j \in [M]}\left(\frac{w_{t,j}}{w_{t,j}+\gamma}\right)\gamma\left(L-\mathbb{E}[r_t]^\top q_t^j\right) \\ &\le \gamma T L M.
	\end{align*}
	Adding and subtracting $\mathbb{E}\left[\sum_{t \in [T]}\sum_{j \in [M]}\frac{w_{t,j}\mathbb{I}(j_t=j)}{w_{t,j}+\gamma}\sum_{(x_k^t,a_k^t)}\left(1-r_t\left(x_k^t,a_k^t\right)\right)\right]$ to the quantity of interest and employing the previous bound concludes the proof.
\end{proof}
We provide an additional result on the optimistic estimator employed by Algorithm~\ref{alg:unknown_c}.
\begin{lemma} 
	\label{lemma: rhat i*}
	For any $\delta \in (0,1)$, given $\gamma \in \mathbb{R}_{\ge0}$, %and recalling that the notation $q_t^{j^*}$ indicates the occupancy measure defined by the strategy proposed at episode $t \in [T]$ by the instance of algorithm \ref{alg:UOB-REPS with known $C^G$} of index $j^*$ initiated with a supposed corruption value of $2^{j^*}$ ran by Algorithm \ref{alg:unknown $C^G$},
	with probability at least $1-\delta$, it holds:
	\begin{equation*}
		\sum_{t \in [T]}\frac{\mathbb{I}(j_t=j^*)}{w_{t,j^*}+\gamma}\sum_{(x_k^t,a_k^t)}\left(1-r_t\left(x_k^t,a_k^t\right)\right)- \sum_{t \in [T]}\left(L-\mathbb{E}[r_t]^\top q_t^{j^*}\right) = \mathcal{O}\left(\frac{L}{\gamma} \ln \left(\frac{1}{\delta}\right) \right)
	\end{equation*}
\end{lemma}
	\begin{proof}
		The  proof closely follows the idea of Corollary~\ref{cor: Neu2015}.
		We define the loss $\bar\ell_{t}=\sum_{(x_k^t,a_k^t)}(1-r_t(x_k^t,a_k^t))$,  the optimistic loss estimator $\widehat{\ell}_{t}:=\frac{\mathbb{I}(j_t=j^*)}{w_{t,j^*}+\gamma}\sum_{(x_k^t,a_k^t)}(1-r_t(x_k^t,a_k^t))$  and the unbiased estimator $\widetilde{\ell}_{t}:=\frac{\mathbb{I}(j_t=j^*)}{w_{t,j^*}}\sum_{(x_k^t,a_k^t)}(1-r_t(x_k^t,a_k^t))$.
		
		Employing the same argument as~\cite{neu} it holds:
		\begin{equation*}
		\widehat{\ell}_{t}= \frac{\mathbb{I}(j_t=j^*)}{w_{t,j^*}+\gamma}\bar\ell_{t}\le \frac{\mathbb{I}(j_t=j^*)}{w_{t,j^*}+\gamma\bar\ell_{t}/L}\bar\ell_{t} \le \frac{L}{2\gamma }\frac{2\gamma  \bar\ell_{t}/w_{t,j^*}L}{1+\gamma \bar\ell_{t}/w_{t,j^*}L} \mathbb{I}(j_t=j^*) \le \frac{L}{2\gamma}\ln\left(1 + \frac{2\gamma}{L} \widetilde{\ell}_{t}\right),
		\end{equation*}
		 since  $\frac{z}{1+z/2}\le \ln(1+z), z\in \mathbb{R}^+$. Employing the previous inequality, it holds:
		\begin{align*}
			\mathbb{E}\left[\exp\left(\frac{2\gamma}{L} \widehat{\ell}_t\right)\Bigg|\mathcal{F}_{t-1}\right]& \le \mathbb{E}\left[\exp\left(\frac{2\gamma}{L} \frac{L}{2\gamma}\ln\left(1+\frac{2\gamma}{L}\widetilde{\ell}_{t}\right)\right)\Bigg|\mathcal{F}_{t-1}\right] \\
			& = \mathbb{E}\left[1 + \frac{2\gamma}{L} \widetilde{\ell}_t\Bigg|\mathcal{F}_{t-1}\right]\\
			& = 1 + \frac{2\gamma}{L}\mathbb{E}\left[\frac{\mathbb{I}(j_t=j^*)}{w_{t,j^*}}\sum_{(x_k^t,a_k^t)}(1-r_t(x_k^t,a_k^t))\Bigg|\mathcal{F}_{t-1}\right]\\
			&\le 1 + \frac{2\gamma}{L}\left(L-\mathbb{E}[r_t]^\top q_t^{j^*}\right) \\
			& \le \exp\left(\frac{2\gamma}{L} \left(L-\mathbb{E}[r_t]^\top q_t^{j^*}\right)\right),
		\end{align*}
		where $\mathcal{F}_{t-1}$ is the filtration up to episode $t$.
		We conclude the proof employing the Markov inequality as follows:
		\begin{align*}
			\mathbb{P}\Bigg(\sum_{t \in [T]}\frac{2\gamma}{L}&\left(\widehat{\ell}_t-\left(L-\mathbb{E}[r_t]^\top q_t^{j^*}\right)\right) \ge \epsilon\Bigg) \\& \le \mathbb{E}\left[\exp\left(\sum_{t \in [T]}\frac{2\gamma}{L}\left(\widehat{\ell}_t-\left(L-\mathbb{E}[r_t]^\top q_t^{j^*}\right)\right)\right)\right] \exp(-\epsilon)\\ & \le \exp(-\epsilon).
		\end{align*}
		Solving  $\delta=\exp(-\epsilon)$ for $\epsilon$ we obtain:
		\begin{equation*}
			\mathbb{P}\left(\sum_{t \in [T]}\left(\widehat{\ell}_t-\left(L-\mathbb{E}[r_t]^\top q_t^{j^*}\right)\right) \ge \frac{L}{2\gamma}\ln \left(\frac{1}{\delta}\right)\right) \le \delta.
		\end{equation*}
	This concludes the proof.
	\end{proof}
	
We are now ready to prove the regret bound attained by FTRL with respect to the Lagrangian underlying problem.	

\begin{lemma}
	\label{lemma:FTRL}
	For any $\delta \in (0,1)$ and properly setting the learning rate $\eta$ such that $\eta \le \frac{1}{2\Lambda m\left(\sqrt{\beta_1T}+\beta_2+\beta_5+\sqrt{\beta_4T}\right)}$, Algorithm \ref{alg:unknown_c} attains, with probability at least $1-2\delta$:
	\begin{align*}
		&\sum_{t \in [T]}\mathbb{E}[r_t]^\top q_t^{j^*}-\sum_{t \in [T]}\sum_{j \in [M]}w_{t,j}\mathbb{E}[r_t]^\top q_t^{j} + \frac{Lm+1}{\rho}\widehat{V}_T - \frac{Lm+1}{\rho}\widehat{V}_{T,j^*}\\
		& \hspace{1cm} +\left(\frac{m(mL+1)}{\rho}\beta_5+\beta_2\right)\nu_{T,j^*} + \left(\sqrt{\beta_1}+\left(\frac{m(Lm+1)}{\rho}\right)\sqrt{\beta_4}\right)\sqrt{T}\nu_{T,j*}
		\\&  \le \mathcal{O}\bigg(\frac{M\ln T}{\eta}+ \eta~ {m^4 L^4  T M }+ \eta~M\ln(T)m^4L^2\beta_5^2 + \eta~M\ln(T)\beta_2^2  \\ &\mkern100mu+\eta T(\beta_1+L^2m^4\beta_4)M\log(T)+\gamma T L M + L\sqrt{T\ln \left(\nicefrac{1}{\delta}\right)} + \frac{L}{\gamma} \ln \left(\nicefrac{1}{\delta}\right)\bigg).
	\end{align*}
\end{lemma}
\begin{proof}
	 First, we define $\ell_{t,j}$, for all $t \in [T]$, for all $j \in [M]$ as: 
	\begin{equation*}
		\ell_{t,j}:= \frac{\mathbb{I}(j_t=j)}{w_{t,j}+\gamma}\left(\sum_{(x_k^t,a_k^t)}(1-r_t(x_k^t,a_k^t))+\frac{Lm+1}{\rho}\sum_{i \in [m]}\left[\widehat{g}_{t,i}^{\ j\top} \widehat{q}_{t}^{\ j}-\alpha_i\right]^+\right),
	\end{equation*}
	and $b_{t,j}$ for all $t \in [T]$, for all $j \in [M]$ as: 
	\begin{equation*}
		b_{t,j}:= \left(\left(\frac{m(mL+1)}{\rho}\beta_5+\beta_2\right) + \left(\sqrt{\beta_1} + \frac{m(Lm+1)}{\rho}\sqrt{\beta_4}\right)\sqrt{T}\right)(\nu_{t,j}-\nu_{t-1,j}),
	\end{equation*}
	with $\nu_{t,j} = \max_{\tau \in [t]}\frac{1}{w_{\tau,j}}$.\\
	%Notice that $\ell_{t,j}$ is an optimistic estimator of $L-r_t^\top q_t + \frac{L}{\rho}\sum_{i \in [m]}[G_{t,i}^\top q_t -\alpha_i]^+$, and the $q_t$ that minimizes this quantity is the same $q_t$ that maximizes $r_t^\top q_t - \frac{L}{\rho}\sum_{i \in [m]}[G_{t,i}^\top q_t -\alpha_i]^+$.
	 First we prove that $\eta w_{t,j}|\ell_{t,j}-b_{t,j}|\le \nicefrac{1}{2}$ for all $t\in [T],j \in [M]$, to apply Lemma~\ref{lem: FTRL}. It holds that $\eta w_{t,j}\lvert \ell_{t,j} \rvert \le  \frac{\eta (L\rho+L^2m^2+Lm)}{\rho} \le \frac{1}{2}$ for all $j \in [M]$, for all $t \in [T]$ as long as $\eta \le \frac{\rho}{2(L\rho+L^2m^2+Lm)}\le \frac{\rho}{2(L^2m^2+Lm)}$, which is true if $\eta \le \frac{\rho}{2Lm(Lm+1)}$. It also holds that 
	 \begin{align*}
	 	\eta w_{t,j}|b_{t,j}|& = \eta w_{t,j}\left(\left(\frac{m(Lm+1)}{\rho}\beta_5 + \beta_2\right)+\left(\frac{m(Lm+1)}{\rho}\sqrt{\beta_4}+ \sqrt{\beta_1}\right)\sqrt{T} \right) (\nu_{t,j}-\nu_{t-1,j})\\
	 	& \le \eta\left(\left(\frac{m(Lm+1)}{\rho}\beta_5 + \beta_2\right)+\left(\frac{m(Lm+1)}{\rho}\sqrt{\beta_4}+ \sqrt{\beta_1}\right)\sqrt{T} \right)\left(1-\frac{\nu_{t-1,j}}{\nu_{t,j}}\right)\\
	 	& \le \eta  \left(\left(\frac{m(Lm+1)}{\rho}\beta_5 + \beta_2\right)+\left(\frac{m(Lm+1)}{\rho}\sqrt{\beta_4}+ \sqrt{\beta_1}\right)\sqrt{T} \right)\\
	 	& \le \frac{1}{2},
	 \end{align*}
	 if $\eta \le \frac{1}{2\Lambda m\left(\sqrt{\beta_1T}+\beta_2+\beta_5+\sqrt{\beta_4T}\right)}$, where we used the fact that $\nu_{t,j} \neq \nu_{t-1,j} \iff \nicefrac{1}{w_{t,j}}=\nu_{t,j}$. 
	 Thus, if the previous conditions on $\eta$ hold, and notice that the second condition implies the first,  Algorithm~\ref{alg:unknown_c} attains, by Lemma~\ref{lem: FTRL} :
	\begin{align}
		 &\sum_{t \in [T]}\Bigg[\sum_{j \in [M]}\frac{w_{t,j}\mathbb{I}(j_t=j)}{w_{t,j}+\gamma}\sum_{(x_k^t,a_k^t)}(1-r_t(x_k^t,a_k^t))- \frac{\mathbb{I}(j_t=j^*)}{w_{t,j^*}+\gamma}\sum_{(x_k^t,a_k^t)}(1-r_t(x_k^t,a_k^t))\Bigg] + \frac{Lm+1}{\rho}\widehat{V}_{T} \nonumber\\
		&\le \frac{M\ln T}{ \eta}+ 2\eta  \frac{ T M(L\rho+L^2m^2+Lm)^2}{\rho^2} \nonumber\\
		& \ \ \ \ + 2\eta \left(2 \left(\frac{m(mL+1)}{\rho}\beta_5+\beta_2\right)^2 M \ln\left(T\right)+  2T\left(\sqrt{\beta_1}+\left(\frac{m(Lm+1)}{\rho}\right)\sqrt{\beta_4}\right)^2M\ln(T)\right)\nonumber\\
		&\mkern250mu+ \frac{Lm+1}{\rho}\widehat{V}_{T,j^*} + \sum_{t\in[T]}\sum_{j\in[M]}w_{t,j}b_{t,j}-\sum_{t \in [T]}b_{t,j^*} ,\label{eq:FTRL}
	\end{align} %\anna{Ho modificato qualcosa}
	where we used the following inequalities:
	\begin{itemize}
	\item First inequality:
	 \begin{equation*}\sum_{t\in[T]}\sum_{j\in[M]}w_{t,j}^2(\ell_{t,j}-b_{t,j})^2\le 2\sum_{t\in[T]}\sum_{j\in[M]}w_{t,j}^2\ell_{t,j}^2 + 2 \sum_{t\in[T]}\sum_{j\in[M]}w_{t,j}^2b_{t,j}^2,
	  \end{equation*}  \item Second inequality: \begin{equation*}\left(\sum_{(x_k^t,a_k^t)}(1-r_t(x_k^t,a_k^t))+\frac{Lm+1}{\rho}\sum_{i \in [m]}\left[\widehat{g}_{t,i}^{\ j\top} \widehat{q}_{t}^{\ j}-\alpha_i\right]^+\right) \le \frac{(L\rho+L^2m^2+Lm)}{\rho},\end{equation*} 
	  \item Third inequality: \begin{equation*}\sum_{t \in [T]} \sum_{j\in [M]} w_{t,j}^2 \ell_{t,j}^2  \le \frac{ T M(L\rho + L^2m^2+Lm)^2}{\rho^2},\end{equation*} \end{itemize}
	 and that, it holds:
	\begin{subequations}
	\begin{align}
		&\sum_{t\in [T]}\sum_{j \in [M]}w_{t,j}^2 b_{t,j}^2 \nonumber\\&= \sum_{t\in [T]}\sum_{j \in [M]}(w_{t,j} b_{t,j})^2 \nonumber\\
		& \le \left(\left(\frac{m(Lm+1)}{\rho}\beta_5 + \beta_2\right)+\left(\frac{m(Lm+1)}{\rho}\sqrt{\beta_4}+ \sqrt{\beta_1}\right)\sqrt{T} \right)^2 \sum_{j \in [M]}\sum_{t \in [T]}\left(\frac{1}{\nu_{t,j}}(\nu_{t,j}-\nu_{t-1,j})\right)^2 \label{eq: b 1}\\
		& \le \left(2 \left(\frac{m(mL+1)}{\rho}\beta_5+\beta_2\right)^2 + 2T\left(\frac{m(Lm+1)}{\rho}\sqrt{\beta_4}+ \sqrt{\beta_1}\right)^2\right) \sum_{j \in [M]}\sum_{t \in [T]}\left(1-\frac{\nu_{t-1,j}}{\nu_{t,j}}\right)^2 \nonumber\\
		& \le\left(2 \left(\frac{m(mL+1)}{\rho}\beta_5+\beta_2\right)^2 + 2T\left(\frac{m(Lm+1)}{\rho}\sqrt{\beta_4}+ \sqrt{\beta_1}\right)^2\right)  \sum_{j \in [M]}\sum_{t \in [T]}\left(1-\frac{\nu_{t-1,j}}{\nu_{t,j}}\right)\nonumber\\
		& \le \left(2 \left(\frac{m(mL+1)}{\rho}\beta_5+\beta_2\right)^2 + 2T\left(\frac{m(Lm+1)}{\rho}\sqrt{\beta_4}+ \sqrt{\beta_1}\right)^2\right) \sum_{j \in [M]}\sum_{t \in [T]}\ln\left(\frac{\nu_{t,j}}{\nu_{t-1,j}}\right)\label{eq: b 2}\\
		& \le \left(2 \left(\frac{m(mL+1)}{\rho}\beta_5+\beta_2\right)^2 + 2T\left(\frac{m(Lm+1)}{\rho}\sqrt{\beta_4}+ \sqrt{\beta_1}\right)^2\right) \sum_{j \in [M]} \ln \left(\prod_{t\in [T]}\frac{\nu_{t,j}}{\nu_{t-1,j}}\right)\nonumber\\
		& \le \left(2 \left(\frac{m(mL+1)}{\rho}\beta_5+\beta_2\right)^2 + 2T\left(\frac{m(Lm+1)}{\rho}\sqrt{\beta_4}+ \sqrt{\beta_1}\right)^2\right) \sum_{j \in [M]} \ln\left(\frac{\nu_{T,j}}{\nu_{0,j}}\right)\nonumber\\
		& \le\left(2 \left(\frac{m(mL+1)}{\rho}\beta_5+\beta_2\right)^2 + 2T\left(\frac{m(Lm+1)}{\rho}\sqrt{\beta_4}+ \sqrt{\beta_1}\right)^2\right) M \ln\left(T\right) ,\label{eq: b 3}
	\end{align} 
	\end{subequations}
	where Inequality \eqref{eq: b 1} is true since $\nu_{t,j}-\nu_{t-1,j} \neq 0$ only when $w_{t,j}=\nicefrac{1}{\nu_{t,j}}$ by definition,  Inequality \eqref{eq: b 2} holds since $1-a\le -\ln a$, and Inequality \eqref{eq: b 3} holds since by definition $\nu_{T,j}\le T$ and $\nu_{0,j}=M$.
	Notice also that, following a similar reasoning, it holds:
	\begin{align*}
		&\sum_{t\in[T]}w_{t,j}b_{t,j}-\sum_{t \in [T]}b_{t,j^*} \\& = \left(\left(\frac{m(Lm+1)}{\rho}\beta_5 + \beta_2\right)+\left(\frac{m(Lm+1)}{\rho}\sqrt{\beta_4}+ \sqrt{\beta_1}\right)\sqrt{T} \right)\sum_{t\in [T]}\sum_{j \in [M]}\left(1-\frac{\nu_{t-1,i}}{\nu_{t,i}}\right) \\
		&\quad- \left(\left(\frac{m(Lm+1)}{\rho}\beta_5 + \beta_2\right)+\left(\frac{m(Lm+1)}{\rho}\sqrt{\beta_4}+ \sqrt{\beta_1}\right)\sqrt{T} \right)\sum_{t\in [T]}(\nu_{t,j^*}-\nu_{t-1,j^*})\\
		& \le \mathcal{O}\left(m^2L\beta_5 M \log(T)+ \beta_2 M \log(T)+ (\sqrt{\beta_1}+Lm^2\sqrt{\beta_4})\sqrt{T}M\log(T)\right) \\
		& \quad- \left(\left(\frac{m(Lm+1)}{\rho}\beta_5 + \beta_2\right)+\left(\frac{m(Lm+1)}{\rho}\sqrt{\beta_4}+ \sqrt{\beta_1}\right)\sqrt{T} \right)\nu_{T,j^*}\\
		%& \le \mathcal{O}\left(m^2L\beta_5 M \log(T) + \beta_2 M \log(T)+ (\sqrt{\beta_1}+Lm^2\sqrt{\beta_4})M^{\nicefrac{3}{2}}\sqrt{T}\right)\\
		%& \quad - \left(\frac{m(mL+1)}{\rho}\beta_5+\beta_2\right)\nu_{T,j^*} - \left(\sqrt{\beta_1}+\left(\frac{m(Lm+1)}{\rho}\right)\sqrt{\beta_4}\right)\sqrt{\nu_{T,j*}T}.
	\end{align*} 
	Thus, with probability at least $1-2\delta$, it holds:
	\begin{align}
		&\sum_{t \in [T]}\mathbb{E}[r_t]^\top q_t^{j^*}-\sum_{t \in [T]}\sum_{j \in [M]}w_{t,j}\mathbb{E}[r_t]^\top q_t^{j} + \frac{Lm+1}{\rho}\widehat{V}_T \nonumber
		\\ &= \sum_{t \in [T]}\sum_{j \in [M]}w_{t,j}\left(L-\mathbb{E}[r_t]^\top q_t^j\right) - \sum_{t \in [T]}\left(L-\mathbb{E}[r_t]^\top q_t^{j^*}\right)+ \frac{Lm+1}{\rho}\widehat{V}_T \label{eq:sumw1}  \\
		& \le \mathcal{O}\bigg(\frac{M\ln T}{\eta}+ \eta~ {m^4 L^4  T M }+ \eta~M\ln(T)m^4L^2\beta_5^2 + \eta~M\ln(T)\beta_2^2\nonumber\\
		& \mkern70mu+\eta T(\beta_1+L^2m^4\beta_4)M\log(T)+\gamma T L M + L\sqrt{T\ln \left(\nicefrac{1}{\delta}\right)} + \frac{L}{\gamma} \ln \left(\nicefrac{1}{\delta}\right)\bigg)+ \frac{Lm+1}{\rho}\widehat{V}_{T,j^*} \nonumber \\
		& \mkern70mu- \left(\frac{m(mL+1)}{\rho}\beta_5+\beta_2\right)\nu_{T,j^*} - \left(\sqrt{\beta_1}+\left(\frac{m(Lm+1)}{\rho}\right)\sqrt{\beta_4}\right)\sqrt{T}\nu_{T,j*}, \label{eq:ftrl2} 
	\end{align}
	where Equation~\eqref{eq:sumw1} holds since $\sum_{j \in [M]}w_{t,j}=1,~ \forall t \in [T]$, and Inequality~\eqref{eq:ftrl2} holds, with probability at least $1-2\delta$, by Lemma~\ref{lemma: rhat w}, Lemma~\ref{lemma: rhat i*} and Equation~\eqref{eq:FTRL}. This concludes the proof.
\end{proof}

In order to provide the desired bound $R_T$ and $V_T$ for Algorithm~\ref{alg:unknown_c}, it is necessary to study the relation between the aforementioned performance measures and the terms appearing from the FTRL analysis in Lemma~\ref{lemma:FTRL}.

Thus, we bound the distance between the incurred violation and the estimated one.
\begin{lemma}
	\label{lemma: widehat V}
	%If we consider the positive cumulative violation attained by Algorithm \ref{alg:unknown $C^G$} $$V_T^+ = \sum_{j \in [M]}\max_{i \in [m]}\sum_{t \in [T]}\mathbb{I}(j_t=j)\left[G_{t,i}^\top q_t^j - \alpha_i\right]^+$$, and $\widehat{V}_T^+$ defined as in equation \eqref{def: VThat}, then
	 For any $\gamma \in \mathbb{R}_{\ge0}$, given  $\delta \in (0,1)$, with probability at least $1-10\delta$, it holds:
	\begin{equation*}
		 V_T-\widehat{V}_T  = \mathcal{O}\left(mL|X|\sqrt{|A|T \ln \left(\frac{mT|X||A|}{\delta}\right)}+ m \ln(T)|X||A|C + \gamma TLM\right).
	\end{equation*}
\end{lemma}
\begin{proof}
	We start defining the quantity $\widehat\xi_{t,j}(x,a)$ -- for all episode $t \in [T]$, for all state-action pairs $(x,a) \in X \times A$, for all instance $j \in [M]$ -- as in Theorem \ref{theo: violationbound} but using the true value of adversarial corruption $C$, considering that the counter $N_t^j(x,a)$ increases on one unit from episode $t$ to $t+1$, if and only if $\mathbb{I}(j_t=j)\mathbb{I}_t(x,a)=1$, and by applying a Union Bound over all instances $j \in [M]$ namely,
	\begin{equation}
		\label{def: xi True}
		\widehat\xi_{t,j}(x,a):= \min \left\{1, \sqrt{\frac{1}{2 \max\{N_t^j(x,a),1\}}\ln\left( \frac{2mMT|X||A|}{\delta}\right) }+\frac{C}{\max\{N_t^j(x,a),1\}}+\frac{C}{T}\right\},
	\end{equation}
	By Corollary~\ref{cor: CI for G}, and applying a Union Bound on instances $j\in [M]$ simultaneously $ \forall t \in [T], \forall i \in [m], \forall (x,a)\in X\times A,\forall j \in [M]$, with probability at least $1-\delta$, it holds:
	\begin{equation}
		\label{eq: ghat xitrue}
		\widehat{g}_{t,i}^j(x,a)+\widehat\xi_{t,j}(x,a) \ge g_i^\circ(x,a). %\anna{ho cambiato da $\frac{1}{T}\sum_{t\in [T]}G_{t,i}$ a $g$}
	\end{equation}
	Resorting to the definition of $\widehat{V}_T$, we obtain that, with probability at least $1-\delta$, under $\mathcal{E}_{\widehat{q}}$:
	\begin{subequations}
	\begin{align}
		&\widehat{V}_T  =\sum_{t \in [T]}\sum_{j \in [M]} \frac{w_{t,j}\mathbb{I}(j_t=j)}{w_{t,j}+\gamma} \sum_{i \in [m]} \left[\widehat{g}^j_{t,i}{}^\top \widehat{q}_t^j-\alpha_i\right]^+ \nonumber\\
		& =
		\sum_{t \in [T]}\sum_{j \in [M]} \frac{w_{t,j}\mathbb{I}(j_t=j)}{w_{t,j}+\gamma} \sum_{i \in [m]} \left[(\widehat{g}_{t,i}^j{}^\top q_t^j +\widehat\xi_{t,j}^{\ \top} q_t^j-\alpha_i)  -\widehat\xi_{t,j}^{\ \top} q_t^j - \widehat{g}_{t,i}^j{}^\top (q_t^j-\widehat{q}_t^j)\right]^+ \nonumber\\
		& \ge \sum_{t \in [T]}\sum_{j \in [M]} \frac{w_{t,j}\mathbb{I}(j_t=j)}{w_{t,j}+\gamma} \sum_{i \in [m]}\left(\left[(\widehat{g}_{t,i}^j+\widehat\xi_{t,j})^\top q_t^j - \alpha_i\right]^+  - \widehat\xi_{t,j}^{\ \top} q_t^j -\widehat{g}_{t,i}^j{}^\top |q_t^j-\widehat{q}_{t}^j|\right) \label{eq: vhat eq1}\\
		&\ge \sum_{t \in [T]}\sum_{j \in [M]} \frac{w_{t,j}\mathbb{I}(j_t=j)}{w_{t,j}+\gamma} \sum_{i \in [m]}\left( \left[g_i^{\circ\top} q_t^j-\alpha_i\right]^+ - \widehat\xi_{t,j}^{\ \top} q_t^j -\| q_t^j-\widehat{q}_t^j\|_1\right) \label{eq: vhat eq2}\\
		& \ge \sum_{t \in [T]}\sum_{j \in [M]} \frac{w_{t,j}\mathbb{I}(j_t=j)}{w_{t,j}+\gamma} \sum_{i \in [m]} \left( \left[\mathbb{E}[g_{t,i}]^\top q_t^j-\alpha_i\right]^+ -\widehat\xi_{t,j}^{\ \top} q_t^j \right) - \sum_{t \in [T]}\sum_{j \in [M]} \frac{w_{t,j}\mathbb{I}(j_t=j)}{w_{t,j}+\gamma} \nonumber\\ & \mkern80mu \cdot  \sum_{i \in [m]}\left[(g^\circ_i - \mathbb{E}[g_{t,i}])^\top q_t^j\right]^+ - \mathcal{O}\left(mL|X|\sqrt{|A|T \ln \left(\frac{T|X||A|}{\delta}\right)}\right), \label{eq: vhat eq3}
	\end{align}
\end{subequations}
	where Inequality~\eqref{eq: vhat eq1} holds since $[a-b]^+ \ge [a]^+ - b,~ a \in \mathbb{R},b \in \mathbb{R}_{\ge0}$, Inequality~\eqref{eq: vhat eq2} follows from Inequality~\eqref{eq: ghat xitrue} and since, by definition, $\widehat{g}_{t,i}^j(x,a)\le 1,\forall(x,a)\in X\times A,\forall i \in [m],\forall t \in [T],\forall j \in [M]$ and, finally, Inequality~\eqref{eq: vhat eq3} holds under event $\mathcal{E}_{\widehat{q}}$ by Lemma~\ref{lem:transition_jin} after noticing that $\sum_{t \in [T]}\sum_{j \in [M]}\frac{w_{t,j}\mathbb{I}(j_t=j)}{w_{t,j}+\gamma}\sum_{i \in [m]}\lVert q_t^j-\widehat{q}_t^j\rVert_1 \le \sum_{t \in [T]}\sum_{j \in [M]}\mathbb{I}(j_t=j)\left(\frac{w_{t,j}}{w_{t,j}+\gamma}\right)\sum_{i \in [m]}\lVert q_t^j-\widehat{q}_t^j\rVert_1 \le m\sum_{t \in [T]}\lVert q_t^{j_t}-\widehat{q}_t^{\ j_t}\rVert_1$. 
	
	We will bound the previous terms separately.
	
	\textbf{Lower-bound to $\sum_{t \in [T]}\sum_{j \in [M]} \frac{w_{t,j}\mathbb{I}(j_t=j)}{w_{t,j}+\gamma} \sum_{i \in [m]}  \left[\mathbb{E}[g_{t,i}]^\top q_t^j-\alpha_i\right]^+$.}
	
	We bound the term by the Azuma-Hoeffding inequality. Indeed, with probability at least $1-\delta$, it holds:
	\begin{align*}
		\sum_{t \in [T]}&\sum_{j \in [M]} \frac{w_{t,j}\mathbb{I}(j_t=j)}{w_{t,j}+\gamma} \sum_{i \in [m]}  \left[\mathbb{E}[g_{t,i}]^\top q_t^j-\alpha_i\right]^+ \\ &  \ge \left(\sum_{t \in [T]}\sum_{j \in [M]} \frac{w_{t,j}^2}{w_{t,j}+\gamma} \sum_{i \in [m]}  \left[\mathbb{E}[g_{t,i}]^\top q_t^j-\alpha_i\right]^+\right) - mL \sqrt{2T \ln \left(\frac{1}{\delta}\right)},
	\end{align*} 
	where we used the following upper-bound to the martingale sequence:
	\begin{align*}
		\sum_{j \in [M]} \frac{w_{t,j}\mathbb{I}(j_t=j)}{w_{t,j}+\gamma} \sum_{i \in [m]}  \left[\mathbb{E}[g_{t,i}]^\top q_t^j-\alpha_i\right]^+ &\le \sum_{j \in [M]} \mathbb{I}(j_t=j)\left(\frac{w_{t,j}}{w_{t,j}+\gamma}\right) \sum_{i \in [m]}  \left[\mathbb{E}[g_{t,i}]^\top q_t^j\right]^+ \\
		&\le \sum_{j \in [M]} \mathbb{I}(j_t=j) \sum_{i \in [m]}  \lVert q_t^j\rVert_1 \\&\le  m  \lVert q_t^{j_t}\rVert_1 \\&\le mL.
	\end{align*}
	Moreover, we observe the following bounds:
	\begin{align*}
		\sum_{t \in [T]}\sum_{j \in [M]} w_{t,j}\sum_{i \in [m]}  \left[\mathbb{E}[g_{t,i}]^\top q_t^j-\alpha_i\right]^+- \sum_{t \in [T]}\sum_{j \in [M]} \frac{w_{t,j}^2}{w_{t,j}+\gamma} \sum_{i \in [m]}  &\left[\mathbb{E}[g_{t,i}]^\top q_t^j-\alpha_i\right]^+  \\ &\le \gamma TLm,
	\end{align*}
	and,
	\begin{equation*}
		\sum_{t \in [T]}\sum_{j \in [M]} w_{t,j}\sum_{i \in [m]}  \left[\mathbb{E}[g_{t,i}]^\top q_t^j-\alpha_i\right]^+ \ge \sum_{j \in [M]}\max_{i \in [m]}\sum_{t \in [T]}w_{t,j}\left[\mathbb{E}[g_{t,i}]^\top q_t^j-\alpha_i\right]^+.
	\end{equation*}
	Combining the previous results, we obtain, with probability at least $1-\delta$:
	\begin{align*}
		\sum_{t \in [T]}\sum_{j \in [M]}& \frac{w_{t,j}\mathbb{I}(j_t=j)}{w_{t,j}+\gamma} \sum_{i \in [m]}  \left[\mathbb{E}[g_{t,i}]^\top q_t^j-\alpha_i\right]^+  \\& \ge \sum_{j \in [M]}\max_{i \in [m]}\sum_{t \in [T]}w_{t,j}\left[\mathbb{E}[g_{t,i}]^\top q_t^j-\alpha_i\right]^+ - \left(\gamma T L m + L m\sqrt{2T \ln \left(\frac{1}{\delta}\right)}\right).
	\end{align*}
	
	\textbf{Upper-bound to $\sum_{t \in [T]} \sum_{j \in [M]}\frac{w_{t,j}\mathbb{I}(j_t=j)}{w_{t,j}+\gamma}\sum_{i \in [m]}\widehat\xi_{t,j}^{\ \top} q_t^j $.}
	
	 We bound the term noticing that, with probability at least $1-\delta$, it holds:
	\begin{align*}
		\sum_{t \in [T]} \sum_{j \in [M]}&\frac{w_{t,j}\mathbb{I}(j_t=j)}{w_{t,j}+\gamma}\sum_{i \in [m]}\widehat\xi_{t,j}^{\ \top} q_t^j \\ &\le  \sum_{j \in [M]} m \max_{i \in [m]}\sum_{t \in [T]}\frac{w_{t,j}\mathbb{I}(j_t=j)}{w_{t,j}+\gamma}\widehat\xi_{t,j}^{\ \top} q_t^j \\
		& \le \sum_{j \in [M]} m \max_{i \in [m]}\sum_{t \in [T]}\sum_{x,a}\mathbb{I}(j_t=j)\mathbb{I}_t(x,a)\widehat\xi_{t,j}(x,a) + L\sqrt{2T\ln\frac{1}{\delta}} \\
		& = \mathcal{O}\left( m \sqrt{|X||A|LT \ln \left(\frac{mMT|X||A|}{\delta}\right)} + m\ln T |X||A| C + L\sqrt{T \ln \frac{1}{\delta}}\right),
	\end{align*}
	where we employed the Azuma-Hoeffding inequality and where the last step holds following the proof of Theorem~\ref{theo: violationbound}.
	
	\textbf{Upper-bound to $\sum_{t \in [T]} \sum_{j \in [M]} \frac{w_{t,j}\mathbb{I}(j_t=j)}{w_{t,j}+\gamma}\sum_{i \in [m]}\left[\left(g^\circ_i-\mathbb{E}[g_{t,i}]\right)^\top q_t^j\right]^+$.}
	
	We simply bound the quantity of interest as follows:
	\begin{align*}
		\sum_{t \in [T]} \sum_{j \in [M]} \frac{w_{t,j}\mathbb{I}(j_t=j)}{w_{t,j}+\gamma}\sum_{i \in [m]}&\left[\left(g_i^\circ-\mathbb{E}[g_{t,i}]\right)^\top q_t^j\right]^+ \\ & \le m \max_{i \in [m]}\sum_{t \in [T]}\sum_{j \in [M]}\mathbb{I}(j_t=j) \lVert g_i^\circ-\mathbb{E}[g_{t,i}]\rVert_1 \\ &\le mC.
	\end{align*} 
	\textbf{Final result.}
	To conclude we employ the Azuma-Hoeffding inequality on the violation definition, obtaining, with probability at least $1-\delta$:
	\begin{align*}
		V_T & = \sum_{j \in [M]}\max_{i \in [m]}\sum_{t \in [T]}\mathbb{I}(j_t=j)\left[\mathbb{E}[g_{t,i}]^\top q_t^j - \alpha_i\right]^+ \\
		& \le \sum_{j \in [M]}\max_{i \in [m]}\sum_{t \in [T]}w_{t,j}\left[\mathbb{E}[g_{t,i}]^\top q_t^j - \alpha_i\right]^+ + L\sqrt{2T\ln\left(\frac{1}{\delta}\right)}.
	\end{align*}
	Thus, plugging the previous bounds in Equation~\eqref{eq: vhat eq3}, we obtain, with probability at least $1-10\delta$:
	\begin{align*}
		&V_T-\widehat{V}_T \\ & \le  \sum_{j \in [M]}\max_{i \in [m]}\sum_{t \in [T]}\mathbb{I}(j_t=j)\left[\mathbb{E}[g_{t,i}]^\top q_t^j - \alpha_i\right]^+ -\sum_{t \in [T]}\sum_{j \in [M]} \frac{w_{t,j}\mathbb{I}(j_t=j)}{w_{t,j}+\gamma} \sum_{i \in [m]} \left[\widehat{g}^j_{t,i}{}^\top \widehat{q}_t^j-\alpha_i\right]^+\\
		& \le
		m\sum_{t \in [T]}\sum_{j\in [M]} \frac{w_{t,j}\mathbb{I}(j_t=j)}{w_{t,j}+\gamma}\widehat\xi_{t,j}^{\ \top} q_t^j + \sum_{t \in [T]}\sum_{j\in [M]} \frac{w_{t,j}\mathbb{I}(j_t=j)}{w_{t,j}+\gamma}\sum_{i \in [m]}\left[\frac{1}{T}\sum_{\tau \in [T]}(\mathbb{E}[g_{\tau,i}]-\mathbb{E}[g_{t,i}])^\top q_t^j\right]^+ \\
		& \mkern200mu+\gamma TLm + 2Lm\sqrt{2T\left(\frac{1}{\delta} \right)} + 
		 \mathcal{O}\left(mL|X|\sqrt{|A|T\ln\left(\frac{T|X||A|}{\delta}\right)}\right)\\
		 & = \mathcal{O}\left(mL|X|\sqrt{|A|T \ln \left(\frac{mMT|X||A|}{\delta}\right)}+ m \ln(T)|X||A|C + \gamma TLM\right)
	\end{align*}
	This concludes the proof.
\end{proof}
We proceed bounding the estimated violation attained by the optimal instance $j^*$.
\begin{lemma}
		\label{lemma: widehat V i*}
	%Given $\widehat{V}_{T,j^*}^+$ as it is defined in \eqref{def: VThat i*}, and 
	For any $\delta \in (0,1)$, with probability at least $1-16 \delta$, it holds:
	\begin{align*}
		\widehat{V}_{T,j^*} &\le \mathcal{O}\left(mL|X|\sqrt{|A|T\ln\left(\frac{mMT|X||A|}{\delta}\right)}+ m\beta_6C +m\ln(T)|X||A|C + Lm \frac{\ln\left(\frac{M}{\delta}\right)}{2\gamma} \right) \\&\mkern350mu+  m\sqrt{\beta_4T}\nu_{T,j^*}+m\beta_5\nu_{T,j^*} .
	\end{align*}
\end{lemma}
\begin{proof}
	We start by observing that with, probability at least $1-\delta$ under $\mathcal{E}_{\widehat{q}}$, the quantity of interest is bounded as follows:
	\begin{subequations}
	\begin{align}
		&\sum_{t \in [T]}\frac{\mathbb{I}(j_t=j^*)}{w_{t,j^*}+ \gamma}\sum_{i \in [m]}\left[\widehat{g}_{t,i}^{j^*}{}^\top \widehat{q}_t^{j^*}-\alpha_i\right]^+ \nonumber
		\\& \le \sum_{t \in [T]}\frac{\mathbb{I}(j_t=j^*)}{w_{t,j^*}+ \gamma}\sum_{i \in [m]}\left(\left[\widehat{g}_{t,i}^{j^*}{}^\top (\widehat{q}_t^{j^*}-{q}_t^{j^*})+ \widehat{g}_{t,i}^{j^*}{}^\top {q}_t^{j^*} -\widehat\xi_{t,j^*}^{\ \top} q_t^{j^*}-\alpha_i\right]^+ + \widehat\xi_{t,j^*}^{\ \top} q_t^{j^*}\right) \label{eq:Vi 1}\\
		& \le \sum_{t \in [T]}\frac{\mathbb{I}(j_t=j^*)}{w_{t,j^*}+ \gamma} \sum_{i \in [m]}\Bigg(\left[\mathbb{E}[g_{t,i}]^\top q_t^{j^*}-\alpha_i\right]^+ + \widehat\xi_{t,j^*}^{\ \top}{q_t^{j^*}}+ \nonumber\\ & \mkern240mu+ \left[g^\circ_i{}^\top q_t^{j^*} - \mathbb{E}[g_{t,i}]^\top q_t^{j^*} \right]^+ + \lVert \widehat{q}_t^{j^*}-{q}_t^{j^*} \rVert_1\Bigg) \label{eq:Vi 2} \\ 
		& \le \sum_{t \in [T]}\frac{\mathbb{I}(j_t=j^*)}{w_{t,j^*}+ \gamma}\sum_{i \in [m]}\left(\left[\mathbb{E}[g_{t,i}]^\top q_t^{j^*}-\alpha_i\right]^+  + \widehat\xi_{t,j^*}^{\ \top} q_t^{j^*}+ \left[\left(g^\circ_i -\mathbb{E}[g_{t,i}]\right)^\top q_t^{j^*} \right]^+\right) \nonumber\\
		&  \mkern370mu+\mathcal{O}\left(L|X|\sqrt{|A|T\ln\left(\frac{T|X||A|}{\delta}\right)}\right), \label{eq: Vi qhat}
    	\end{align}
	\end{subequations}
	where Inequality~\eqref{eq:Vi 1} holds since $[a+b]^+ \le [a]^++[b]^+,~\forall a,b \in \mathbb{R}$ and by the definition of $\widehat\xi_{t,j^*}$ (see Equation~\eqref{def: xi True}) which implies that all its elements are positive, Inequality~\eqref{eq:Vi 2} holds with probability at least $1-\delta$ by Corollary~\ref{cor: CI for G} and by union bound over $M$, and since that $\lVert\widehat{g}_{t,i}\rVert_\infty \le 1$ and Inequality~\eqref{eq: Vi qhat} holds with probability at least $1-6\delta$ by Lemma~\ref{lem:transition_jin}.
	
	\textbf{Upper-bound to $\sum_{t \in [T]}\frac{\mathbb{I}(j_t=j^*)}{w_{t,j^*}+ \gamma}\sum_{i \in [m]} \left[\left(g^\circ_i-\mathbb{E}[g_{t,i}]\right)^\top  q_t^{j^*} \right]^+$.}
	
	It is immediate to bound the quantity of interest employing the definition of corruption $C$ and by Lemma~\ref{lem:Neu2015}. Indeed, with probability at least $1-\delta$:
	\begin{equation*}
		\sum_{t \in [T]}\frac{\mathbb{I}(j_t=j^*)}{w_{t,j^*}+ \gamma}\sum_{i \in [m]} \left[\left(g^\circ_i-\mathbb{E}[g_{t,i}]\right)^\top  q_t^{j^*}  \right]^+ \le Lm \sqrt{2T\ln \left(\frac{1}{\delta}\right)} + mC.
	\end{equation*}
	
	\textbf{Upper-bound to $\sum_{t \in [T]}\frac{\mathbb{I}(j_t=j^*)}{w_{t,j^*}+ \gamma}\sum_{i \in [m]}\left[\mathbb{E}[g_{t,i}]^\top q_t^{j^*}-\alpha_i\right]^+$.}
	
	We bound the quantity of interest as follows. With probability at least $1-11\delta$, it holds:
	\begin{subequations}
	\begin{align}
		\sum_{t \in [T]}&\frac{\mathbb{I}(j_t=j^*)}{w_{t,j^*}+ \gamma}\sum_{i \in [m]}\left[\mathbb{E}[g_{t,i}]^\top q_t^{j^*}-\alpha_i\right]^+ \nonumber\\
		& \le   m\sqrt{\beta_4T}\nu_{T,j^*}+ m\beta_5\nu_{T,j^*} + 2m\beta_6C + Lm \frac{\ln\left(\frac{M}{\delta}\right)}{2\gamma},
	\end{align}
	\end{subequations}
	thank to Corollary \ref{cor: Neu2015} and  Corollary \ref{corollary: stable V} .
	\begin{comment}

	$V_{T,j^*}$ in Inequality~\eqref{eq: Vi 4} is the positive cumulative violation of costs constraints that Algorithm~\ref{alg:NS_UCSPS} would attain when instantiated with value of corruption $C=2^{j^*}$, if it were to run on its own, and thus Equality~\eqref{eq: Vi 5} holds by Theorem~\ref{theo: violationbound} and Corollary~\ref{theo: over CG} with probability at least $1-9\delta$ considering that, by definition of $j^*$ , $2^{j^*}$ is at most $2C$.
\end{comment}

	\textbf{Upper-bound to $\sum_{t \in [T]}\frac{\mathbb{I}(j_t=j^*)}{w_{t,j^*}+ \gamma}\sum_{i \in [m]}\widehat\xi_{t,j^*}^{\ \top} q_t^{j^*}$.} 
	
	First, notice that, with probability at least $1-\delta$, it holds:
	\begin{equation*}
		\sum_{t \in [T]}\frac{\mathbb{I}(j_t=j^*)}{w_{t,j^*}+ \gamma}\sum_{i \in [m]}\widehat\xi_{t,j^*}^{\ \top} q_t^{j^*}- m\sum_{t \in [T]} \mathbb{I}(j_t=j^*)\widehat\xi_{t,j^*}^{\ \top}q_t^{j^*} \le L\sqrt{2T\ln \left(\frac{1}{\delta}\right)},
	\end{equation*}
	where we employed Lemma~\ref{lem:Neu2015}. Now we observe that, with probability at least $1-\delta$, it holds:
	\begin{align*}
		\sum_{t=1}^T\widehat\xi_{t-1,j^*}^{\ \top} q_t \mathbb{I}(j_t=j^*)& = \sum_{t=1}^T\sum_{x,a}\widehat\xi_{t-1,j^*}(x,a) q_t^{j^*}(x,a) \mathbb{I}(j_t=j^*)\nonumber \\
		& \leq  \sum_{t=1}^T\sum_{x,a}\widehat\xi_{t-1,j^*}(x,a) \mathbb{I}_t(x,a)\mathbb{I}(j_t=j^*) + L\sqrt{2T\ln\frac{1}{\delta}} \\
		& = \mathcal{O}\left(\sqrt{|X||A|LT\ln\left(\frac{mMT|X||A|}{\delta}\right)}+ \ln(T)|X||A|C + L\sqrt{T\ln\frac{1}{\delta}}\right),
		\end{align*}
	where employed the same steps as in the proof of Theorem~\ref{theo: violationbound}, considering that the counter increases if and only if $\mathbb{I}_t(x,a)\mathbb{I}(j_t=j^*)=1$. 
	
	Combining the previous bounds concludes the proof.
\end{proof}

\subsection{Main results}
%\stradi{settare learning rate $\eta$ in base a $\rho$ e modifiche di $L$}
%\stradi{Deve essere $\eta \le \frac{\rho}{2(L^2m+L)}$, quindi metterei $\rho$ direttamente nel learning rate, così non facciamo assunzioni, ma non inserirei l'$\eta$ (e il $\gamma$) giusti nel teorema, ma nei require dell'algoritmo}

In the following, we provide the main results attained by Algorithm~\ref{alg:unknown_c} in terms of regret and violations.
We start providing the regret bound and the related proof.
\RegretUnknown*	
\begin{proof}
	%\anna{ with $\eta=\sqrt{\frac{\ln(T)}{T}}\frac{\rho}{2Lm(Lm+1)}$}
	Employing algorithm \ref{alg:unknown_c}, with probability at least $1-14\delta$,  it holds:
	\begin{subequations}
	\begin{align}
		R_T& = \sum_{t \in [T]}\overline{r}^\top q^*-\sum_{t \in [T]}\overline{r}^\top q_t \nonumber\\
		& = \sum_{t \in [T]}\overline{r}^\top (q^*-q_t^{j^*}) + \sum_{t \in [T]}\overline{r}^\top (q_t^{j^*}-q_t) \nonumber\\
		& =  \sqrt{\beta_1T}\nu_{T,j^*} + \beta_2\nu_{T,j^*} + 2\beta_3C + \sum_{t \in [T]}\overline{r}^\top (q_t^{j^*}-q_t) \label{eq: reg5}\\
		 & \le \sqrt{\beta_1T}\nu_{T,j^*} + \beta_2\nu_{T,j^*} + 2\beta_3C  + 2C - \frac{Lm+1}{\rho}\widehat{V}_T + \frac{Lm+1}{\rho}\widehat{V}_{T,j^*}\nonumber\\
		 & \mkern50mu  - (\sqrt{\beta_1}+\frac{m(Lm+1)}{\rho}\sqrt{\beta_4})\sqrt{T}\nu_{T,j*}-\left(\beta_2+\frac{m(mL+1)}{\rho}\beta_5\right)\nu_{T,j^*}\nonumber
		 \\&  \mkern50mu+ \mathcal{O}\bigg(\frac{M\ln T}{\eta}+ \eta~ {m^4 L^4  T M }+ \eta~M\ln(T)m^4L^2\left(\beta_2^2+\beta_5^2\right) \nonumber\\
		 & \mkern50mu + \eta T(\beta_1+L^2m^4\beta_4)M\log(T)+\gamma T L M + L\sqrt{T\ln \left(\nicefrac{1}{\delta}\right)} + \frac{Lm}{\gamma} \ln \left(\nicefrac{1}{\delta}\right)\bigg)\label{eq: reg6}.
	\end{align}
	\end{subequations}
	where Inequality \eqref{eq: reg5} hold with probability at least $1-11\delta$ by Corollary \ref{corollary: stable R},Inequality \eqref{eq: reg6} holds with probability at least $1-3\delta$ thanks to Lemma \ref{lemma:FTRL} and to the following reasoning, which holds with probability at least $1-\delta$:
	 \begin{subequations}
	 	\begin{align}
	 		\sum_{t \in [T]}\overline{r}^\top (q_t^{j*}-q_t) &= \sum_{t \in [T]} (\overline{r}-\mathbb{E}[r_t])^\top (q_t^{j*}-q_t)+ \sum_{t \in [T]}\mathbb{E}[r_t]^\top (q_t^{j*}-q_t) \nonumber\\
	 		& \le \sum_{t \in [T]}\lVert \overline{r}-\mathbb{E}[r_t]\rVert_1 + \sum_{t \in [T]} \mathbb{E}[r_t]^\top\left( q_t^{j*}-q_t  \right)\label{un1}\\
	 		& \le 2C + \sum_{t \in [T]} \mathbb{E}[r_t]^\top\left( q_t^{j*}-q_t \right) \label{un2}\\
	 		& \le 2C + \sum_{t \in [T]} \mathbb{E}[r_t]^\top q_t^{j*}-\sum_{t \in [T]}\sum_{j\in[M]}w_{t,j} \mathbb{E}[r_t]^\top q_t^{j} + L\sqrt{2T\ln(\nicefrac{1}{\delta})}\label{un3}
	 	\end{align}
	 \end{subequations}
	 where Inequality~\eqref{un1} holds since $|q_t(x,a)-q_t^{j^*}(x,a)|\le 1,~ \forall (x,a)\in X\times A$, where Inequality~\eqref{un2} holds by definition of $C$, and where Inequality \eqref{un3} use Azuma-Hoeffding inequality.
	 We can apply Lemma~\ref{lemma: widehat V i*} to bound $\widehat{V}_{T,j^*}$ with high probability. In fact we observe that with probability at least $1- 16\delta$, it holds:
	\begin{align*}
	 &\frac{Lm+1}{\rho}\widehat{V}_{T,j^*} 		\nonumber \\
	 & \le \mathcal{O}\left(m^2L^2|X|\sqrt{|A|T\ln\left(\frac{mMT|X||A|}{\delta}\right)}+ m^2L\beta_6C +m^2L\ln(T)|X||A|C + L^2m^2 \frac{\ln\left(\frac{M}{\delta}\right)}{2\gamma}\right) \\
		& \mkern200mu+ \frac{(Lm+1)m}{\rho}\beta_5\nu_{T,j^*}+ \frac{m(Lm+1)}{\rho}\sqrt{\beta_4T}\nu_{T,j^*}.
	\end{align*}
	Finally, combining the previous results and by Union Bound, with probability at least $1-30 \delta$, it holds:
	\begin{align}
		&R_T + \frac{Lm+1}{\rho}\widehat{V}_T \nonumber\\ &\leq \mathcal{O}\bigg(\frac{M\ln T}{\eta}+ \eta~ {m^4 L^4  T M }+ \eta~M\ln(T)m^4L^2(\beta_2^2 + \beta_5^2) + \eta T(\beta_1+L^2m^4\beta_4)M\log(T)\nonumber\\
		& \quad +\gamma T L M + L\sqrt{T\ln \left(\nicefrac{1}{\delta}\right)} + \frac{Lm}{\gamma} \ln \left(\nicefrac{1}{\delta}\right)\nonumber\\
		& \quad + m^2L^2|X|\sqrt{|A|T\ln\left(\frac{mMT|X||A|}{\delta}\right)}+ mL\beta_6C + \beta_3 C + m^2L|X||A|\ln(T)C \bigg)\label{eq: reg + Vhat}
	\end{align}
	which concludes the proof after observing that $\widehat{V}_T\ge 0$, by definition, and setting $\gamma=\sqrt{\frac{\ln(M/\delta)}{TM}}$, $\eta \le \frac{1}{2\Lambda m\left(\sqrt{\beta_1T}+\beta_2+\beta_5+\sqrt{\beta_4T}\right)}$.
\end{proof}
We conclude the section providing the violations bound and the related proof.
\ViolationUnknown*
\begin{proof}%\anna{ $\eta=\sqrt{\frac{\ln(T)}{T}}\frac{\rho}{2Lm(Lm+1)}$}
	
	Starting from Inequality \eqref{eq: reg + Vhat}, in order to obtain the final violations bound, it is necessary to find an upper bound for $-R_T$. We proceed as follows,
	%\begin{align*}
		%-R_T& = \sum_{t \in [T]} \mathbb{E}[r_t]^\top (q_t-q^*) \\
		%& = \sum_{t \in [T]} (\mathbb{E}[r_t]-\overline{r})^\top(q_t-q^*) + \sum_{t \in [T]}\overline{r}^\top(q_t-q^*) \\
		%& \le \sum_{t \in [T]} (\mathbb{E}[r_t]-\overline{r})^\top(q_t-q^*)\\& \le \sum_{t \in [T]}L\lVert\mathbb{E}[r_t]-\overline{r}\rVert_1 \\&\le LC,
	%\end{align*}
	\begin{subequations}
		\begin{align}
			\overline{r}^\top q^*& = \text{OPT}_{\overline{r},\overline{G},\alpha} \label{-reg1}\\
			& = \max_{q \in \Delta(P)}\left(\overline{r}^\top q - \frac{L}{\rho} \sum_{i \in [m]}\left[\overline{G}_i^\top q-\alpha_i \right]^+\right) \label{-reg2} \\
			& \ge \overline{r}^\top q_t - \frac{L}{\rho} \sum_{i \in [m]}\left[\overline{G}_i^\top q_t-\alpha_i \right]^+ \nonumber,
		\end{align}
	\end{subequations}
	where Equality~\eqref{-reg1} holds since $q^*$ is the feasible occupancy that maximizes the reward vector $\overline{r}$ and Equality~\eqref{-reg2} holds by Theorem~\ref{theo: strong duality} . This implies $\overline{r}^\top q_t-\overline{r}^\top q^*  \le \frac{L}{\rho}\sum_{i \in [m]}\left[\overline{G}_i^\top q_t-\alpha_i \right]^+$.
	Moreover, it holds:
	\begin{subequations}
		\begin{align}
			\sum_{t \in [T]}&\sum_{i \in [m]}\left[\overline{G}_i^\top q_t-\alpha_i\right]^+\nonumber \\ & \le \sum_{t \in [T]}\left(	\sum_{i \in [m]}\left[\mathbb{E}[g_{t,i}]^\top q_t-\alpha_i \right]^+ + \sum_{i \in [m]}\left[(\overline{G}_i-\mathbb{E}[g_{t,i}])^\top q_t\right]^+ \right) \label{-reg4}\\ 
			& \le\sum_{t \in [T]}\left( \sum_{i \in [m]}\left[\mathbb{E}[g_{t,i}]^\top q_t-\alpha_i \right]^+ + \sum_{i \in [m]} \left\lVert\overline{G}_i-\mathbb{E}[g_{t,i}] \right\rVert_1 \right)\label{-reg5}\\
			& \le \sum_{t \in [T]}\left( \sum_{i \in [m]}\left[\mathbb{E}[g_{t,i}]^\top q_t-\alpha_i \right]^+ + \sum_{i \in [m]} \left( \left\lVert\overline{G}_i-g^\circ_i\right\rVert_1 + \left\lVert g^\circ_i-\mathbb{E}[g_{t,i}] \right\rVert_1 \right)\right) \nonumber\\
			& \le m V_T + 2mC \label{-reg6},
		\end{align}
	\end{subequations}
	where Inequality~\eqref{-reg4} holds since $[a+b]^+ \le [a]^+ + [b]^+,a\in \mathbb{R},b\in \mathbb{R}$, Inequality~\eqref{-reg5} holds since  $q_t(x,a)\le 1 \forall t \in [T],\forall(x,a)\in X\times A$, and finally Inequality~\eqref{-reg6} holds by definition of $C$ and $V_T$ and noticing that $m\max_{i \in [m]}a_i \ge \sum_{i \in [m]}a_i, ~\forall\{a_i\}_{i\in[m]}\subset\mathbb{R}^m$. 
	Thus, combining the previous bounds we lower bound the quantity of interest as follows:
	\begin{subequations}
	\begin{align}
		R_T + \frac{Lm+1}{\rho}V_T &= \sum_{t \in [T]} \mathbb{E}[r_t]^\top \left(q^*-q_t\right) + \frac{Lm+1}{\rho}V_T \nonumber\\
		& = \sum_{t \in [T]} \left(\mathbb{E}[r_t]-\overline{r}\right)^\top(q^*-q_t) + \sum_{t \in [T]}\overline{r}^\top(q^*-q_t) + \frac{Lm+1}{\rho}V_T \nonumber\\
		& \ge -\sum_{t \in [T]}\left\lVert\mathbb{E}[r_t]-\overline{r}\right\rVert_1+ \sum_{t \in [T]}\overline{r}^\top(q^*-q_t) + \frac{Lm+1}{\rho}V_T \label{aux_r_-}\\
		&\ge -2C - \frac{L}{\rho}\left(m V_T + 2mC\right) + \frac{Lm+1}{\rho}V_T \label{eq_V_t_rho_bis}\\
		& = -2C -\frac{2LmC}{\rho}+ V_T\left(\frac{Lm+1}{\rho}-\frac{Lm}{\rho}\right)\nonumber\\
		& = \frac{1}{\rho}V_T - \left(2C + \frac{2LmC}{\rho}\right), \label{eq_V_t_rho}
	\end{align}
	\end{subequations}
	where Inequality~\eqref{aux_r_-} holds since $\underline{v}^\top \underline{w} \ge -\lVert\underline{v}\rVert_1\lVert\underline{w}\rVert_\infty, \forall\underline{v},\underline{w}\in \mathbb{R}^p,p\in \mathbb{N}$, and
	where Inequality~\eqref{eq_V_t_rho_bis} holds since $\overline{r}^\top(q^*-q_t) \ge -\frac{L}{\rho}\sum_{i \in [m]}\left[\overline{G}_i^\top q_t-\alpha_i \right]^+ \ge -\left(mV_T+2mC\right)$ and by definition of $C$.
	Thus, rearranging Inequality~\eqref{eq_V_t_rho}, we finally bound the cumulative violation as follows:
	\begin{align*}
		V_T & \le 2\rho C + 2LmC + \rho R_T + (Lm+1)V_T \\
		& = 2\rho C + 2LmC + (Lm+1)\left(V_T-\widehat{V}_T\right) + \rho\left(R_T + \frac{Lm+1}{\rho}\widehat{V}_T\right)\\
		& \leq  \mathcal{O}\Bigg(m^2L^2|X|\sqrt{|A|T\ln\left(\frac{mMT|X||A|}{\delta}\right)} + m^2L \ln(T)|X||A|C + 
		\gamma mTL^2M\Bigg) \\ &\mkern400mu+ \mathcal{O}\left(R_T + \frac{Lm+1}{\rho}\widehat{V}_T\right),
	\end{align*}
	where the last inequality holds by Equation~\eqref{eq: reg + Vhat} and by Lemma~\ref{lemma: widehat V}, with probability at least $1-4\delta$ under $\mathcal{E}_{\widehat{q}}$. Employing a Union Bound, setting $\gamma=\sqrt{\frac{\ln(M/\delta)}{TM}}$  and $\eta \le \frac{1}{2\Lambda m\left(\sqrt{\beta_1T}+\beta_2+\beta_5+\sqrt{\beta_4T}\right)}$ concludes the proof.
\end{proof}

\section{Auxiliary lemmas from existing works}
\label{app:auxiliary}

In the following section, we provide useful lemma from existing works.
\subsection{Auxiliary lemmas for the FTRL master algorithm}
In the following, we provide the optimization bound attained by the FTRL instance employed by Algorithm~\ref{alg:unknown_c}.
\begin{lemma} [\cite{jin2024no}]
	\label{lem: FTRL}
	The FTRL algorithm over a convex subset $\Omega$
	of the $(M-1)$-dimensional simplex $\Delta_M$ :
	\begin{equation*}
		w_{t+1} = \underset{w \in \Omega}{\argmin} \left\{\sum_{\tau\in[t]} \ell_\tau^\top w + \frac{1}{\eta}\sum_{j\in[M]} \ln \left(\frac{1}{w_j}\right) \right\},
	\end{equation*}
	ensures for all $u \in \Omega$:
	\begin{equation*}
		\sum_{t \in[T]} \ell_t^\top (w_t-u) \le \frac{M \ln T}{\eta} + \eta \sum_{t \in[T]} \sum_{j\in[M]} w_{t,j}^2 \ell_{t,j}^2,
	\end{equation*}
	as long as $\eta w_{t,j} \lvert \ell_{t,j} \rvert \le \frac{1}{2}$ for all $t,j$.
\end{lemma}

\subsection{Auxiliary lemmas for the optimistic loss estimator}
In the following, we provide some results related to the optimistic biased estimator of the loss function. Notice that, given any loss vector $\ell_t\in[0,1]^M$, the following results are provided for $\widehat{\ell}_{t,j}:= \frac{\mathbb{I}_t(j)}{w_{t,j}+\gamma_t}\ell_{t,j}$, where $j\in[M]$, $\ell_{t,j}$ is the $j$-th component of the loss vector, $\mathbb{I}_t(j)$ is the indicator functions which is $1$ when arm $j$ is played and $\gamma_t$ is defined as in the following lemmas.
\begin{lemma} [\cite{neu}]
	\label{lem:Neu2015}
	Let $(\gamma_t)$ be a fixed non-increasing sequence with $\gamma_t \ge 0$ and let $\alpha_{t,j}$ be nonnegative
	$\mathcal{F}_{t-1}$-measurable random variables satisfying $\alpha_{t,j}\le 2 \gamma_t$ for all $t$ and $j$. Then, with probability at
	least $1-\delta$,
	\begin{equation*}
		\sum_{t\in[T]} \sum_{j\in[M]} \alpha_{t,j}\left(\widehat{\ell}_{t,j}-\ell_{t,j}\right) \le \ln \left(\frac{1}{\delta}\right).
	\end{equation*}
\end{lemma}
\begin{corollary}[\cite{neu}]
	\label{cor: Neu2015}
	 Let $\gamma_t=\gamma \ge 0$ for all $t$. With probability at least $1 - \delta$,
	\begin{equation*}
		\sum_{t\in[T]} \left(\widehat{\ell}_{t,j}-\ell_{t,j}\right) \le \frac{\ln \left(\frac{M}{\delta}\right)}{2\gamma},
	\end{equation*}
	simultaneously holds for all $j \in [M]$.
\end{corollary}
%\begin{lemma} [\cite{JinLearningAdversarial2019}] \label{lem:alpha_jin} For any sequence of functions $\alpha_1, \ldots, \alpha_T$ such that $\alpha_t \in[0,2 \gamma]^{X \times A}$ is $\mathcal{F}_t$-measurable for all $t$, we have with probability at least $1-\delta$,
	%$$
	%\sum_{t=1}^T \sum_{x, a} \alpha_t(x, a)\left(\widehat{\ell}_t(x, a)-\frac{q_t(x, a)}{u_t(x, a)} \ell_t(x, a)\right) \leq L \ln \frac{L}{\delta}.
	%$$
%\end{lemma}
\subsection{Auxiliary lemmas for the transitions estimation}
Next, we introduce \emph{confidence sets} for the transition function of a CMDP, by exploiting suitable concentration bounds for estimated transition probabilities.
By letting $M_t (x, a , x')$ be the total number of episodes up to $t \in [T]$ in which $(x,a) \in X \times A$ is visited and the environment transitions to state $x' \in X$, the estimated transition probability at $t$ for $(x,a,x')$ is:
\begin{equation*}
\overline{P}_t\left(x' |x, a\right)=\frac{M_t ( x, a , x')}{\max \left\{1, N_t(x, a)\right\}}.
\end{equation*}
Then, the confidence set for $P$ at episode $t \in [T]$ is defined as:
\begin{align*}
	\mathcal{P}_t := \Bigg\{\widehat{P}:\Big |&\overline{P}_t(x'|x,a)-\widehat{P}(x'|x,a)\Big |\le  \epsilon_t(x'|x,a), \\& \forall(x,a,x') \in X_k \times A \times X_{k+1}, k\in[0...L-1]\Bigg\},
\end{align*}
where $\epsilon_t(x'|x,a)$ is defined as:
\begin{equation*}
	\epsilon_t (x^\prime | x, a) \coloneqq 2\sqrt{\frac{\overline{P}_t\left(x^{\prime} | x, a\right)\ln \left({T|X||A|}/{\delta}\right)}{\max \left\{1, N_t(x, a)-1\right\}}} + \frac{14 \ln \left({T|X||A|}/{\delta}\right)}{3\max \left\{1, N_t(x, a)-1\right\}},
\end{equation*}
for some confidence $\delta \in(0,1)$.

Given the estimated transition function space $\mathcal{P}_t$, the following result can be proved.
\begin{lemma}[\cite{JinLearningAdversarial2019}]
	\label{lem:prob_int_jin}
	With probability at least $1-4\delta$, we have $P \in \mathcal{P}_t$ for all $t \in [T]$.
\end{lemma}
Notice that we refer to the event $P \in \mathcal{P}_t$ for all $t \in [T]$ as $\mathcal{E}_{P}$.

We underline that the estimated occupancy measure space by Algorithm~\ref{alg:NS_UCSPS} is the following:
\begin{equation*}
	\Delta(\mathcal{P}_t):=\begin{cases}
		\forall k, &  \underset{x \in X_k, a \in A, x^{\prime} \in X_{k+1}}{\sum} q\left(x, a, x^{\prime}\right)=1\\
		\forall k, \forall x, & \underset{a \in A, x^{\prime} \in X_{k+1}}{\sum} q\left(x, a, x^{\prime}\right)=\underset{x^{\prime} \in X_{k-1}, a \in A}{\sum} q\left(x^{\prime}, a, x\right) \\
		\forall k, \forall\left(x, a, x^{\prime}\right), &  q\left(x, a, x^{\prime}\right) \leq\left[\overline{P}_t\left(x^{\prime} | x, a\right)+\epsilon_t\left(x^{\prime} \mid x, a\right)\right] \underset{y \in X_{k+1}}{\sum} q(x, a, y)\\
		&  q\left(x, a, x^{\prime}\right) \geq\left[\overline{P}_t\left(x^{\prime} |x, a\right)-\epsilon_t\left(x^{\prime} \mid x, a\right)\right] \underset{y \in X_{k+1}}{\sum} q(x, a, y)\\
		& q\left(x, a, x^{\prime}\right) \geq 0 \\
	\end{cases}.
\end{equation*}
To conclude, we restate the result which bounds the cumulative distance between the estimated occupancy measure and the real one.
\begin{lemma}[\cite{JinLearningAdversarial2019}]
	\label{lem:transition_jin}
	With probability at least $1-6\delta$, for any collection of transition functions $\{P_t^x\}_{x\in X}$ such that $P_t^x \in \mathcal{P}_{t}$, we have, for all $x$,
	\begin{equation*}
	\sum_{t\in[T]} \sum_{x \in X, a \in A}\left|q^{P_t^x,\pi_t}(x, a)- q_t(x, a)\right| \leq \mathcal{O}\left(L|X| \sqrt{|A| T \ln \left(\frac{T|X||A|}{\delta}\right)}\right).
	\end{equation*}
\end{lemma}

\section{Auxiliary lemmas for stability}
\label{app: stability}
In this section we state the results related to the stability of the arm-algorithms when $C$ is not known. The procedure is inspired by \cite{jin2024no} and \cite{corralling}, but adapted to the case of \emph{Constrained} MDP in high probability. We first give some important definitions. In these definitions we will use $C_t$ as the value of adversarial corruption at episode $t\in[T]$, where $C_t$ is defined as $C_t:= \max \{C_t^G,C_t^r\}$, which meets the requirement of upper bounding the adversarial corruption at each considered episode. In addition it holds that  $\sum_{t \in [T]}C_t \le C_r + C_G$ or equivalently $C\le \sum_{t \in [T]}C_t \le 2C$, which does not influence the order of the analysis.
%\stradi{non dovrebbe valere per tutti i CMDPs? e non solo avversari}\anna{riprendo la definizione di \cite{jin2024no}. Essendo una definizione credo inventata da loro (o comunque on citano puntualmente nessuno) non cambia niente credo, l'unica questione è se cmdp stochastici possano diventare 1/2 stable partendo da corruption robust con stabilize. Credo di sì ma la dimostrazione andrebbe modificata un po'. } \stradi{pero se non fosse cosi staremmo dicendo che nel caso stocastico il nostro master algorithm non funziona? Cioe, quello che non mi torna è questo, se $C=0$, la definizione di sotto dovrebbe valere lo stesso, quindi anche nel caso stocastico. In Jin, l'MDP è avversario (nelle loss) ma $C$ è solo sulle transizioni, quindi l'MDP è avversario indipendetemente da $C$,forse è per questo che loro usano questa definizione. Cerca di controllare}\anna{Sì,probabilmente hai ragione, da Jin non vedo altri motivi per cui non dovrebbe funzionare in caso stochastic, per cui suppongo che abbiano aggiunto adversarial perchè nel loro caso lo è}\stradi{ok, ricontrolla che torni tutto e nel caso togliamo adversarial dalla definition}\anna{A me sembra torni tutto}
\begin{definition}
	\label{def: corruption robust}
	A CMDP algorithm is \textbf{corruption-robust} if it takes $\theta$ (a guess on the
	corruption amount) as input, and achieves  for any random stopping time $t'\le T$, whenever $\sum_{t\in[t']}C_t<\theta$:
	\begin{equation*}
		\sum_{t \in [t']}\overline{r}^\top (q^*-q_t)\le \sqrt{\beta_1t'}+\left(\beta_2+\beta_3\theta\right)\mathbb{I}(t'\ge 1) ,
	\end{equation*}
	and
	\begin{equation*}
		\max_{i\in[m]}\sum_{t \in [t']}\left[g_{t,i}^\top q_t -\alpha_i\right]^+\le \sqrt{\beta_4t'}+\left(\beta_5+\beta_6\theta\right)\mathbb{I}(t'\ge 1) .
		\end{equation*}
\end{definition}
Notice that Algorithm \ref{alg:NS_UCSPS} is corruption-robust after applying a doubling trick to make it work for any stopping time, with probability at least $1-9\delta$ thank to Theorem~\ref{theo: over Cr} and Theorem~\ref{theo: over CG} %\anna{Quando l'algoritmo viene reso for any time potrebbe esserci un problema di definizione di corruption fino ad episode $t'$, nel senso che se per esempio $C=C_r$ with horizon $T$, questo potrebbe non essere vero con horizon $t'$. Si potrebbe passare a definizione $C=\sum_{t\in[T]}\max\{C_t^G,C_t^r\}$, con quindi $C_t=\max\{C_t^G,C_t^r\}$. Dove questa nuova definizioe di $C$ è minore del doppio di quella precedente, per cui stesso ordine}. \stradi{mi torna il fatto che fino a $t^\prime$ il massimo potrebbe essere un altro $C$ (ad esempio $C_{r,t}$ mentre su tutti i round è $C_G$), non capisco pero perchè dovrebbe essere un problema, cioe non appena $C_G$ super $C_r$ dovrebbe tornare tutto no? Cioè perchè prendendo $\sum_{t\in[t']}C_t=\max\{\sum_{t\in[t']} C_t^G, \sum_{t\in[t']} C_t^r\}$ nella definizione di corruption robust non dovrebbe tornare? Magari hai ragione tu e la cambiamo, non c'è comunque problema}
%\anna{Credo basterebbe dire $\sum_{t\in[t']}C_t=\max\{\sum_{t\in[t']} C_t^G, \sum_{t\in[t']} C_t^r\}$, però questo porterebbe a una definizione di $C_t$ non univoca ma dipendente da su quanti episodi $t'$ è definita la sommatoria, che potrebbe rendere la notazione un po' confusa. Se  è abbastanza chiaro io aggiungerei solo qualcosa del genere: $C_t$ is defined as $C_t:= \max \{C_t^G,C_t^r\}$, which respect the requisite of upper bounding the adversarial corruption at each considered episode. In addition it holds that  $\sum_{t \in [T]}C_t \le C_r + C_G$ or equivalently $C\le \sum_{t \in [T]}C_t \le 2C$, which does not influence the order of the analysis. }
Furthermore, we introduce the notion of $\alpha$-stability. An algorithm is considered to be $\alpha$-stable, if its regret under condition imposed by Algorithm~\ref{alg:unknown_c} is of order $\nu_T^\alpha \cdot \BigOL{R_T}$, where $R_T$ is the upper bound on the regret attained by the algorithm if it receives feedback at each episode. In particular, we are interested in the $1$-stability.
\begin{definition}
	\label{def: stable}
	An algorithm is $1$-\textbf{stable} if, under the condition imposed by Algorithm~\ref{alg:unknown_c}, it holds:
	\begin{equation*}
		\sum_{t\in[T]}\overline{r}^\top (q^*-q_t) \le \sqrt{\beta_1  T}\nu_{j,T}+\beta_2\nu_{j,T}  +\beta_3C,
	\end{equation*} and
	\begin{equation*}
		\max_{i\in[m]}\sum_{t\in[T]}\left[g_{t,i}^\top q_t -\alpha_i\right]^+ \le \sqrt{\beta_4T}\nu_{j,T}+\beta_5\nu_{j,T}  +\beta_6C.
	\end{equation*}
\end{definition}

We can use the procedure defined by Algorithm~\ref{alg: stabilize} and originally proposed by~\cite{jin2024no} to transform a generic corruption robust algorithm to a $1$-stable algorithm. Differently from~\cite{jin2024no}, in our setting, we use the natural symmetry between regret and positive cumulative constraints violation to stabilize both the regret and the positive cumulative constraints violation. We have a different bound for $C_t$ (value of adversarial corruption at episode $t$): indeed,  $C_t \le\max\{\lVert\mathbb{E}[r_t]-r^\circ\rVert_1,\max_{i\in[m]}\lVert\mathbb{E}[g_{t,i}]-g_i^\circ\rVert_1\}$ is bounded by $|X||A|$. Finally, we are interested in obtaining results that hold in high probability rather than in expectation. To do so, we focus on $1$-stability guarantee rather than $1/2$-stability as in~\cite{jin2024no} since removing the expectation prevents us from achieving the result above with lower coefficients. We can state the following result.
\begin{lemma} [Adapted from \cite{jin2024no}]
	Given an algorithm which is corruption robust according to Definition~\ref{def: corruption robust} with parameters $(\beta_1,\beta_2,\beta_3,\beta_4,\beta_5,\beta_6)$ and $\beta_1\ge \mathcal{O}(L^2 \log(\nicefrac{T}{\delta}))$, $\beta_4\ge \mathcal{O}(L^2 \log(\nicefrac{T}{\delta}))$, with probability at least $1-p$ with $p \in (0,1)$, then, it is possible  convert it to an $1$-stable algorithm with probability at least $1-p-2\delta$ according to Definition~\ref{def: stable} with parameters $(\beta_1',\beta_2',\beta_3',\beta_4',\beta_5',\beta_6')$ as $\beta_1'=\mathcal{O}\left(\beta_1\right),\beta_2'=\mathcal{O}\left(\beta_2+\beta_3|X||A|\log(\nicefrac{\log(T)}{\delta})\right),\beta_3'=\mathcal{O}\left(\beta_3\log(T)\right),\beta_4'=\mathcal{O}\left(\beta_4\right),\beta_5'=\mathcal{O}\left(\beta_5+\beta_6|X||A|\log(\nicefrac{\log(T)}{\delta})\right),\beta_6'=\mathcal{O}\left(\beta_6\log(T)\right)$, employing Algorithm~\ref{alg: stabilize}.
\end{lemma}
\begin{proof}
	Suppose Algorithm~\ref{alg: stabilize} is initialized with the true value of adversarial corruption $C$.
	We will first prove the result for the regret. We will start by considering a generic instance algorithm $k\in[M]$. Define the quantity $d_{t,k}=\mathbb{I}(w_t\in(2^{-k-1},2^{-k}])$ and $h_{t,k}=\mathbb{I}(\text{Instance $k$ receives feedback at episode $t$})$. We observe that with probability at least $1- \left(p + \mathbb{P}\left(\bigcup_{k\in[\log(T)]}\{\sum_{t\in[T]}C_t d_{t,k}h_{t,k}>\theta_k\}\right)\right)$ it holds:
	\begin{align*}
		\sum_{t\in[T]}\overline{r}^\top(q^*-q_t) d_{t,k} h_{t,k} \le \sqrt{\beta_1\sum_{t\in [T]}d_{t,k}h_{t,k}}+\left(\beta_2+\beta_3\theta\right)\max_{t\in[T]}d_{t,k} ,
	\end{align*}
	by the corruption-robust property of instance $k$.
	We study now the quantity $\mathbb{P}\left(\bigcup_{k\in[M]}\{\sum_{t\in[T]}C_t d_{t,k}h_{t,k}>\theta_k\}\right)$.
	Notice that $\mathbb{E}[h_{t,k}|d_{t,k}]=2^{-k-1}d_{t,k}$, and since $d_{t,k}$ is an indicator function then $\mathbb{E}[h_{t,k}|d_{t,k}]d_{t,k}=\mathbb{E}[h_{t,k}|d_{t,k}]$. In addition, since $\sum_{t\in [T]}C_t \le 2C$, it holds:
	\begin{equation*}
		\sum_{t\in[T]}C_t \mathbb{E}[h_{t,k}|d_{t,k}]d_{t,k} = 2^{-k-1}\sum_{t\in[T]}C_t d_{t,k}\le 2^{-k}C,
	\end{equation*}
	and  with probability at least $1-\nicefrac{\delta}{\log(T)}$ noticing that $M=\log(T)$:
	\begin{subequations}
	\begin{align}
		\sum_{t\in[T]}C_t& d_{t,k}h_{t,k}-\sum_{t\in[T]}C_t \mathbb{E}[h_{t,k}|d_{t,k}]d_{t,k} \nonumber\\
		&\le 2\sqrt{\sum_{t\in[T]}C_t^2d_{t,k}\mathbb{E}[h_{t,k}|d_{t,k}]\log\left(\frac{\log(T)}{\delta}\right)}+ |X||A|\log\left(\frac{\log(T)}{\delta}\right)\label{eq: stable eq1}\\
		& \le 2\sqrt{|X||A|\sum_{t\in[T]}C_td_{t,k}\mathbb{E}[h_{t,k}|d_{t,k}]\log\left(\frac{\log(T)}{\delta}\right)}+|X||A|\log\left(\frac{\log(T)}{\delta}\right) \label{eq: stable eq2}\\
		&\le \sum_{t\in[T]}C_t\mathbb{E}[h_{t,k}|d_{t,k}]d_{t,k} + 2|X||A|\log\left(\frac{\log(T)}{\delta}\right)\label{eq: stable eq3},
	\end{align}
\end{subequations}
	where Inequality \eqref{eq: stable eq1} holds with probability at least $1-\nicefrac{\delta}{\log(T)}$ by Freedman inequality, Inequality \eqref{eq: stable eq2} holds since $C_t\le|X||A|$, and Inequality \eqref{eq: stable eq3} holds by AM-GM inequality.
	Therefore, it holds simultaneously for all $k\in [M]$:
	\begin{align*}
	\sum_{t\in[T]}C_t d_{t,k}h_{t,k} & \le 2\sum_{t\in[T]}C_t\mathbb{E}[h_{t,k}|d_{t,k}]d_{t,k} + 2|X||A|\log\left(\frac{\log(T)}{\delta}\right)\\& \le 2^{-k+1}C + 2|X||A|\log\left(\frac{\log(T)}{\delta}\right)= \theta_k,
	\end{align*} 
	with probability at least $1-\delta$, so $\mathbb{P}\left(\bigcup_{k\in[M]}\{\sum_{t\in[T]}C_t d_{t,k}h_{t,k}>\theta_k\}\right) \le \delta$.
\begin{comment}
	\begin{align*}
		\sum_{t\in [T]}\overline{r}^\top(q^*-q_t)d_{t,k}& = \frac{1}{2^{-k-1}}\sum_{t\in [T]}\overline{r}^\top(q^*-q_t)2^{-k-1}d_{t,k}\\
		& = \frac{1}{2^{-k-1}}\sum_{t\in [T]}\overline{r}^\top(q^*-q_t)\mathbb{E}[h_{t,k}|d_{t,k}]d_{t,k}.
	\end{align*} 
\end{comment}
Moreover, notice that with probability at least $1-p-2\delta$ thanks to the definition of corruption robust and Azuma-Hoeffding inequality,  it holds simultaneously for all $k$: 
	\begin{align*}
		&\sum_{t \in [T]} \overline{r}^\top (q^* - q_t) d_{t,k} \\&
		=  \frac{1}{2^{-k-1}}\sum_{t\in [T]}\overline{r}^\top(q^*-q_t)2^{-k-1}d_{t,k}\\
		&=\frac{1}{2^{-k-1}} \sum_{t \in [T]} \overline{r}^\top (q^* - q_t) d_{t,k} \mathbb{E}[h_{t,k} \mid d_{t,k}] \\
		& = \frac{1}{2^{-k-1}} \left( \sum_{t \in [T]} \overline{r}^\top (q^* - q_t) d_{t,k} \left( \mathbb{E}[h_{t,k} \mid d_{t,k}] - h_{t,k} \right) + \sum_{t \in [T]} \overline{r}^\top (q^* - q_t) d_{t,k} h_{t,k} \right) \\
		& \leq \frac{1}{2^{-k-1}} \left( L \sqrt{2 \ln\left( \frac{\log(T)}{\delta} \right) \sum_{t \in [T]} d_{t,k}} + \sqrt{\beta_1 \sum_{t \in [T]} d_{t,k}} + (\beta_2 + \beta_3 \theta_k) \max_{t \in [T]} d_{t,k} \right)\\
		& \le \mathcal{O}\left(\frac{1}{2^{-k-1}}\left(\left(\sqrt{\beta_1}+L\sqrt{\log\left(\frac{T}{\delta}\right)}\right)\sqrt{ T}\max_{t\in[T]}d_{t,k} + (\beta_2 + \beta_3\theta)\max_{t\in[T]}d_{t,k}\right)\right),
	\end{align*}
	noticing that $\mathbb{E}\left[d_{t,k}\left(\mathbb{E}\left[h_{t,k}|d_{t,k}\right]-h_{t,k}\right)\right]= \mathbb{E}\left[h_{t,k}|d_{t,k}\right]-\mathbb{E}[h_{t,k}]d_{t,k}=\mathbb{E}\left[h_{t,k}|d_{t,k}\right]-\mathbb{E}\left[h_{t,k}|d_{t,k}\right]=0$, since the expectation is taken w.r.t. the randomization of Algorithm~\ref{alg: stabilize} and the distribution generated given the external probability of receiving feedback $w_t$.

	To conclude with probability at least $1-p-2\delta$:
	\begin{align*}
		\sum_{t\in[T]}&\overline{r}^\top (q^*-q_t)\mathbb{I}\left(w_t \ge \frac{1}{T}\right) \\
		&\le \sum_{k\in[M]}\sum_{t\in[T]}\overline{r}^\top(q^*-q_t) d_{t,k}
		\\
		& \le \mathcal{O}\left(\sqrt{\beta_1 T} \max_{t\in [T]}\frac{1}{w_t}+ (\beta_2 + \beta_3|X||A|\log(\nicefrac{\log(T)}{\delta}))\max_{t \in [T]}\frac{1}{w_t} +  \beta_3\log(T)C  \right)\\
		& \le \mathcal{O}\left(\left(\sqrt{\beta_1'T}+ \beta_2'\right)\nu_T + \beta_3'C\right),
	\end{align*}
	with $\sqrt{\beta_1}\ge \mathcal{O}(L\sqrt{ \log(\nicefrac{T}{\delta})})$.
	Notice that the analogous reasoning can be applied to the positive cumulative constraints violation with parameters $\beta_4,\beta_5,\beta_6$.
\end{proof}

\begin{algorithm}[H]\caption{Adapted \texttt{STABILIZE} \cite{jin2024no}}
	\label{alg: stabilize}
	\begin{algorithmic}[1]
		\Require $C$, $\delta\in (0,1)$
		\State Initialize $M=\log(T)$ instance of Algorithm~$\ref{alg:NS_UCSPS}$, each instance $k \in [M]$ initialized with corruption parameter:
		\begin{equation*}
			\theta_k:= 2^{-k+1}C + 2|X||A|\log\left(\frac{\log(T)}{\delta}\right)
		\end{equation*}
		\For{$t \in [T]$}
		
			\State Observe $w_t$, probability of receiving feedback.
			\If{$w_t > \frac{1}{T}$}
			\State Let $k_t$ be such that $w_t \in (2^{-k_t-1},2^{-k_t}]$
			\State Choose $\pi_t$ as policy proposed by instance $k_t$
			\State If the algorithm receives feedback send it to instance $k_t$ with probability $\frac{2^{-k_t-1}}{w_t}$
			\EndIf
			\If{$w_t \le \frac{1}{T}$}
			\State Propose random policy $\pi_t$
			\EndIf
		\EndFor
	\end{algorithmic}
\end{algorithm}

\begin{corollary}
	\label{corollary: stable V}
	Being $j^*$ such that $C\in (2^{j^*-1},2^{j^*}]$ then with probability at least $1-11\delta$ it holds:
	\begin{equation*}
		\max_{i\in [m]}\sum_{t\in [T]}\left[\mathbb{E}[g_{t,i}]^\top q_t^{j^*}-\alpha_i\right]^+ \le    \sqrt{\beta_4T}\nu_{T,j^*}+ \beta_5 \nu_{T,j^*} +2\beta_6C,
	\end{equation*}
	with $\sqrt{\beta_4}=\mathcal{O}\left(L|X|\sqrt{|A|\ln(\nicefrac{mT|X||A|}{\delta})}\right)$, $\beta_5=\mathcal{O}\left(|X|^2|A|^2\log(T)\log\left(\nicefrac{\log(T)}{\delta}\right)\right)$and $\beta_6=\mathcal{O}\left(\ln(T)^2|X||A|\right)$ .
	
\end{corollary}

\begin{corollary}
	\label{corollary: stable R}
	Being $j^*$ such that $C\in (2^{j^*-1},2^{j^*}]$ then with probability at least $1-11\delta$ it holds:
	\begin{equation*}
		\sum_{t\in [T]}\overline{r}^\top (q^*-q_t^{j^*}) \le  \sqrt{\beta_1T}\nu_{T,j^*} + \beta_2 \nu_{T,j^*} +2\beta_3C,
	\end{equation*}
	where $\sqrt{\beta_1}=\mathcal{O}\left(L|X|\sqrt{|A|\ln(\nicefrac{T|X||A|}{\delta})}\right)$, $\beta_2= \mathcal{O}\left(|X|^2|A|^2\log(T)\log\left(\nicefrac{\log(T)}{\delta}\right)\right)$ and  $\beta_3=\mathcal{O}\left(\ln(T)^2|X||A|\right)$ .
	
\end{corollary}

% This document was modified from the file originally made available by
% Pat Langley and Andrea Danyluk for ICML-2K. This version was created
% by Iain Murray in 2018, and modified by Alexandre Bouchard in
% 2019 and 2021 and by Csaba Szepesvari, Gang Niu and Sivan Sabato in 2022. 
% Previous contributors include Dan Roy, Lise Getoor and Tobias
% Scheffer, which was slightly modified from the 2010 version by
% Thorsten Joachims & Johannes Fuernkranz, slightly modified from the
% 2009 version by Kiri Wagstaff and Sam Roweis's 2008 version, which is
% slightly modified from Prasad Tadepalli's 2007 version which is a
% lightly changed version of the previous year's version by Andrew
% Moore, which was in turn edited from those of Kristian Kersting and
% Codrina Lauth. Alex Smola contributed to the algorithmic style files.

\end{document}